\pdfoutput=1

\documentclass[11pt]{article}

\usepackage[preprint]{acl}

\usepackage{times}
\usepackage{latexsym}

\usepackage[T1]{fontenc}

\usepackage[utf8]{inputenc}

\usepackage{microtype}

\usepackage{inconsolata}

\usepackage{graphicx}

\usepackage{amsmath}
\usepackage{amssymb}
\usepackage{mathtools}
\usepackage{amsthm}

\usepackage{booktabs}
\usepackage{adjustbox}
\usepackage{subfig} 
\usepackage{multirow}
\usepackage{tabularx}

\usepackage{dsfont}
\usepackage{amsfonts}
\usepackage{stackrel}
\DeclareMathOperator*{\argmax}{arg\,max}
\DeclareMathOperator*{\argmin}{arg\,min}

\usepackage[english]{babel}
\newtheorem{proposition}{Proposition}

\usepackage{xcolor}
\usepackage{xurl}

\newcommand{\hypos}{Y_\mathrm{ref}}

\newcommand{\bE}{\mathop{\mathbb{E}}}

\newcommand{\klrbon}{RBoN$_\text{KL}$}
\newcommand{\wdrbon}{MBR-BoN}

\title{Regularized Best-of-N Sampling with Minimum Bayes Risk Objective \\ for Language Model Alignment}

\author{Yuu Jinnai, Tetsuro Morimura, Kaito Ariu, Kenshi Abe \\
CyberAgent\\
Tokyo, Japan \\
\texttt{\{jinnai\_yu,morimura\_tetsuro,kaito\_ariu,abe\_kenshi\}@cyberagent.co.jp} \\}

\begin{document}
\maketitle

\begin{abstract}
Best-of-N (BoN) sampling with a reward model has been shown to be an effective strategy for aligning Large Language Models (LLMs) to human preferences at the time of decoding.
BoN sampling is susceptible to a problem known as \textit{reward hacking} when the accuracy of the reward model is not high enough due to the quality or the quantity of the preference dataset. Because the reward model is an imperfect proxy for the true objective, over-optimizing its value can compromise its performance on the true objective.
In this research, we propose MBR-BoN, a variant of BoN that aims to mitigate reward hacking at inference time by incorporating the Minimum Bayes Risk (MBR) objective as a proximity regularization term.
We show empirically and analytically that the MBR objective quantifies the proximity of the response to the reference policy, serving as a proximity regularizer.
We evaluate MBR-BoN on the AlpacaFarm and Anthropic's hh-rlhf datasets and show that it outperforms both BoN sampling and MBR decoding. 
We also evaluate MBR-BoN to generate a pairwise preference learning dataset for Direct Preference Optimization (DPO). Empirical results show that models trained on a dataset generated with MBR-BoN outperform those with vanilla BoN.
Our code is available at \url{https://github.com/CyberAgentAILab/regularized-bon}.
\end{abstract}

\section{Introduction}

Language model alignment is a widely used technique for optimizing the behavior of Large Language Models (LLMs) to human preferences, steering the models to generate informative, harmless, and helpful responses \citep{ziegler2020finetuning,stiennon2020,NEURIPS2022_b1efde53}.
\textbf{Best-of-N (BoN) sampling} is widely used to align the LLM at decoding time \citep{stiennon2020,nakano2022webgpt}. BoN samples $N$ responses from the language model and selects the best response according to the proxy reward model as the output of the system.

However, BoN sampling is known to suffer from the \textit{reward hacking} problem \citep{amodei2016concrete,ziegler2020finetuning,stiennon2020,skalse2022defining,pmlr-v202-gao23h}. 
The reward hacking is defined as a phenomena where the learning agent overfits to the misspecified reward model, failing to optimise for the true intended objective \citep{pan2022the,lambert2024alignment}.
The problem occurs because of reward misspecification; the proxy reward trained from a human preference dataset of a limited quality or quantity does not perfectly reflect true human preferences. As a result, optimizing for the reward model does not always optimize for the preference of the true intended objective.
For example, \citet{NEURIPS2023_5fc47800} shows that with 25\% label noise, which is the amount of disagreement observed in real-world preference annotations \citep{stiennon2020,NEURIPS2022_b1efde53}, BoN sampling degrades performance with $N$ greater than 16 (Figures 12 and 13 in \citealt{NEURIPS2023_5fc47800}). 
\citet{wen2024rethinking} shows that even when the proxy reward model performs reasonably well relative to the reference model, it still exhibits overoptimization behavior. 
We also observe the degradation of performance with $N$ greater than 32 when the amount of train data for the proxy reward model is limited (Appendix~\ref{apd:overoptimization}).

Given that human preferences depend on the domain, language, culture, and various other factors of the users \citep{hu-etal-2023-decipherpref,Wan_Kim_Kang_2023,li2024personal,DBLP:conf/aaai/SorensenJHLPWDL24,li-etal-2024-land,afzoon2024persobench,agrawal2024modeling}, it is desirable to develop a method that is robust to the situation where the reward model is misspecified due to limited quality and/or quality of the preference dataset.
A common approach to mitigate reward hacking in preference learning is to add a proximity regularization term to the loss function to keep the trained model close to the reference model \citep{stiennon2020,NEURIPS2022_b1efde53,rafailov2023direct}.
Previous work in BoN has shown that reducing the number of samples $N$ mitigates the reward hacking \citep{nakano2022webgpt,pan2022the,lambert2024alignment}. This approach successfully increases the proximity to the reference policy \citep{nakano2022webgpt,beirami2024theoretical} but at the expense of diminished improvement obtained by the method. 

To this end, we propose \textbf{\wdrbon{}}, a method that introduces the Minimum Bayes Risk (MBR) objective \citep{kumar-byrne-2002-minimum,kumar-byrne-2004-minimum,eikema-aziz-2020-map} as a proximity regularization term into the BoN to mitigate the reward hacking problem.\footnote{\wdrbon{} was referred to as RBoN$_\text{WD}$ in an earlier version of this manuscript.}
The MBR objective serves as a proximity regularizer by its nature which we show in Section~\ref{sec:regularizer}.
Instead of optimizing the raw reward score, we optimize a sum of the reward score and a regularization term. \wdrbon{} can tune the regularization strength by the hyperparameter $\beta$, similar to the proximity regularization in RLHF and DPO.

We evaluate the performance of \wdrbon{} on the AlpacaFarm \citep{NEURIPS2023_5fc47800} and Anthropic's hh-rlhf datasets \cite{bai2022training} and show that it outperforms the performance of vanilla BoN in a wide range of settings.
We also use \wdrbon{} to generate a pairwise preference learning dataset and show that a model trained by DPO on a dataset generated with \wdrbon{} outperforms a model trained on a dataset generated with vanilla BoN.

\section{Background}
First, we give an overview of preference learning algorithms including RLHF and DPO. Then we introduce the decoding-time alignment algorithm, BoN sampling.

\subsection{Preference Learning}

Let $\mathcal{D}$ be a set of instruction, response pair, and preference over response pair: $\mathcal{D} = \{x^{(i)}, y^{(i)}_w, y^{(i)}_l\}_{i=1}$.
RLHF uses the learned reward function to train the language model. Typically, the RL process is formulated as the following optimization problem:
\begin{align}
    \argmax_{\pi} \bE_{x \sim \mathcal{D}} &\bE_{y \sim \pi(y | x)}[R(x, y)] \nonumber\\
    &- \beta \mathbb{D}_\mathrm{KL}[\pi(\cdot | x) || \pi_\mathrm{ref}(\cdot | x)], \label{eq:ppo-loss}%
\end{align}
where $\beta$ is a hyperparameter that controls the proximity to the base reference model $\pi_\mathrm{ref}$.
The proximity regularization term $\mathbb{D}_\mathrm{KL}$ is important to prevent the model from deviating too far from the base model.
Since the objective is not differentiable, reinforcement learning algorithms are used for optimization \citep{schulman2017proximal,stiennon2020,bai2022training,NEURIPS2022_b1efde53,zheng2023secrets}.

DPO trains the language model to align directly with the human preference data over the responses, so it doesn't need a separate reward model \citep{rafailov2023direct}.
Although DPO is based on supervised learning rather than reinforcement learning, it uses essentially the same loss function under the Bradley-Terry model \citep{bradley1952rank}.
The objective function of the DPO is the following:
\begin{align}
    \argmax_{\pi} \bE_{(x, y_w, y_l) \sim \mathcal{D}}[\log \sigma (\beta& \log \frac{\pi(y_w | x)}{\pi_\mathrm{ref}(y_w | x)} - \nonumber\\
    &\beta \log \frac{\pi(y_l | x)}{\pi_\mathrm{ref}(y_l | x)})],
\label{eq:dpo}
\end{align}
where $\sigma$ is the sigmoid function.
Several variants of DPO also use KL-divergence as proximity regularization \citep{azar2023general,liu2024lipo}.

Thus, both lines of work in preference optimization have proximity regularization in common to keep the model $\pi$ close to the reference model $\pi_\mathrm{ref}$.

\subsection{Best-of-N (BoN) Sampling}
While many methods have been proposed for learning human preferences, a simple, popular, and well-performing method for preference optimization remains Best-of-N (BoN) sampling \citep{stiennon2020,nakano2022webgpt}.
Let $x$ be an input prompt to the language model $\pi_\mathrm{ref}$. Let $\hypos$ be $N$ responses drawn from $\pi_\mathrm{ref}(\cdot | x)$.
BoN sampling selects the response with the highest reward score according to the proxy reward model $R$:
\begin{equation}   
    y_{\mathrm{BoN}}(x) = \argmax_{y \in \hypos} R(x, y).
\end{equation}
The advantages of BoN over preference learning methods are as follows.
First, BoN is simple. It does not require any additional training of the language model. While learning-based alignment methods need to train the LLM, BoN can be applied on the fly. Every time human preferences are updated, learning-based methods must retrain the LLM to adapt to them. On the other hand, BoN only requires an update of the reward model and does not require the training of the LLM, which is the most expensive process.
Second, BoN is an effective strategy in its own right. Several previous works have shown that BoN sampling can outperform learning-based alignment methods \citep{pmlr-v202-gao23h,eisenstein2023helping,mudgal2024controlled,gui2024bonbon}. 
Third, BoN is applicable to a black-box model where fine-tuning is not available. BoN does not require access to the model itself and is applicable using the output sequences from the black-box model. 
In summary, BoN is a practical and efficient alignment strategy that complements the shortcomings of learning-based strategies and is worthy of investigation.

\subsection{Minimum Bayes Risk Decoding}
\textbf{MBR decoding} \citep{kumar-byrne-2002-minimum,kumar-byrne-2004-minimum,eikema-aziz-2020-map,bertsch-etal-2023-mbr} has recently gained attention as an effective decoding strategy in a variety of tasks including machine translation, text summarization, text simplification, and reasoning \cite{eikema-aziz-2020-map,eikema-aziz-2022-sampling,freitag-etal-2022-high,suzgun-etal-2023-follow,bertsch-etal-2023-mbr,heineman2024improving,li2024agents,deguchi2024mbrs}.

MBR decoding consists of the following steps. First, it samples $N$ sequences from the model ($\hypos$), similar to BoN sampling. Then, it computes the utility $U$ (e.g., similarity) between each pair of sequences in $\hypos$. Finally, it selects the sequence that maximizes the average utility between the rest of the sequences:
\begin{equation}
    y_\mathrm{MBR}(x) = \argmax_{y \in \hypos} \sum_{y' \in \hypos} \frac{1}{N} U(y, y'),
\label{eq:mbr}
\end{equation}
where \textbf{the summation represents the Bayes risk, which we refer to as the MBR objective in this work}.
MBR decoding is based on the concept of Bayes risk minimization which originates from the decision theoretic framework (\citealt{goel2000minimum}; \citealt{bickel2015mathematical}, p.27-28). Instead of selecting the output with the highest probability (maximum-a-posteriori decoding; \citealt{stahlberg-byrne-2019-nmt,Holtzman2020The}), Bayes risk minimization selects the output that is robust to the inaccuracy of the probability model \citep{meister-etal-2022-high,eikema-2024-effect}.
Bayes risk minimization is instead formalized as expected utility maximization as utility functions are more common in text generation tasks.

An alternative view of the MBR decoding is that it selects the most centered point (medoid; \citealt{rdusseeun1987clustering}) in $\hypos$ where the utility function $U$ measures the similarity between the data points \citep{jinnai2024hyperparameterfree}.
In other words, the MBR objective quantifies the proximity of the data point to the rest of the samples. 

\section{Minimum Bayes Risk Objective is a Proximity Regularizer}
\label{sec:regularizer}
Although BoN sampling is shown to be effective with a decent reward model, it is prone to the reward hacking problem under less accurate reward models \citep{NEURIPS2023_5fc47800,wen2024rethinking}.
A naive approach to prevent reward hacking is to introduce a proximity regularizer to the BoN sampling in the form of a KL-divergence term, as is common in preference learning methods \citep{stiennon2020,NEURIPS2022_b1efde53,rafailov2023direct}. However, we observe that this strategy does not improve over BoN in most cases (Appendix~\ref{apd:klrbon}). 

To this end, we propose to use the MBR objective as a proximity regularizer. 
We first show an empirical observation in Section~\ref{sec:reg-eval} that the MBR objective is correlated with the proximity of the reference policy. Then, we show an analytical result in Section~\ref{sec:reg-wd} that the MBR objective corresponds to the Wasserstein distance \citep{peyre2020computational,villani-2021-changing}, indicating that the MBR objective by its nature quantifies the proximity of the text to the reference policy.

\subsection{Empirical Evaluation}
\label{sec:reg-eval}
We evaluate the effect of the MBR objective as a proximity regularizer to keep the output closer to the center of the sample distribution.
In particular, we evaluate the correlation between the MBR objective and the closeness to the center of the sample distribution.
We run an experiment using the first 1000 entries of the training split of the AlpacaFarm \citep{NEURIPS2023_5fc47800} and Anthropic's hh-rlhf \citep{bai2022training} datasets. $N=128$ responses are sampled from 
\texttt{mistral-7b-sft-beta} (Mistral) for each instruction \citep{jiang2023mistral,tunstall2023zephyr}. The MBR objective (Eq~\ref{eq:mbr}) is calculated for each sample, and normalized to the range of [0, 1]. We use a cosine similarity of the embedding computed with \texttt{all-mpnet-base-v2} (MPNet; \citealt{reimers-gurevych-2019-sentence,song2020mpnet}):
\begin{equation}
    U(y, y') = \mathrm{cos}(\mathrm{emb}(y), \mathrm{emb}(y')),
\label{eq:utility}
\end{equation}
where $\mathrm{emb}$ denotes the embedding function.
We then compute the components of the text embedding using Principal Component Analysis (PCA; \citealt{pearson1901liii}) and Independent Component Analysis (ICA; \citealt{comon1994independent}).
Since the utility matrix between samples is likely to be approximated by a low-rank matrix \citep{trabelsi2024efficient}, the first few components are likely to be sufficient to illustrate the proximity between samples in the utility space.
We interpolate the values in component space for each instruction and then compute the average over the instructions of the dataset.

\begin{figure}[tb]
    \centering
    \includegraphics[width=\columnwidth]{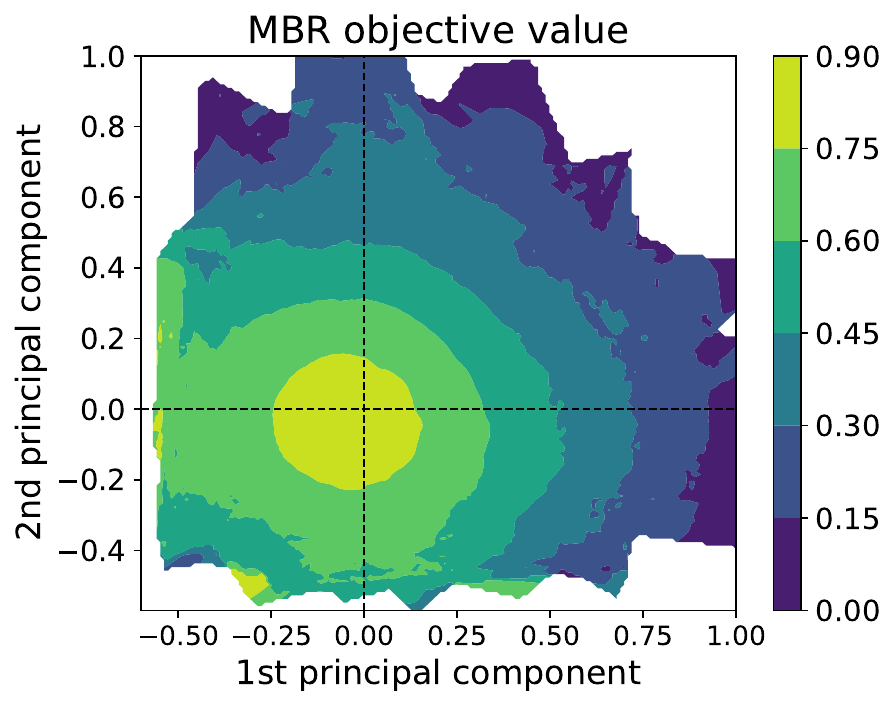}
    \caption{Mapping of the average MBR objective values to the first and the second principal components using PCA. The figure illustrates that the value of the MBR objective tends to get smaller as it moves away from the center of the distribution in the space of the principal components.}
    \label{fig:pca}
\end{figure}

\begin{table}[tb]
    \centering
    \caption{Correlation of the distance to the center point in the component space with the MBR objective (Eq~\ref{eq:mbr}) on the AlpacaFarm dataset. The mean and standard deviation of the correlation are shown in the table. The result shows that the more an output $y$ deviates from the center of the distribution, the lower the value of the MBR objective. Dim is the number of components.}
    \label{tab:pca-corr}
    \begin{tabular}{lcc}
    \toprule
    Dim  & PCA & ICA \\
    \midrule
2 & -0.5747 $\pm$ 0.1858 & -0.5621 $\pm$ 0.1830 \\
5 & -0.7494 $\pm$ 0.1329 & -0.6683 $\pm$ 0.1291 \\
10 & -0.8512 $\pm$ 0.1010 & -0.6809 $\pm$ 0.1222 \\
    \bottomrule
    \end{tabular}
\end{table}

Figure~\ref{fig:pca} shows the mapping of the average MBR objective values, with the horizontal and vertical axes showing the first and second principal components of the embeddings. Table~\ref{tab:pca-corr} shows the correlation of the distance to the center in principal component space with the MBR objective. Regardless of the dimension of the components, we observe qualitatively the same result that the correlation of the distance from the center with the MBR objective is strongly negative. On the other hand, the correlation with the log probability of the output is weak, indicating that the KL-divergence based on probability is not a reliable measure of proximity in the embedding space (Table~\ref{tab:logprob-pca} in Appendix~\ref{apd:pca}). 
The result shows that the MBR objective value becomes smaller as it moves away from the center of the distribution.
We observe the same qualitative results in Anthropic's hh-rlhf and in a machine translation dataset (WMT'21 De-En; \citealt{akhbardeh-etal-2021-findings}) which we show in Appendix~\ref{apd:pca}.

\subsection{Analytical Evaluation}
\label{sec:reg-wd}

Formally, the MBR objective corresponds to selecting the output $y$ that minimizes the Wasserstein Distance (WD; \citealt{peyre2020computational,villani2021topics}) to the sample distribution.
WD, also known as the Earth Mover's Distance (EMD; \citealt{rubner1998metric}), measures the cost required to transform one probability distribution into another. The cost function $C$ typically represents the ``distance'' or ``effort'' required to move a unit of probability mass from one location to another.
In the context of NLP, it is also called the Word Mover's Distance to evaluate the similarity between a pair of texts \citep{pmlr-v37-kusnerb15,NIPS2016_10c66082}.
For a pair of probability distributions $P$ and $Q$ over $\hypos$, WD is defined as follows:
\begin{align}
    WD&(P, Q) = \nonumber\\
    & \min_{\{\mu_{i, j}\}_{i, j} \in \mathcal{J}(P, Q)} \sum_{i = 1}^{|\hypos|} \sum_{j = 1}^{|\hypos|} \mu_{i, j} C(y_i, y_j),
\end{align}
where $C$ is the cost function that represents the dissimilarity of the elements. 
$\mathcal{J}(P, Q)$ is a set of all couplings over $P$ and $Q$ \citep{villani2021topics}:
\begin{align}
    &\mathcal{J}(P, Q) = \bigl\{ \{\mu_{i, j}\}_{i, j} : \nonumber\\
    &\sum_{i = 1}^{|\hypos|} \mu_{i, j} = Q(y_j), \sum_{j = 1}^{|\hypos|} \mu_{i, j} = P(y_i), \mu_{i, j} \geq 0 \bigr\}.
\end{align}

The objective of MBR decoding is identical to minimizing the WD to the empirical distribution of $\pi_\mathrm{ref}$.
\begin{proposition}
\label{prop:wd}
    Let the cost function $C$ for WD be $C(y, y')=-U(y, y')$ for all $y$ and $y'$. MBR decoding selects the output with the smallest WD of the sample distribution:
    \begin{align}        
        y_\mathrm{MBR}(x) &= \argmax_{y \in \hypos} \sum_{y' \in \hypos} \frac{1}{N} U(y, y') \\
        &= \argmin_{y \in \hypos} WD(\pi_y(\cdot \mid x), \hat{\pi}_\mathrm{ref}(\cdot \mid x)),
    \end{align}
    where $\pi_y$ is a policy that outputs $y$ with a probability of 1 and $\hat{\pi}_\mathrm{ref}$ is the empirical distribution constructed from $\hypos$: $\hat{\pi}_\mathrm{ref}(y \mid x) = \frac{1}{N} \sum_{y_i \in \hypos} \mathbb{I}[y_i = y]$.
\end{proposition}
\begin{proof}
The proof is in Appendix~\ref{apd:wd}.    
\end{proof}
The proposition shows that the MBR objective measures the WD of the output selection strategy to the sample distribution of the reference policy. Maximizing the objective results in selecting an output that is closest to the sample distribution of the reference policy with respect to the utility function $U$.

\paragraph{Summary.}
Both the empirical and analytical results show that the MBR objective serves as a proximity regularizer to penalize an output that is less representative of the samples from the reference policy, as measured by the utility function.

\section{MBR-Best-of-N (MBR-BoN) Sampling}
\label{sec:mbrbon}

We propose \textbf{MBR-Best-of-N (MBR-BoN) sampling}, a variant of BoN sampling with an MBR objective as the proximity regularizer, to mitigate the reward hacking problem of BoN sampling.
\wdrbon{} uses the MBR objective as the proximity regularizer: 
\begin{align}
    y&_{\textrm{\wdrbon}}(x) = \nonumber\\
    &\argmax_{y \in \hypos} R(x, y) + \beta \sum_{y' \in \hypos} \frac{1}{N} U(y, y'),
\label{eq:wdbon}
\end{align}
where $\beta$ is a hyperparameter to adjust the strength of the regularization.
As the MBR objective corresponds to the WD between the resulting policy and the reference policy (Section~\ref{sec:regularizer}), it serves as a proximity regularizer to ensure that the resulting policy is close to the reference policy $\pi_\mathrm{ref}$.

The hyperparameter $\beta$ controls the tradeoff between the reward and proximity to the reference model.
Using a small $\beta$ makes the output more aligned with the proxy reward, with $\beta=0$ recovering vanilla BoN sampling. A larger $\beta$ makes the output closer to the behavior of the reference model $\pi_\mathrm{ref}$, with $\beta=+\infty$ recovering MBR decoding. 

\paragraph{Advantage of WD over KL-divergence.}
WD is a more suited regularizer than KL-divergence for inference-time algorithms where the number of samples is very small. While KL-divergence is useful for training-time alignment algorithms, it poses several challenges for inference-time algorithms with limited samples.

Theoretically, any high confidence lower bound on KL-divergence requires a sample size exponential in the value of KL-divergence \cite{mcallester2020formal}. This suggests that estimating KL-divergence is unreliable in finite-sample settings. For example, for the first instance of the AlpacaEval instruction (\textit{What are the names of some famous actors that started their careers on Broadway?}), the KL-divergence of the randomly sampled 128 responses from Mistral has a minimum of 627, a maximum of 5870, a mean of 1854, and a standard deviation of 1039.

Moreover, KL-divergence is sensitive to small differences in the sequences. Specifically, KL-divergence can be large even if the underlying sequences differ very little. For example, the two sentences: ``\textit{Yes I will do it.}'' and ``\textit{Yes I'll do it.}'' are considered completely different data instances when computing KL-divergence. Conversely, WD considers them to be quite similar data instances. This is because the WD uses the utility function to quantify the divergence and represents the difference between two distributions in terms of the semantic distance between the sequences. This makes the WD a more robust measure against the minor variances that naturally occur in natural language texts.
See Appendix \ref{apd:klrbon} for experimental evaluation of using KL-divergence as a regularization term.

In addition to being a good proximal regularizer, the MBR objective is a useful text generation objective in its own right.
The objective is shown to be effective, outperforming MAP decoding in a variety of text generation tasks, including instruction-following task \citep{suzgun-etal-2023-follow,bertsch-etal-2023-mbr,li2024agents}.

\section{Experiments}

We evaluate the performance of \wdrbon{} for two use cases. First, we evaluate the performance of \wdrbon{} for decoding time alignment (Section \ref{sec:decoding}). 
Then, we evaluate \wdrbon{} as a sampling strategy to generate a preference learning dataset to be used for DPO (Section \ref{sec:learning}). 

\subsection{\wdrbon{} for Decoding-Time Alignment}
\label{sec:decoding}

\begin{figure*}[t]
    \centering
    \subfloat[AlpacaFarm - Proxy: SHP-Large]{\includegraphics[width=0.32\textwidth]{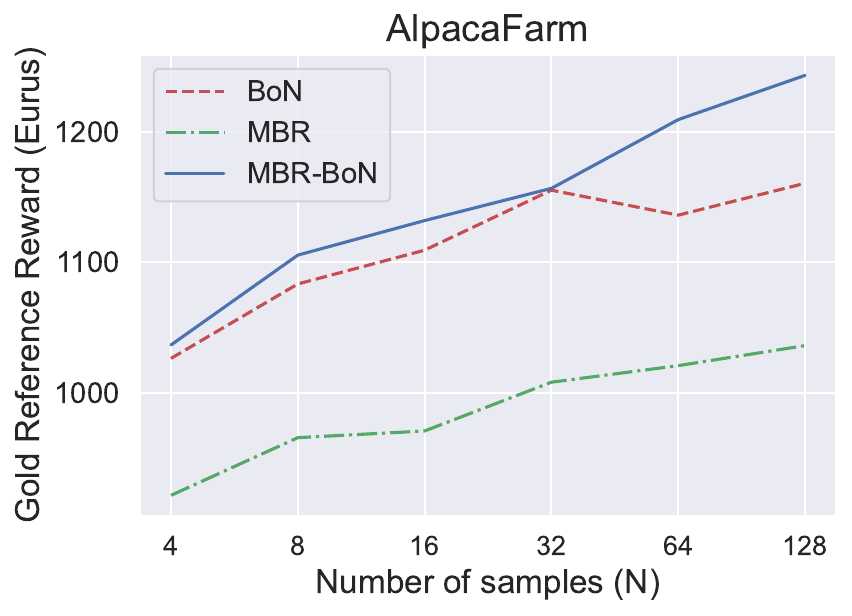}}
    \subfloat[AlpacaFarm - Proxy: SHP-XL]{\includegraphics[width=0.32\textwidth]{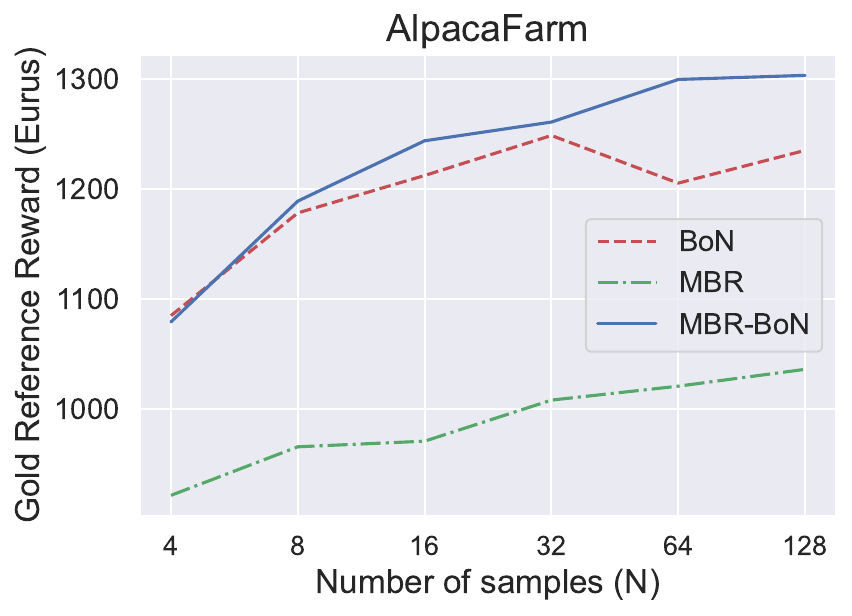}} 
    \subfloat[AlpacaFarm - Proxy: OASST]{\includegraphics[width=0.32\textwidth]{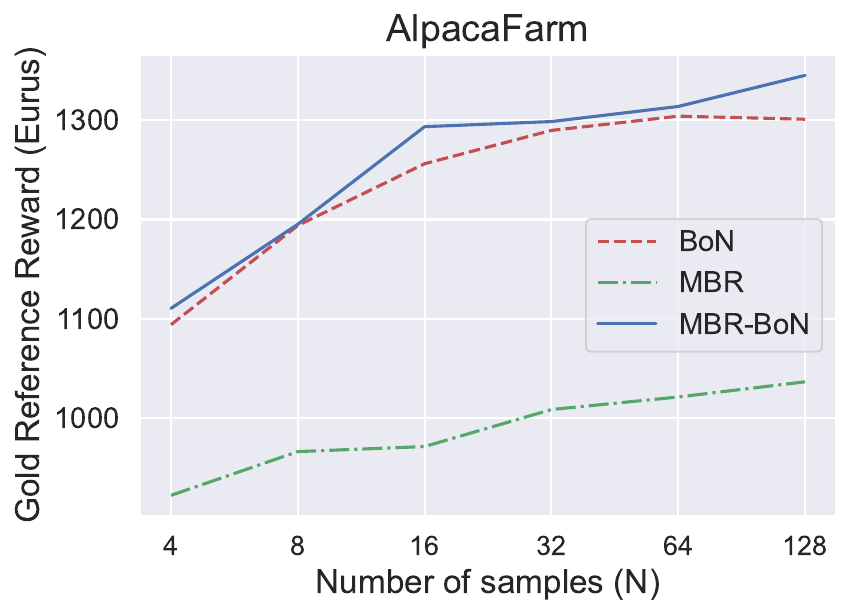}} \\
    \subfloat[Helpfulness - Proxy: SHP-Large]{\includegraphics[width=0.32\textwidth]{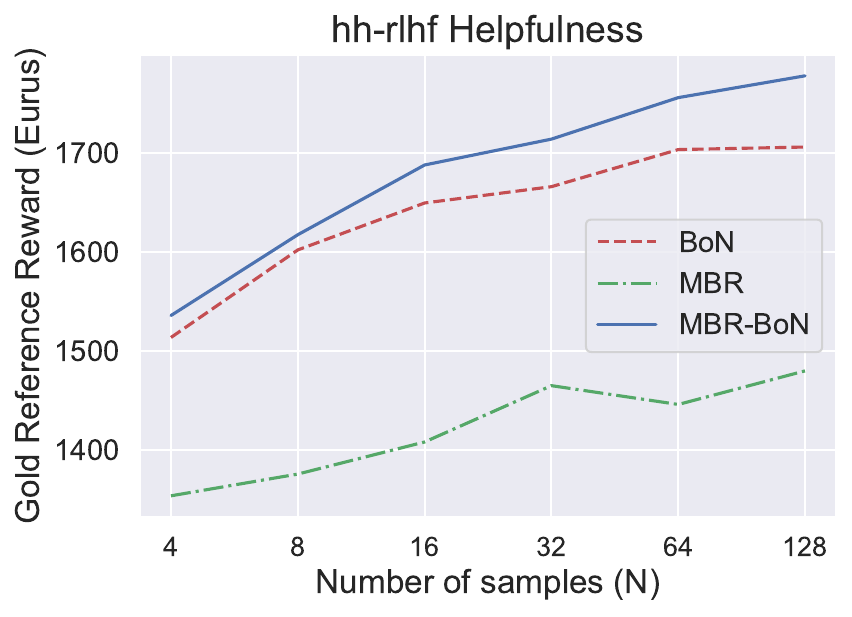}}
    \subfloat[Helpfulness - Proxy: SHP-XL]{\includegraphics[width=0.32\textwidth]{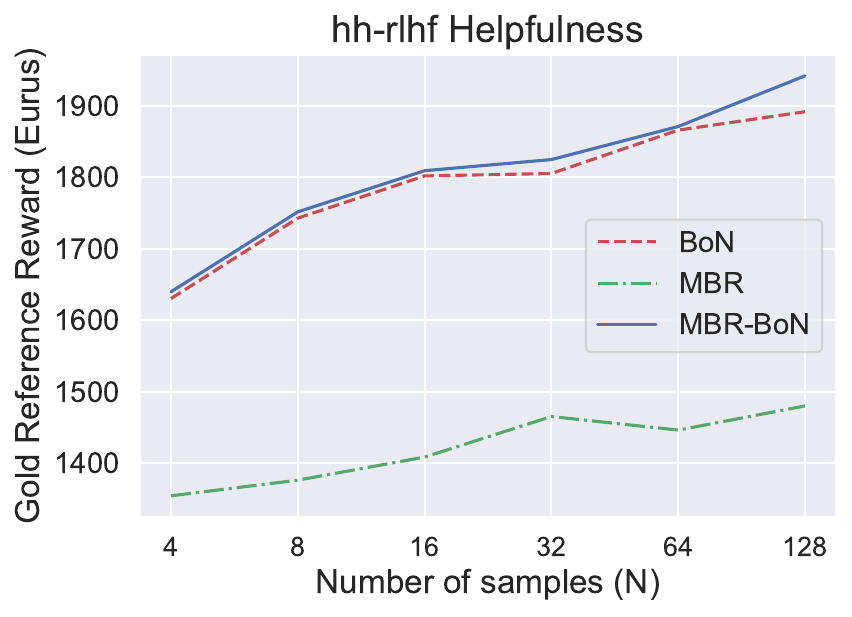}}
    \subfloat[Helpfulness - Proxy: OASST]{\includegraphics[width=0.32\textwidth]{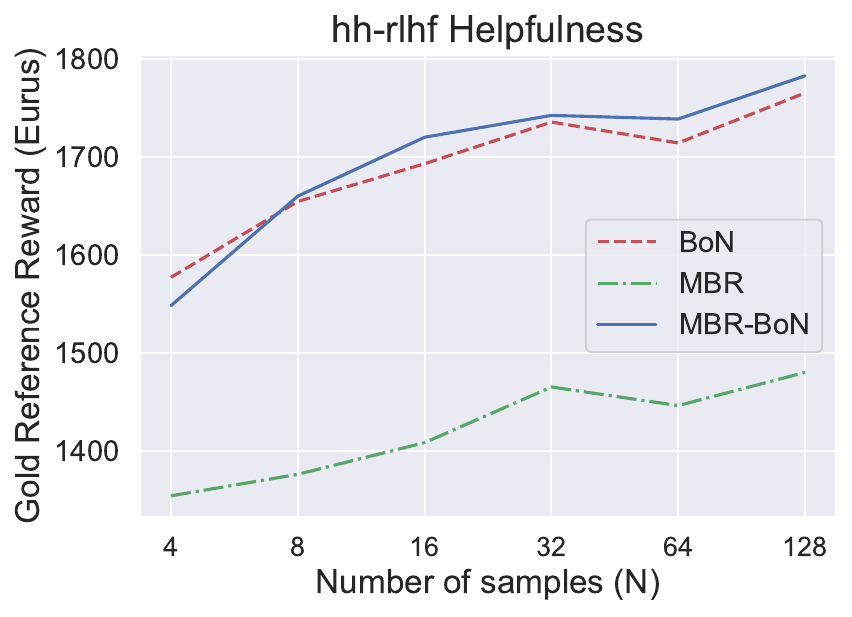}} \\
    \subfloat[Harmlessness - Proxy: SHP-Large]{\includegraphics[width=0.32\textwidth]{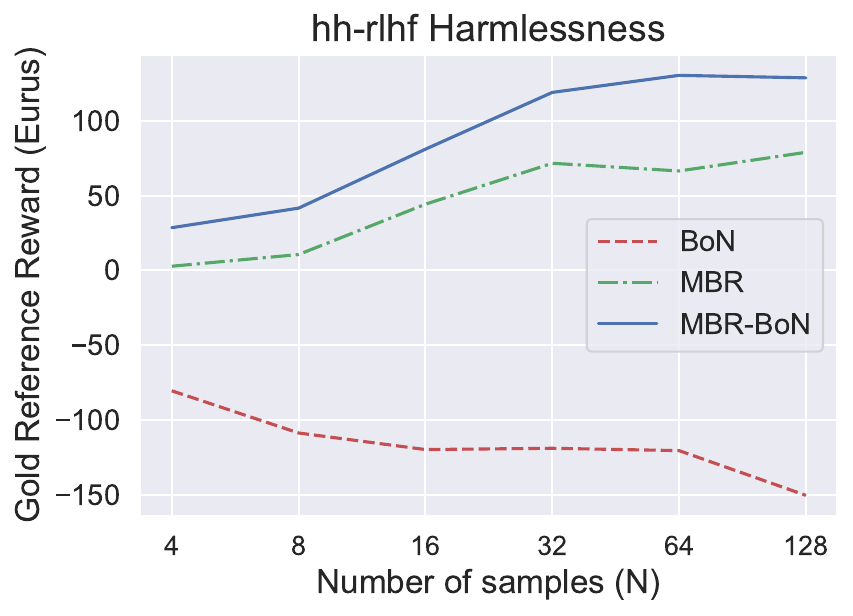}}
    \subfloat[Harmlessness - Proxy: SHP-XL]{\includegraphics[width=0.32\textwidth]{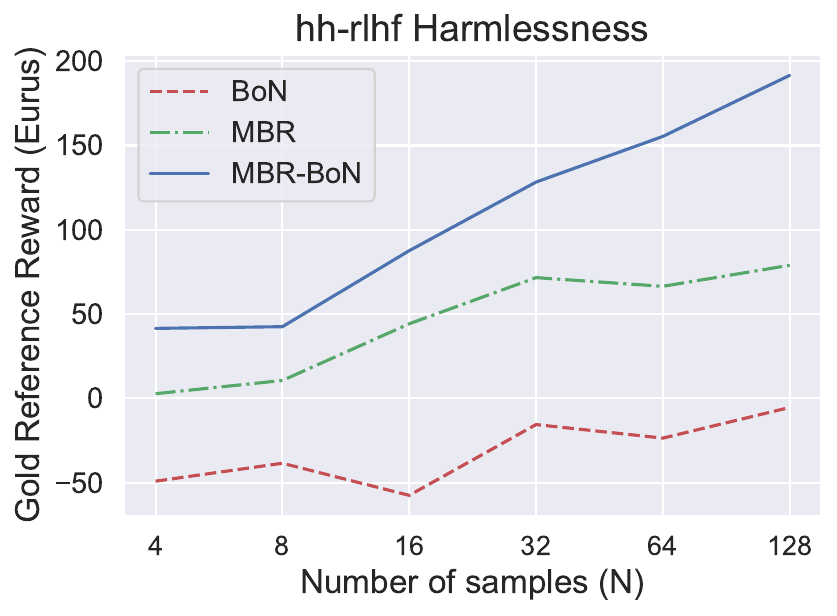}}
    \subfloat[Harmlessness - Proxy: OASST]{\includegraphics[width=0.32\textwidth]{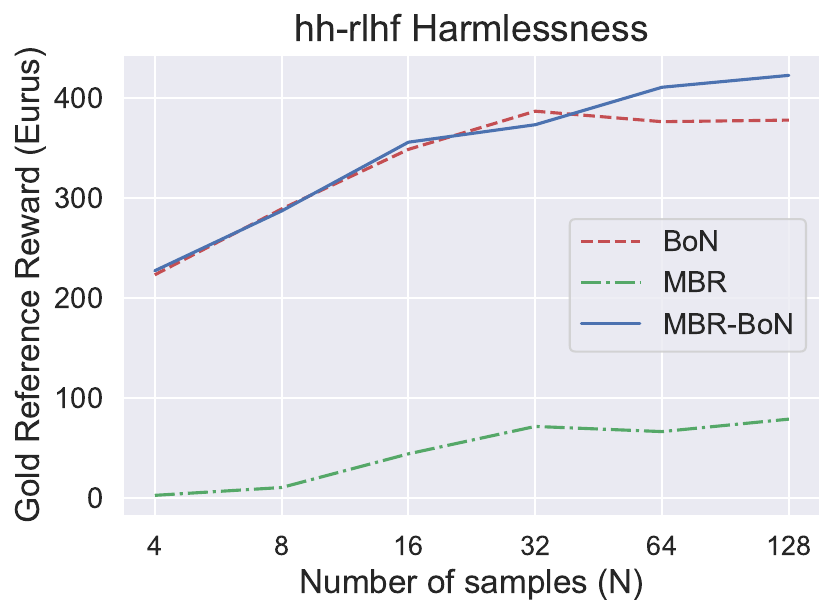}}
    \caption{Evaluation BoN, MBR, and \wdrbon{} on the AlpacaFarm, hh-rlhf Helpfullness, and hh-rlhf Harmlessness datasets. Mistral is used as the language model.}
    \label{fig:mistral-alpaca}
\end{figure*}

\begin{figure*}[t]
    \centering
    \subfloat[Proxy: SHP-Large]{\includegraphics[width=0.32\textwidth]{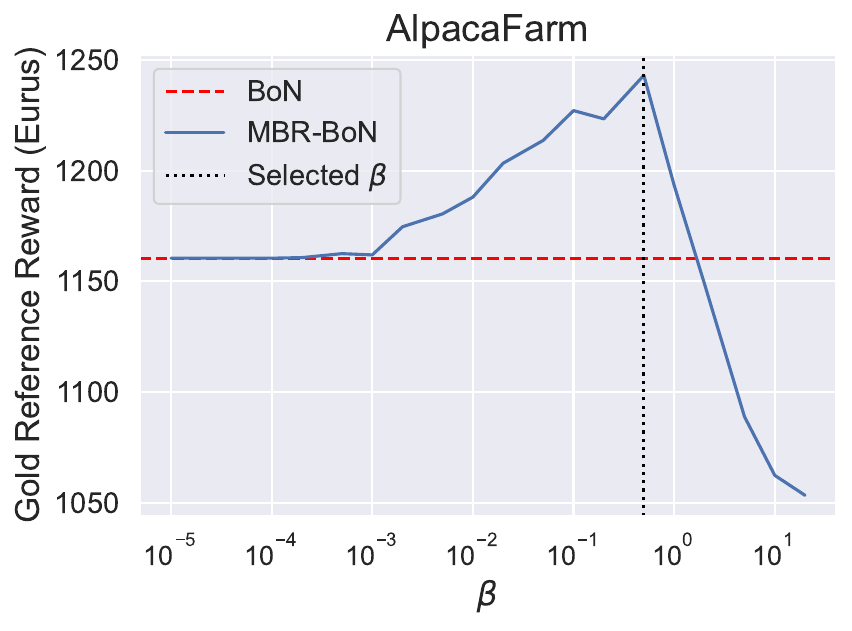}}
    \subfloat[Proxy: SHP-XL]{\includegraphics[width=0.32\textwidth]{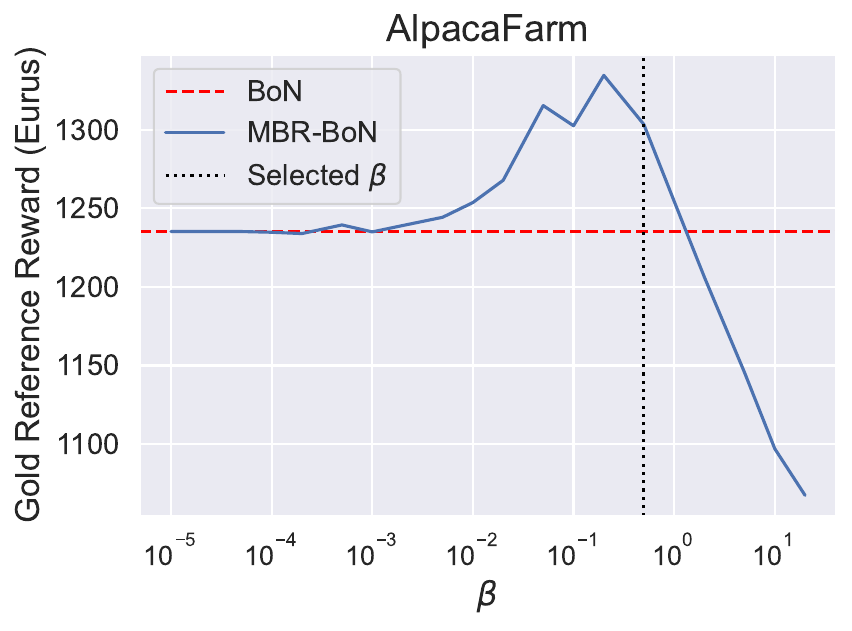}} 
    \subfloat[Proxy: OASST]{\includegraphics[width=0.32\textwidth]{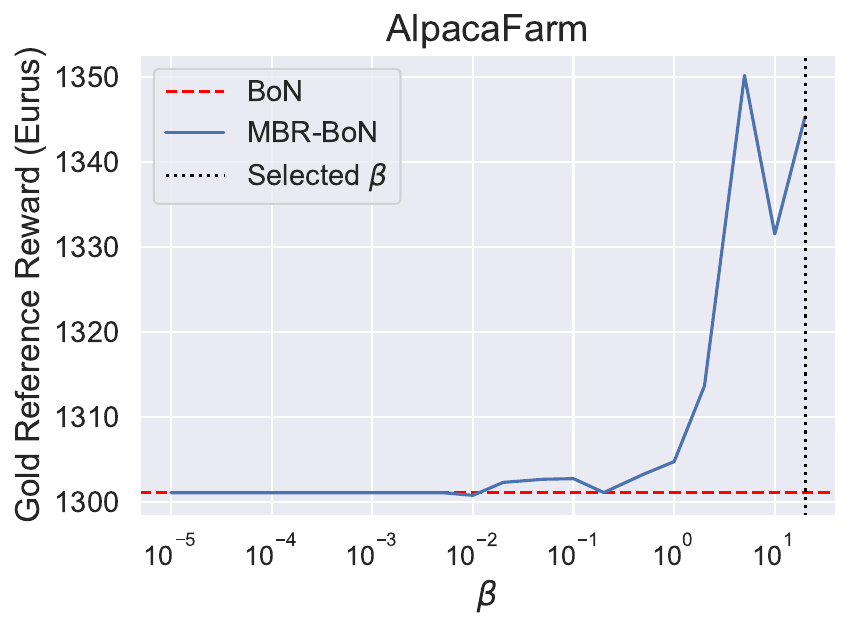}} 
    \caption{Evaluation of the \wdrbon{} using Mistral on the AlpacaFarm dataset with varying regularization strength $\beta$. The number of samples is $N=128$.}
    \label{fig:mistral-alpaca-opt}
\end{figure*}

\paragraph{Setup.}
The evaluation is conducted using the AlpacaFarm \citep{NEURIPS2023_5fc47800} and Anthropic's hh-rlhf datasets \citep{bai2022training}.
For the AlpacaFarm dataset, we use the first 1000 entries of the train split (\texttt{alpaca\_human\_preference}) as the development set and the whole evaluation split (\texttt{alpaca\_farm\_evaluation}) (805 instructions) as a test dataset.
For Anthropic's datasets, we conduct experiments on the \texttt{helpful-base} (Helpfulness) and \texttt{harmless-base} (Harmlessness) subsets separately. For each subset, we use the first 1000 entries of the train split as the development set and the first 1000 entries of the test split as a test dataset. 
We use \texttt{mistral-7b-sft-beta} (Mistral) and \texttt{dolly-v2-3b} (Dolly) as the language models \citep{jiang2023mistral,tunstall2023zephyr,DatabricksBlog2023DollyV2}.

To evaluate \wdrbon{} under various conditions, we use SHP-Large, SHP-XL \citep{pmlr-v162-ethayarajh22a}, and OASST \citep{kopf2023openassistant} as proxy reward models. 
We use Eurus as a gold reference reward model as it is one of the most accurate reward models \citep{lambert2024rewardbench,zhou2024rmb} and is open-source which makes the experiments reproducible. The results using other reward models as a gold reference are reported in Appendix~\ref{apd:trainedrm} and \ref{apx:beta}.
The average Spearman's rank correlation coefficient $\rho$ \citep{spearman1904proof} to the gold reference reward (Eurus) is reported in Table \ref{tab:correlation}. 

\begin{table}[tb]
    \centering
    \caption{Average Spearman's rank correlation coefficient of the proxy reward models to the gold reference reward model (Eurus) on AlpacaFarm.}
    \label{tab:correlation}
    \begin{tabular}{cc}
    \toprule
        Proxy reward & Correlation Coefficient \\\midrule
        SHP-Large & 0.32 \\
        SHP-XL    & 0.39 \\
        OASST     & 0.40 \\
    \bottomrule
    \end{tabular}
\end{table}

We compare the performance of BoN, MBR, and \wdrbon{}. 
We sample up to $N=128$ responses per instruction using nucleus sampling and select the output using the algorithms. We set the top-$p$ to be $p=0.9$ and the temperature to be $T=1.0$ for the nucleus sampling \citep{Holtzman2020The}. For a fair comparison, we use the same set of $N$ responses for all algorithms.
We use the Sentence BERT model \citep{reimers-gurevych-2019-sentence} based on MPNet \citep{song2020mpnet} to compute the sentence embedding for MBR and \wdrbon{}.

\wdrbon{} use the development set to select the optimal $\beta$.
For each pair of a proxy reward and a gold reference reward, we run \wdrbon{} with $\beta \in \{10^{-6}, 2 \cdot 10^{-6}, 5 \cdot 10^{-6}, 10^{-5},..., 2 \cdot 10^{1}\}$ and pick the best performing $\beta$ for $N=128$. We use the same $\beta$ for all $N$ in evaluation. See Appendix~\ref{apx:beta} and \ref{apd:devset} for the ablation study on the regularization strength $\beta$.

\paragraph{Results.}
Figure~\ref{fig:mistral-alpaca} shows the performance of BoN, MBR, and \wdrbon{} using Mistral as a language model, evaluated by Eurus score. See Appendix~\ref{apx:beta} for the result of Dolly.
Overall, \wdrbon{} outperforms BoN and MBR in most of the settings, showing that the method is effective in a wide range of tasks.
Figure~\ref{fig:mistral-alpaca-opt} shows the performance of \wdrbon{} with $N=128$ and with varying regularization strength $\beta$. The vertical line shows the $\beta$ selected using the development set.
Overall, \wdrbon{} outperforms BoN in a wide range of $\beta$ and is relatively robust to the choice of the $\beta$. 

As expected, we observe that \wdrbon{} have lower scores with respect to the proxy reward than BoN (Appendix~\ref{apx:beta}). The regularization term effectively mitigates the reward hacking of the BoN, resulting in a higher score in the gold reference score (Eurus).

\begin{table}
    \caption{The values of hyperparameter $\beta$ used by \wdrbon{} determined using the development set.}
    \label{tab:beta}
\adjustbox{max width=\columnwidth}{
    \centering
    \begin{tabular}{cccc}
    \toprule
        Dataset      & SHP-Large & SHP-XL & OASST \\\midrule
        AlpacaFarm   & 0.5 & 0.5 &  20.0 \\
        Helpfulness  & 0.05 & 0.1 & 20.0 \\
        Harmlessness & 2.0 & 2.0 & 20.0 \\
    \bottomrule
    \end{tabular}
}
\end{table}

\paragraph{Choice of Regularization strength.}
Table~\ref{tab:beta} summarizes the regularization strength $\beta$ picked using the development set. The optimal value of $\beta$ depends on the choice of the language model, dataset, and proxy reward model, which requires the use of the development set to tune the hyperparameter $\beta$. 
Still, we find that the amount of development data we need for hyperparameter tuning is small. Our post-hoc analysis on the size of the developement set shows that with as little as 10 instances it already outperforms BoN and also finds $\beta$ close to the optimal $\beta$ (Appendix~\ref{apd:devset}).
Also note that the computational cost of tuning the hyperparameter for \wdrbon{} is marginal compared to that of RLHF or DPO as it does not involve any training of the language model or reward model.
Running \wdrbon{} with different $\beta$ only requires the computation of Eq.~\ref{eq:wdbon} with the different $\beta$.

\subsection{\wdrbon{} for Generating Preference Learning Dataset}
\label{sec:learning}

\begin{figure*}
    \centering
    \subfloat[Alpaca]{
    \includegraphics[width=0.325\textwidth]{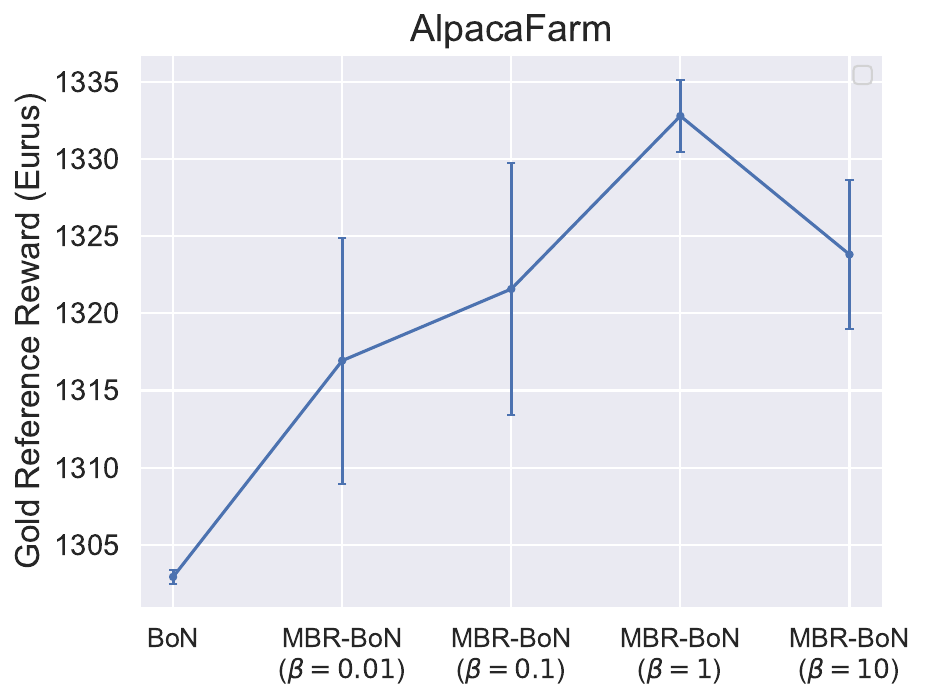}
    }
    \subfloat[Helpfulness]{
    \includegraphics[width=0.325\textwidth]{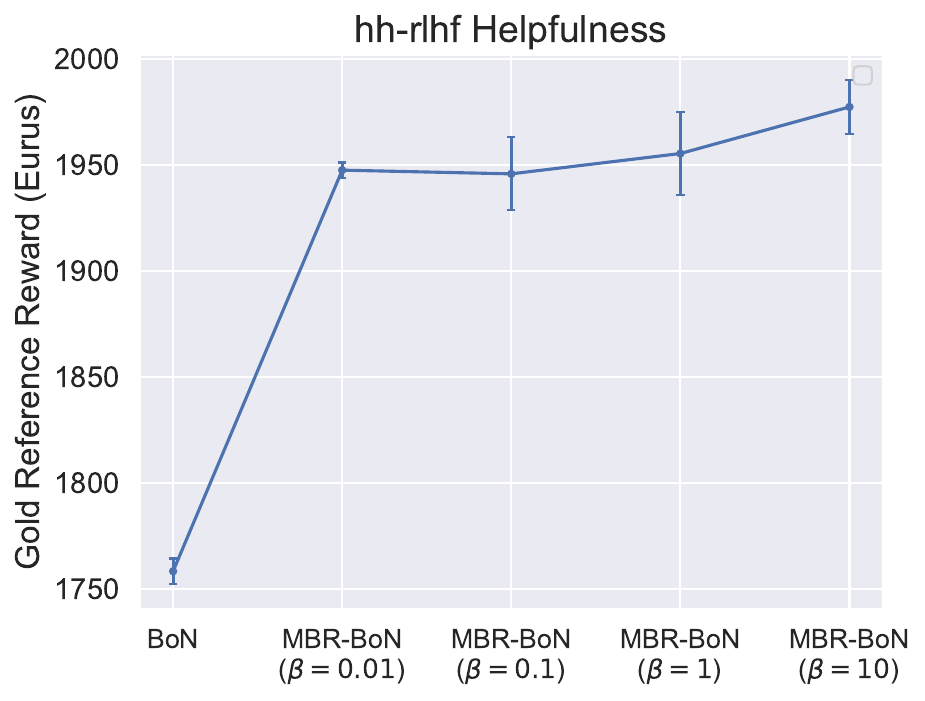}
    }
    \subfloat[Harmlessness]{
    \includegraphics[width=0.325\textwidth]{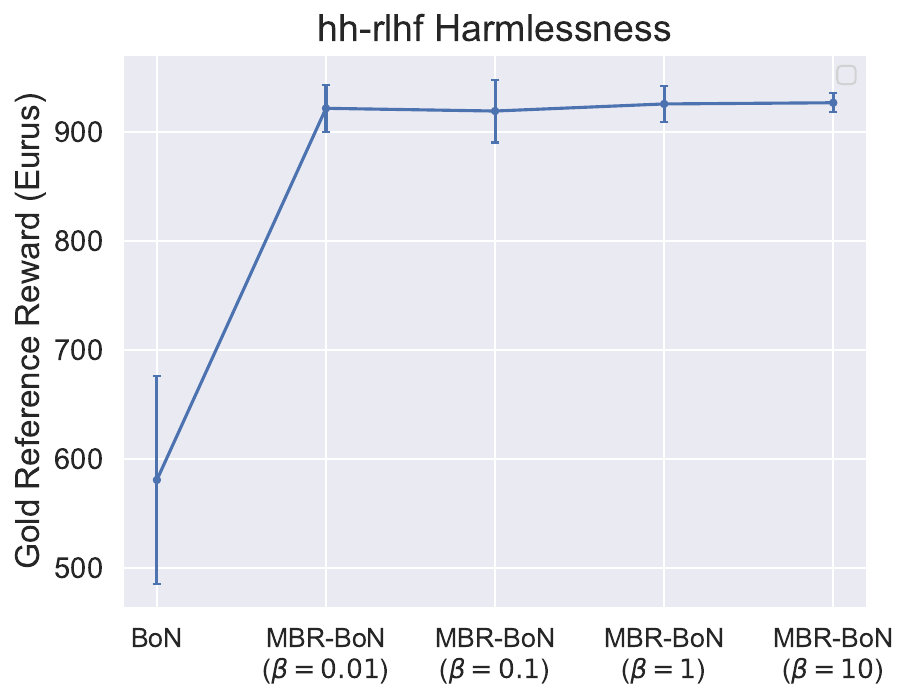}
    }
    \caption{Evaluation of the DPO using \wdrbon{} to generate the preference dataset. OASST is used as the proxy reward model to generate the preference dataset, and Eurus is used as the gold reference reward. The line represents the mean of three runs, and the error bar shows the standard error of the mean.}
    \label{fig:dpo}
\end{figure*}

Previous work has shown that BoN sampling is an effective strategy for generating an efficient preference dataset \citep{xu2023things,yuan2024selfrewarding,pace2024westofn}.
They show that the efficiency of pairwise preference learning is improved by using the best and worst responses according to the reward model as the chosen and rejected responses.
We evaluate the performance of DPO \citep{rafailov2023direct} using the response selected by \wdrbon{} as the chosen response and the response with the lowest reward score as the rejected response.
\paragraph{Setup.}
We sample 128 responses for each instruction in the training dataset and use the response selected by \wdrbon{} or BoN as the chosen response and the response with the lowest reward as the rejected response. We use all 9.69k instructions from AlpacaFarm and the first 5k instructions from each of the Helpfulness and Harmlessness subsets to train a model for the hh-rlhf datasets.
We use Mistral as the language model to generate the pairwise preference dataset and train it using the generated dataset \citep{jiang2023mistral,tunstall2023zephyr}.

OASST is used as a proxy reward model and Eurus is used for evaluation \citep{kopf2023openassistant,yuan2024advancing}.
We train a model with DPO using Low-Rank Adaptation (LoRA; \citealt{hu2022lora,sidahmed2024perl}).
The trained models are evaluated using the evaluation split of the AlpacaFarm dataset. 
Other hyperparameters are described in Appendix~\ref{apx:hyperparams}.

\paragraph{Results.}
Figure~\ref{fig:dpo} shows the performance of models trained using \wdrbon{} and BoN to generate a pairwise preference dataset. The models trained with \wdrbon{} outperform a model trained with BoN. 
According to Figure~\ref{fig:mistral-alpaca}, \wdrbon{} generates higher quality response texts than BoN with respect to the gold reference reward. We expect the models trained by DPO on the higher quality responses to achieve higher quality generation. 
In addition, \wdrbon{} generates on-policy responses that are representative of the reference policy (Section~\ref{sec:reg-eval} and Appendix~\ref{apd:pca}) which is shown to be one of the important characteristics of efficient preference datasets \citep{chang2024dataset,guo2024direct,xu2024dpo,tajwar2024preference,tang2024understanding}. 
Thus, we postulate that by generating high-quality and on-policy responses, models aligned with responses generated by \wdrbon{} outperform that of BoN.

We additionally evaluate the performance of DPO with responses generated by random sampling (i.e., BoN with $N=2$).
According to the Eurus reward model, the scores were as follows: 1140.0 for Alpaca, 1556.9 for Helpfulness, and 433.3 for Harmlessness. Both BoN and MBR-RoN significantly outperform random sampling. This result aligns with prior work showing BoN outperforming random sampling \citep{xu2023things,yuan2024selfrewarding,pace2024westofn}.

The result shows the potential of \wdrbon{} as a tool for generating pairwise preference datasets for preference learning. See Appendix~\ref{apx:gpt4} for the results using GPT-4o as an evaluator.

\section{Related Work}
\paragraph{Mitigating reward hacking at inference time.}
Using proximity regularization is not the only way to mitigate the reward hacking problem.
Several studies have explored the use of multiple rewards. \cite{mudgal2023controlled,coste2024reward,rame2024warm} propose to ensemble multiple reward functions to mitigate reward hacking.
Several studies have investigated training models by the reward functions and combining by interpolating the parameters \citep{RameCDGSSC23,jang2023personalized} or ensembling the model \citep{mitchell2024an,shi2024decodingtime}.
Our approach is applicable to any proxy reward model, so it can be combined with these methods.

\paragraph{MBR for training a model.}
Prior work has discovered that MBR decoding for LLM is useful for inference and for generating preference dataset in machine translation tasks \citep{farinhas-etal-2023-empirical,ramos-etal-2024-aligning}.
\citet{finkelstein2024mbr,guttmann-etal-2024-chasing}  uses the output generated by MBR decoding for supervised fine-tuning to improve the generation quality of a machine translation model.
\citet{yang-etal-2024-direct} trains a model by DPO to prefer outputs with higher MBR objective values than lower ones.
\citet{tomani-etal-2024-quality} trains the machine translation model to predict the quality of the generation so that it can improve its own generation using the estimate.
The novelty of our work is to introduce the MBR objective combined with BoN sampling for language model alignment, improving both the text generation and the training using the generated texts.

\section{Conclusions}

We propose \wdrbon{}, a variant of BoN sampling with MBR objective as a proximity regularizer to mitigate the reward hacking problem. We show that the MBR objective is a proximity regularizer by its nature and show it in the experiments.  
We evaluate the performance of \wdrbon{} using the AlpacaFarm and Anthropic's hh-rlhf datasets. The result shows that \wdrbon{} outperforms BoN when the proxy reward is weakly correlated with the reference objective.
As an application of the method, we also show that \wdrbon{} is an effective strategy for generating a preference dataset for DPO.

We believe that \wdrbon{} will be a practical choice for future decoding-time alignment methods because of its applicability and performance improvements.

\section{Limitations}

The drawback of the proposed method is that it requires a development set to tune the hyperparameter. Given that there is no clear strategy to pick the $\beta$ parameter even for RLHF and DPO, we speculate that it would be challenging to develop a strategy to find an effective $\beta$ automatically.
Still, the hyperparameter tuning of \wdrbon{} is much more computationally efficient than that of RLHF and DPO as it does not involve any training procedures. In fact we observe that around 10 instances are enough to find a near-optimal choice of $\beta$ (Appendix~\ref{apd:devset}).

One of the critical limitations of MBR decoding is its generation speed. It requires computing a utility function that is quadratic to the number of samples. 
\wdrbon{} inherits the same limitation because it is derived from MBR. Given that recent work \citep{cheng-vlachos-2023-faster,jinnai2024hyperparameterfree,deguchi-etal-2024-centroid,vamvas-sennrich-2024-linear} has improved the computational complexity of MBR decoding to linear in the number of samples, we are optimistic that the overhead of MBR-BoN will be reduced in the future.

We use automated evaluation metrics to evaluate the models. Although we use one of the most accurate publicly available reward models and GPT-4o to evaluate the performance of the models \cite{yuan2024advancing,lambert2024rewardbench}, it would be desirable to perform a human evaluation.

Our experiments on preference learning are limited to the evaluation of DPO. Evaluation of \wdrbon{} for other preference optimization algorithms is future work \cite{azar2023general,liu2024lipo,ethayarajh2024kto,xu2024contrastive,morimura2024filtered,hong2024orpo,meng2024simpo,park2024disentangling}.

\section{Impact Statement}
\label{apd:ethics}

We believe that this work will have a positive impact by providing a method for fine-tuning an LLM with limited annotation resources, allowing for alignment with less representative communities in language resources.
LLMs would be more useful if we could prevent them from reward hacking, even when the annotation for the task is limited.

\section*{Acknowledgments}
We thank the anonymous reviewers for their insightful comments and suggestions.
Kaito Ariu's research is supported by JSPS KAKENHI Grant No. 23K19986. 

\bibliography{ms,ms2,culture,anthology}

\begin{thebibliography}{112}
\providecommand{\natexlab}[1]{#1}

\bibitem[{Afzoon et~al.(2024)Afzoon, Naseem, Beheshti, and Jamali}]{afzoon2024persobench}
Saleh Afzoon, Usman Naseem, Amin Beheshti, and Zahra Jamali. 2024.
\newblock Persobench: Benchmarking personalized response generation in large language models.
\newblock \emph{arXiv preprint arXiv:2410.03198}.

\bibitem[{Agrawal et~al.(2024)Agrawal, De~Souza, Rei, Farinhas, Faria, Fernandes, Guerreiro, and Martins}]{agrawal2024modeling}
Sweta Agrawal, Jos{\'e} G.~C. De~Souza, Ricardo Rei, Ant{\'o}nio Farinhas, Gon{\c{c}}alo Faria, Patrick Fernandes, Nuno~M Guerreiro, and Andre Martins. 2024.
\newblock \href {https://doi.org/10.18653/v1/2024.emnlp-main.803} {Modeling user preferences with automatic metrics: Creating a high-quality preference dataset for machine translation}.
\newblock In \emph{Proceedings of the 2024 Conference on Empirical Methods in Natural Language Processing}, pages 14503--14519, Miami, Florida, USA. Association for Computational Linguistics.

\bibitem[{Akhbardeh et~al.(2021)Akhbardeh, Arkhangorodsky, Biesialska, Bojar, Chatterjee, Chaudhary, Costa-jussa, Espa{\~n}a-Bonet, Fan, Federmann, Freitag, Graham, Grundkiewicz, Haddow, Harter, Heafield, Homan, Huck, Amponsah-Kaakyire, Kasai, Khashabi, Knight, Kocmi, Koehn, Lourie, Monz, Morishita, Nagata, Nagesh, Nakazawa, Negri, Pal, Tapo, Turchi, Vydrin, and Zampieri}]{akhbardeh-etal-2021-findings}
Farhad Akhbardeh, Arkady Arkhangorodsky, Magdalena Biesialska, Ond{\v{r}}ej Bojar, Rajen Chatterjee, Vishrav Chaudhary, Marta~R. Costa-jussa, Cristina Espa{\~n}a-Bonet, Angela Fan, Christian Federmann, Markus Freitag, Yvette Graham, Roman Grundkiewicz, Barry Haddow, Leonie Harter, Kenneth Heafield, Christopher Homan, Matthias Huck, Kwabena Amponsah-Kaakyire, Jungo Kasai, Daniel Khashabi, Kevin Knight, Tom Kocmi, Philipp Koehn, Nicholas Lourie, Christof Monz, Makoto Morishita, Masaaki Nagata, Ajay Nagesh, Toshiaki Nakazawa, Matteo Negri, Santanu Pal, Allahsera~Auguste Tapo, Marco Turchi, Valentin Vydrin, and Marcos Zampieri. 2021.
\newblock \href {https://aclanthology.org/2021.wmt-1.1} {Findings of the 2021 conference on machine translation ({WMT}21)}.
\newblock In \emph{Proceedings of the Sixth Conference on Machine Translation}, pages 1--88, Online. Association for Computational Linguistics.

\bibitem[{Amodei et~al.(2016)Amodei, Olah, Steinhardt, Christiano, Schulman, and Mané}]{amodei2016concrete}
Dario Amodei, Chris Olah, Jacob Steinhardt, Paul Christiano, John Schulman, and Dan Mané. 2016.
\newblock Concrete problems in {AI} safety.
\newblock \emph{arXiv preprint arXiv:1606.06565}.

\bibitem[{Azar et~al.(2023)Azar, Rowland, Piot, Guo, Calandriello, Valko, and Munos}]{azar2023general}
Mohammad~Gheshlaghi Azar, Mark Rowland, Bilal Piot, Daniel Guo, Daniele Calandriello, Michal Valko, and Rémi Munos. 2023.
\newblock \href {https://arxiv.org/abs/2310.12036} {A general theoretical paradigm to understand learning from human preferences}.
\newblock \emph{arXiv preprint arXiv:2310.12036}.

\bibitem[{Bai et~al.(2022)Bai, Jones, Ndousse, Askell, Chen, DasSarma, Drain, Fort, Ganguli, Henighan, Joseph, Kadavath, Kernion, Conerly, El-Showk, Elhage, Hatfield-Dodds, Hernandez, Hume, Johnston, Kravec, Lovitt, Nanda, Olsson, Amodei, Brown, Clark, McCandlish, Olah, Mann, and Kaplan}]{bai2022training}
Yuntao Bai, Andy Jones, Kamal Ndousse, Amanda Askell, Anna Chen, Nova DasSarma, Dawn Drain, Stanislav Fort, Deep Ganguli, Tom Henighan, Nicholas Joseph, Saurav Kadavath, Jackson Kernion, Tom Conerly, Sheer El-Showk, Nelson Elhage, Zac Hatfield-Dodds, Danny Hernandez, Tristan Hume, Scott Johnston, Shauna Kravec, Liane Lovitt, Neel Nanda, Catherine Olsson, Dario Amodei, Tom Brown, Jack Clark, Sam McCandlish, Chris Olah, Ben Mann, and Jared Kaplan. 2022.
\newblock Training a helpful and harmless assistant with reinforcement learning from human feedback.
\newblock \emph{arXiv preprint arXiv:2204.05862}.

\bibitem[{Beirami et~al.(2024)Beirami, Agarwal, Berant, D'Amour, Eisenstein, Nagpal, and Suresh}]{beirami2024theoretical}
Ahmad Beirami, Alekh Agarwal, Jonathan Berant, Alexander D'Amour, Jacob Eisenstein, Chirag Nagpal, and Ananda~Theertha Suresh. 2024.
\newblock Theoretical guarantees on the best-of-n alignment policy.
\newblock \emph{arXiv preprint arXiv:2401.01879}.

\bibitem[{Bertsch et~al.(2023)Bertsch, Xie, Neubig, and Gormley}]{bertsch-etal-2023-mbr}
Amanda Bertsch, Alex Xie, Graham Neubig, and Matthew Gormley. 2023.
\newblock \href {https://doi.org/10.18653/v1/2023.bigpicture-1.9} {It{'}s {MBR} all the way down: Modern generation techniques through the lens of minimum {B}ayes risk}.
\newblock In \emph{Proceedings of the Big Picture Workshop}, pages 108--122, Singapore. Association for Computational Linguistics.

\bibitem[{Bickel and Doksum(2015)}]{bickel2015mathematical}
Peter~J Bickel and Kjell~A Doksum. 2015.
\newblock \emph{Mathematical statistics: basic ideas and selected topics, volumes I-II package}.
\newblock Chapman and Hall/CRC.

\bibitem[{Biderman et~al.(2023)Biderman, Schoelkopf, Anthony, Bradley, O'Brien, Hallahan, Khan, Purohit, Prashanth, Raff, Skowron, Sutawika, and Van Der~Wal}]{pmlr-v202-biderman23a}
Stella Biderman, Hailey Schoelkopf, Quentin~Gregory Anthony, Herbie Bradley, Kyle O'Brien, Eric Hallahan, Mohammad~Aflah Khan, Shivanshu Purohit, Usvsn~Sai Prashanth, Edward Raff, Aviya Skowron, Lintang Sutawika, and Oskar Van Der~Wal. 2023.
\newblock \href {https://proceedings.mlr.press/v202/biderman23a.html} {Pythia: A suite for analyzing large language models across training and scaling}.
\newblock In \emph{Proceedings of the 40th International Conference on Machine Learning}, volume 202 of \emph{Proceedings of Machine Learning Research}, pages 2397--2430. PMLR.

\bibitem[{Bradley and Terry(1952)}]{bradley1952rank}
Ralph~Allan Bradley and Milton~E Terry. 1952.
\newblock Rank analysis of incomplete block designs: I. the method of paired comparisons.
\newblock \emph{Biometrika}, 39(3/4):324--345.

\bibitem[{Chang et~al.(2024)Chang, Zhan, Oertell, Brantley, Misra, Lee, and Sun}]{chang2024dataset}
Jonathan~D. Chang, Wenhao Zhan, Owen Oertell, Kianté Brantley, Dipendra Misra, Jason~D. Lee, and Wen Sun. 2024.
\newblock Dataset reset policy optimization for {RLHF}.
\newblock \emph{arXiv preprint arXiv:2404.08495}.

\bibitem[{Cheng and Vlachos(2023)}]{cheng-vlachos-2023-faster}
Julius Cheng and Andreas Vlachos. 2023.
\newblock \href {https://aclanthology.org/2023.emnlp-main.767} {Faster minimum {B}ayes risk decoding with confidence-based pruning}.
\newblock In \emph{Proceedings of the 2023 Conference on Empirical Methods in Natural Language Processing}, pages 12473--12480, Singapore. Association for Computational Linguistics.

\bibitem[{Comon(1994)}]{comon1994independent}
Pierre Comon. 1994.
\newblock Independent component analysis, a new concept?
\newblock \emph{Signal processing}, 36(3):287--314.

\bibitem[{Conover et~al.(2023)Conover, Hayes, Mathur, Xie, Wan, Shah, Ghodsi, Wendell, Zaharia, and Xin}]{DatabricksBlog2023DollyV2}
Mike Conover, Matt Hayes, Ankit Mathur, Jianwei Xie, Jun Wan, Sam Shah, Ali Ghodsi, Patrick Wendell, Matei Zaharia, and Reynold Xin. 2023.
\newblock \href {https://www.databricks.com/blog/2023/04/12/dolly-first-open-commercially-viable-instruction-tuned-llm} {Free dolly: Introducing the world's first truly open instruction-tuned llm}.

\bibitem[{Coste et~al.(2024)Coste, Anwar, Kirk, and Krueger}]{coste2024reward}
Thomas Coste, Usman Anwar, Robert Kirk, and David Krueger. 2024.
\newblock \href {https://openreview.net/forum?id=dcjtMYkpXx} {Reward model ensembles help mitigate overoptimization}.
\newblock In \emph{The Twelfth International Conference on Learning Representations}.

\bibitem[{Deguchi et~al.(2024{\natexlab{a}})Deguchi, Sakai, Kamigaito, and Watanabe}]{deguchi2024mbrs}
Hiroyuki Deguchi, Yusuke Sakai, Hidetaka Kamigaito, and Taro Watanabe. 2024{\natexlab{a}}.
\newblock \href {https://doi.org/10.18653/v1/2024.emnlp-demo.37} {mbrs: A library for minimum {B}ayes risk decoding}.
\newblock In \emph{Proceedings of the 2024 Conference on Empirical Methods in Natural Language Processing: System Demonstrations}, pages 351--362, Miami, Florida, USA. Association for Computational Linguistics.

\bibitem[{Deguchi et~al.(2024{\natexlab{b}})Deguchi, Sakai, Kamigaito, Watanabe, Tanaka, and Utiyama}]{deguchi-etal-2024-centroid}
Hiroyuki Deguchi, Yusuke Sakai, Hidetaka Kamigaito, Taro Watanabe, Hideki Tanaka, and Masao Utiyama. 2024{\natexlab{b}}.
\newblock \href {https://doi.org/10.18653/v1/2024.findings-acl.654} {Centroid-based efficient minimum {B}ayes risk decoding}.
\newblock In \emph{Findings of the Association for Computational Linguistics ACL 2024}, pages 11009--11018, Bangkok, Thailand and virtual meeting. Association for Computational Linguistics.

\bibitem[{Dubois et~al.(2023)Dubois, Li, Taori, Zhang, Gulrajani, Ba, Guestrin, Liang, and Hashimoto}]{NEURIPS2023_5fc47800}
Yann Dubois, Chen~Xuechen Li, Rohan Taori, Tianyi Zhang, Ishaan Gulrajani, Jimmy Ba, Carlos Guestrin, Percy~S Liang, and Tatsunori~B Hashimoto. 2023.
\newblock \href {https://proceedings.neurips.cc/paper_files/paper/2023/file/5fc47800ee5b30b8777fdd30abcaaf3b-Paper-Conference.pdf} {Alpacafarm: A simulation framework for methods that learn from human feedback}.
\newblock In \emph{Advances in Neural Information Processing Systems}, volume~36, pages 30039--30069. Curran Associates, Inc.

\bibitem[{Eikema(2024)}]{eikema-2024-effect}
Bryan Eikema. 2024.
\newblock \href {https://aclanthology.org/2024.uncertainlp-1.9} {The effect of generalisation on the inadequacy of the mode}.
\newblock In \emph{Proceedings of the 1st Workshop on Uncertainty-Aware NLP (UncertaiNLP 2024)}, pages 87--92, St Julians, Malta. Association for Computational Linguistics.

\bibitem[{Eikema and Aziz(2020)}]{eikema-aziz-2020-map}
Bryan Eikema and Wilker Aziz. 2020.
\newblock \href {https://doi.org/10.18653/v1/2020.coling-main.398} {Is {MAP} decoding all you need? the inadequacy of the mode in neural machine translation}.
\newblock In \emph{Proceedings of the 28th International Conference on Computational Linguistics}, pages 4506--4520, Barcelona, Spain (Online). International Committee on Computational Linguistics.

\bibitem[{Eikema and Aziz(2022)}]{eikema-aziz-2022-sampling}
Bryan Eikema and Wilker Aziz. 2022.
\newblock \href {https://doi.org/10.18653/v1/2022.emnlp-main.754} {Sampling-based approximations to minimum {B}ayes risk decoding for neural machine translation}.
\newblock In \emph{Proceedings of the 2022 Conference on Empirical Methods in Natural Language Processing}, pages 10978--10993, Abu Dhabi, United Arab Emirates. Association for Computational Linguistics.

\bibitem[{Eisenstein et~al.(2024)Eisenstein, Nagpal, Agarwal, Beirami, D'Amour, Dvijotham, Fisch, Heller, Pfohl, Ramachandran, Shaw, and Berant}]{eisenstein2023helping}
Jacob Eisenstein, Chirag Nagpal, Alekh Agarwal, Ahmad Beirami, Alexander~Nicholas D'Amour, Krishnamurthy~Dj Dvijotham, Adam Fisch, Katherine~A Heller, Stephen~Robert Pfohl, Deepak Ramachandran, Peter Shaw, and Jonathan Berant. 2024.
\newblock \href {https://openreview.net/forum?id=5u1GpUkKtG} {Helping or herding? reward model ensembles mitigate but do not eliminate reward hacking}.
\newblock In \emph{First Conference on Language Modeling}.

\bibitem[{Ethayarajh et~al.(2022)Ethayarajh, Choi, and Swayamdipta}]{pmlr-v162-ethayarajh22a}
Kawin Ethayarajh, Yejin Choi, and Swabha Swayamdipta. 2022.
\newblock \href {https://proceedings.mlr.press/v162/ethayarajh22a.html} {Understanding dataset difficulty with $\mathcal{V}$-usable information}.
\newblock In \emph{Proceedings of the 39th International Conference on Machine Learning}, volume 162 of \emph{Proceedings of Machine Learning Research}, pages 5988--6008. PMLR.

\bibitem[{Ethayarajh et~al.(2024)Ethayarajh, Xu, Muennighoff, Jurafsky, and Kiela}]{ethayarajh2024kto}
Kawin Ethayarajh, Winnie Xu, Niklas Muennighoff, Dan Jurafsky, and Douwe Kiela. 2024.
\newblock {KTO}: Model alignment as prospect theoretic optimization.
\newblock \emph{arXiv preprint arXiv:2402.01306}.

\bibitem[{Farinhas et~al.(2023)Farinhas, de~Souza, and Martins}]{farinhas-etal-2023-empirical}
Ant{\'o}nio Farinhas, Jos{\'e} de~Souza, and Andre Martins. 2023.
\newblock \href {https://doi.org/10.18653/v1/2023.emnlp-main.733} {An empirical study of translation hypothesis ensembling with large language models}.
\newblock In \emph{Proceedings of the 2023 Conference on Empirical Methods in Natural Language Processing}, pages 11956--11970, Singapore. Association for Computational Linguistics.

\bibitem[{Finkelstein and Freitag(2024)}]{finkelstein2024mbr}
Mara Finkelstein and Markus Freitag. 2024.
\newblock \href {https://openreview.net/forum?id=bkNx3O0sND} {{MBR} and {QE} finetuning: Training-time distillation of the best and most expensive decoding methods}.
\newblock In \emph{The Twelfth International Conference on Learning Representations}.

\bibitem[{Freitag et~al.(2022)Freitag, Grangier, Tan, and Liang}]{freitag-etal-2022-high}
Markus Freitag, David Grangier, Qijun Tan, and Bowen Liang. 2022.
\newblock \href {https://doi.org/10.1162/tacl_a_00491} {High quality rather than high model probability: Minimum {B}ayes risk decoding with neural metrics}.
\newblock \emph{Transactions of the Association for Computational Linguistics}, 10:811--825.

\bibitem[{Gao et~al.(2023)Gao, Schulman, and Hilton}]{pmlr-v202-gao23h}
Leo Gao, John Schulman, and Jacob Hilton. 2023.
\newblock \href {https://proceedings.mlr.press/v202/gao23h.html} {Scaling laws for reward model overoptimization}.
\newblock In \emph{Proceedings of the 40th International Conference on Machine Learning}, volume 202 of \emph{Proceedings of Machine Learning Research}, pages 10835--10866. PMLR.

\bibitem[{Goel and Byrne(2000)}]{goel2000minimum}
Vaibhava Goel and William~J Byrne. 2000.
\newblock Minimum bayes-risk automatic speech recognition.
\newblock \emph{Computer Speech \& Language}, 14(2):115--135.

\bibitem[{Gui et~al.(2024)Gui, Garbacea, and Veitch}]{gui2024bonbon}
Lin Gui, Cristina Garbacea, and Victor Veitch. 2024.
\newblock \href {https://openreview.net/forum?id=haSKMlrbX5} {Bo{NB}on alignment for large language models and the sweetness of best-of-n sampling}.
\newblock In \emph{The Thirty-eighth Annual Conference on Neural Information Processing Systems}.

\bibitem[{Guo et~al.(2024)Guo, Zhang, Liu, Liu, Khalman, Llinares, Rame, Mesnard, Zhao, Piot, Ferret, and Blondel}]{guo2024direct}
Shangmin Guo, Biao Zhang, Tianlin Liu, Tianqi Liu, Misha Khalman, Felipe Llinares, Alexandre Rame, Thomas Mesnard, Yao Zhao, Bilal Piot, Johan Ferret, and Mathieu Blondel. 2024.
\newblock Direct language model alignment from online {AI} feedback.
\newblock \emph{arXiv preprint arXiv:2402.04792}.

\bibitem[{Guttmann et~al.(2024)Guttmann, Pokrywka, Charkiewicz, and Nowakowski}]{guttmann-etal-2024-chasing}
Kamil Guttmann, Miko{\l}aj Pokrywka, Adrian Charkiewicz, and Artur Nowakowski. 2024.
\newblock \href {https://aclanthology.org/2024.eamt-1.11} {Chasing {COMET}: Leveraging minimum {B}ayes risk decoding for self-improving machine translation}.
\newblock In \emph{Proceedings of the 25th Annual Conference of the European Association for Machine Translation (Volume 1)}, pages 80--99, Sheffield, UK. European Association for Machine Translation (EAMT).

\bibitem[{Heineman et~al.(2024)Heineman, Dou, and Xu}]{heineman2024improving}
David Heineman, Yao Dou, and Wei Xu. 2024.
\newblock \href {https://doi.org/10.18653/v1/2024.emnlp-main.1255} {Improving minimum {B}ayes risk decoding with multi-prompt}.
\newblock In \emph{Proceedings of the 2024 Conference on Empirical Methods in Natural Language Processing}, pages 22525--22545, Miami, Florida, USA. Association for Computational Linguistics.

\bibitem[{Holtzman et~al.(2020)Holtzman, Buys, Du, Forbes, and Choi}]{Holtzman2020The}
Ari Holtzman, Jan Buys, Li~Du, Maxwell Forbes, and Yejin Choi. 2020.
\newblock \href {https://openreview.net/forum?id=rygGQyrFvH} {The curious case of neural text degeneration}.
\newblock In \emph{International Conference on Learning Representations}.

\bibitem[{Hong et~al.(2024)Hong, Lee, and Thorne}]{hong2024orpo}
Jiwoo Hong, Noah Lee, and James Thorne. 2024.
\newblock \href {https://doi.org/10.18653/v1/2024.emnlp-main.626} {{ORPO}: Monolithic preference optimization without reference model}.
\newblock In \emph{Proceedings of the 2024 Conference on Empirical Methods in Natural Language Processing}, pages 11170--11189, Miami, Florida, USA. Association for Computational Linguistics.

\bibitem[{Hu et~al.(2022)Hu, Shen, Wallis, Allen-Zhu, Li, Wang, Wang, and Chen}]{hu2022lora}
Edward~J Hu, Yelong Shen, Phillip Wallis, Zeyuan Allen-Zhu, Yuanzhi Li, Shean Wang, Lu~Wang, and Weizhu Chen. 2022.
\newblock \href {https://openreview.net/forum?id=nZeVKeeFYf9} {Lo{RA}: Low-rank adaptation of large language models}.
\newblock In \emph{International Conference on Learning Representations}.

\bibitem[{Hu et~al.(2023)Hu, Song, Cho, Wang, Foroosh, and Liu}]{hu-etal-2023-decipherpref}
Yebowen Hu, Kaiqiang Song, Sangwoo Cho, Xiaoyang Wang, Hassan Foroosh, and Fei Liu. 2023.
\newblock \href {https://doi.org/10.18653/v1/2023.emnlp-main.519} {{D}ecipher{P}ref: Analyzing influential factors in human preference judgments via {GPT}-4}.
\newblock In \emph{Proceedings of the 2023 Conference on Empirical Methods in Natural Language Processing}, pages 8344--8357, Singapore. Association for Computational Linguistics.

\bibitem[{Huang et~al.(2016)Huang, Guo, Kusner, Sun, Sha, and Weinberger}]{NIPS2016_10c66082}
Gao Huang, Chuan Guo, Matt~J Kusner, Yu~Sun, Fei Sha, and Kilian~Q Weinberger. 2016.
\newblock \href {https://proceedings.neurips.cc/paper_files/paper/2016/file/10c66082c124f8afe3df4886f5e516e0-Paper.pdf} {Supervised word mover\textquotesingle s distance}.
\newblock In \emph{Advances in Neural Information Processing Systems}, volume~29. Curran Associates, Inc.

\bibitem[{Jang et~al.(2023)Jang, Kim, Lin, Wang, Hessel, Zettlemoyer, Hajishirzi, Choi, and Ammanabrolu}]{jang2023personalized}
Joel Jang, Seungone Kim, Bill~Yuchen Lin, Yizhong Wang, Jack Hessel, Luke Zettlemoyer, Hannaneh Hajishirzi, Yejin Choi, and Prithviraj Ammanabrolu. 2023.
\newblock Personalized soups: Personalized large language model alignment via post-hoc parameter merging.
\newblock \emph{arXiv preprint arXiv:2310.11564}.

\bibitem[{Jiang et~al.(2023{\natexlab{a}})Jiang, Sablayrolles, Mensch, Bamford, Chaplot, de~las Casas, Bressand, Lengyel, Lample, Saulnier, Lavaud, Lachaux, Stock, Scao, Lavril, Wang, Lacroix, and Sayed}]{jiang2023mistral}
Albert~Q. Jiang, Alexandre Sablayrolles, Arthur Mensch, Chris Bamford, Devendra~Singh Chaplot, Diego de~las Casas, Florian Bressand, Gianna Lengyel, Guillaume Lample, Lucile Saulnier, Lélio~Renard Lavaud, Marie-Anne Lachaux, Pierre Stock, Teven~Le Scao, Thibaut Lavril, Thomas Wang, Timothée Lacroix, and William~El Sayed. 2023{\natexlab{a}}.
\newblock \href {https://arxiv.org/abs/2310.06825} {Mistral 7b}.
\newblock \emph{arXiv}.

\bibitem[{Jiang et~al.(2023{\natexlab{b}})Jiang, Ren, and Lin}]{llm-blender-2023}
Dongfu Jiang, Xiang Ren, and Bill~Yuchen Lin. 2023{\natexlab{b}}.
\newblock Llm-blender: Ensembling large language models with pairwise comparison and generative fusion.
\newblock In \emph{Proceedings of the 61th Annual Meeting of the Association for Computational Linguistics (ACL 2023)}.

\bibitem[{Jinnai and Ariu(2024)}]{jinnai2024hyperparameterfree}
Yuu Jinnai and Kaito Ariu. 2024.
\newblock \href {https://doi.org/10.18653/v1/2024.findings-acl.505} {Hyperparameter-free approach for faster minimum {B}ayes risk decoding}.
\newblock In \emph{Findings of the Association for Computational Linguistics ACL 2024}, pages 8547--8566, Bangkok, Thailand and virtual meeting. Association for Computational Linguistics.

\bibitem[{junyou li et~al.(2024)junyou li, Zhang, Yu, FU, and Ye}]{li2024agents}
junyou li, Qin Zhang, Yangbin Yu, QIANG FU, and Deheng Ye. 2024.
\newblock \href {https://openreview.net/forum?id=bgzUSZ8aeg} {More agents is all you need}.
\newblock \emph{Transactions on Machine Learning Research}.

\bibitem[{Kumar and Byrne(2002)}]{kumar-byrne-2002-minimum}
Shankar Kumar and William Byrne. 2002.
\newblock \href {https://doi.org/10.3115/1118693.1118712} {Minimum {B}ayes-risk word alignments of bilingual texts}.
\newblock In \emph{Proceedings of the 2002 Conference on Empirical Methods in Natural Language Processing ({EMNLP} 2002)}, pages 140--147. Association for Computational Linguistics.

\bibitem[{Kumar and Byrne(2004)}]{kumar-byrne-2004-minimum}
Shankar Kumar and William Byrne. 2004.
\newblock \href {https://aclanthology.org/N04-1022} {Minimum {B}ayes-risk decoding for statistical machine translation}.
\newblock In \emph{Proceedings of the Human Language Technology Conference of the North {A}merican Chapter of the Association for Computational Linguistics: {HLT}-{NAACL} 2004}, pages 169--176, Boston, Massachusetts, USA. Association for Computational Linguistics.

\bibitem[{Kusner et~al.(2015)Kusner, Sun, Kolkin, and Weinberger}]{pmlr-v37-kusnerb15}
Matt Kusner, Yu~Sun, Nicholas Kolkin, and Kilian Weinberger. 2015.
\newblock \href {https://proceedings.mlr.press/v37/kusnerb15.html} {From word embeddings to document distances}.
\newblock In \emph{Proceedings of the 32nd International Conference on Machine Learning}, volume~37 of \emph{Proceedings of Machine Learning Research}, pages 957--966, Lille, France. PMLR.

\bibitem[{Kwon et~al.(2023)Kwon, Li, Zhuang, Sheng, Zheng, Yu, Gonzalez, Zhang, and Stoica}]{kwon2023efficient}
Woosuk Kwon, Zhuohan Li, Siyuan Zhuang, Ying Sheng, Lianmin Zheng, Cody~Hao Yu, Joseph~E. Gonzalez, Hao Zhang, and Ion Stoica. 2023.
\newblock Efficient memory management for large language model serving with pagedattention.
\newblock In \emph{Proceedings of the ACM SIGOPS 29th Symposium on Operating Systems Principles}.

\bibitem[{Köpf et~al.(2023)Köpf, Kilcher, von Rütte, Anagnostidis, Tam, Stevens, Barhoum, Duc, Stanley, Nagyfi, ES, Suri, Glushkov, Dantuluri, Maguire, Schuhmann, Nguyen, and Mattick}]{kopf2023openassistant}
Andreas Köpf, Yannic Kilcher, Dimitri von Rütte, Sotiris Anagnostidis, Zhi-Rui Tam, Keith Stevens, Abdullah Barhoum, Nguyen~Minh Duc, Oliver Stanley, Richárd Nagyfi, Shahul ES, Sameer Suri, David Glushkov, Arnav Dantuluri, Andrew Maguire, Christoph Schuhmann, Huu Nguyen, and Alexander Mattick. 2023.
\newblock Openassistant conversations -- democratizing large language model alignment.
\newblock \emph{arXiv preprint arXiv:2304.07327}.

\bibitem[{Lambert and Calandra(2024)}]{lambert2024alignment}
Nathan Lambert and Roberto Calandra. 2024.
\newblock The alignment ceiling: Objective mismatch in reinforcement learning from human feedback.
\newblock \emph{arXiv preprint arXiv:2311.00168}.

\bibitem[{Lambert et~al.(2024)Lambert, Pyatkin, Morrison, Miranda, Lin, Chandu, Dziri, Kumar, Zick, Choi, Smith, and Hajishirzi}]{lambert2024rewardbench}
Nathan Lambert, Valentina Pyatkin, Jacob Morrison, LJ~Miranda, Bill~Yuchen Lin, Khyathi Chandu, Nouha Dziri, Sachin Kumar, Tom Zick, Yejin Choi, Noah~A. Smith, and Hannaneh Hajishirzi. 2024.
\newblock {RewardBench}: Evaluating reward models for language modeling.
\newblock \emph{arXiv preprint arXiv:2403.13787}.

\bibitem[{Li et~al.(2024{\natexlab{a}})Li, Haider, and Callison-Burch}]{li-etal-2024-land}
Bryan Li, Samar Haider, and Chris Callison-Burch. 2024{\natexlab{a}}.
\newblock \href {https://doi.org/10.18653/v1/2024.naacl-long.213} {This land is {Your, My} land: Evaluating geopolitical bias in language models through territorial disputes}.
\newblock In \emph{Proceedings of the 2024 Conference of the North American Chapter of the Association for Computational Linguistics: Human Language Technologies (Volume 1: Long Papers)}, pages 3855--3871, Mexico City, Mexico. Association for Computational Linguistics.

\bibitem[{Li et~al.(2024{\natexlab{b}})Li, Wen, Wang, Li, Yuan, Liu, Liu, Xu, Wang, Sun et~al.}]{li2024personal}
Yuanchun Li, Hao Wen, Weijun Wang, Xiangyu Li, Yizhen Yuan, Guohong Liu, Jiacheng Liu, Wenxing Xu, Xiang Wang, Yi~Sun, et~al. 2024{\natexlab{b}}.
\newblock Personal llm agents: Insights and survey about the capability, efficiency and security.
\newblock \emph{arXiv preprint arXiv:2401.05459}.

\bibitem[{Liu et~al.(2024)Liu, Qin, Wu, Shen, Khalman, Joshi, Zhao, Saleh, Baumgartner, Liu, Liu, and Wang}]{liu2024lipo}
Tianqi Liu, Zhen Qin, Junru Wu, Jiaming Shen, Misha Khalman, Rishabh Joshi, Yao Zhao, Mohammad Saleh, Simon Baumgartner, Jialu Liu, Peter~J. Liu, and Xuanhui Wang. 2024.
\newblock {LiPO}: Listwise preference optimization through learning-to-rank.
\newblock \emph{arXiv preprint arXiv:2402.01878}.

\bibitem[{McAllester and Stratos(2020)}]{mcallester2020formal}
David McAllester and Karl Stratos. 2020.
\newblock Formal limitations on the measurement of mutual information.
\newblock In \emph{International Conference on Artificial Intelligence and Statistics}, pages 875--884. PMLR.

\bibitem[{Meister et~al.(2022)Meister, Wiher, Pimentel, and Cotterell}]{meister-etal-2022-high}
Clara Meister, Gian Wiher, Tiago Pimentel, and Ryan Cotterell. 2022.
\newblock \href {https://doi.org/10.18653/v1/2022.acl-short.5} {On the probability{--}quality paradox in language generation}.
\newblock In \emph{Proceedings of the 60th Annual Meeting of the Association for Computational Linguistics (Volume 2: Short Papers)}, pages 36--45, Dublin, Ireland. Association for Computational Linguistics.

\bibitem[{Meng et~al.(2024)Meng, Xia, and Chen}]{meng2024simpo}
Yu~Meng, Mengzhou Xia, and Danqi Chen. 2024.
\newblock \href {https://openreview.net/forum?id=3Tzcot1LKb} {Sim{PO}: Simple preference optimization with a reference-free reward}.
\newblock In \emph{The Thirty-eighth Annual Conference on Neural Information Processing Systems}.

\bibitem[{Mitchell et~al.(2024)Mitchell, Rafailov, Sharma, Finn, and Manning}]{mitchell2024an}
Eric Mitchell, Rafael Rafailov, Archit Sharma, Chelsea Finn, and Christopher~D Manning. 2024.
\newblock \href {https://openreview.net/forum?id=Eo7kv0sllr} {An emulator for fine-tuning large language models using small language models}.
\newblock In \emph{The Twelfth International Conference on Learning Representations}.

\bibitem[{Morimura et~al.(2024)Morimura, Sakamoto, Jinnai, Abe, and Ariu}]{morimura2024filtered}
Tetsuro Morimura, Mitsuki Sakamoto, Yuu Jinnai, Kenshi Abe, and Kaito Ariu. 2024.
\newblock \href {https://doi.org/10.18653/v1/2024.emnlp-main.1266} {Filtered direct preference optimization}.
\newblock In \emph{Proceedings of the 2024 Conference on Empirical Methods in Natural Language Processing}, pages 22729--22770, Miami, Florida, USA. Association for Computational Linguistics.

\bibitem[{Mudgal et~al.(2023)Mudgal, Lee, Ganapathy, Li, Wang, Huang, Chen, Cheng, Collins, Chen, Beutel, and Beirami}]{mudgal2023controlled}
Sidharth Mudgal, Jong Lee, Harish Ganapathy, YaGuang Li, Tao Wang, Yanping Huang, Zhifeng Chen, Heng-Tze Cheng, Michael Collins, Jilin Chen, Alex Beutel, and Ahmad Beirami. 2023.
\newblock \href {https://openreview.net/forum?id=jo57H1CpD8} {Controlled decoding from language models}.
\newblock In \emph{Socially Responsible Language Modelling Research}.

\bibitem[{Mudgal et~al.(2024)Mudgal, Lee, Ganapathy, Li, Wang, Huang, Chen, Cheng, Collins, Strohman, Chen, Beutel, and Beirami}]{mudgal2024controlled}
Sidharth Mudgal, Jong Lee, Harish Ganapathy, YaGuang Li, Tao Wang, Yanping Huang, Zhifeng Chen, Heng-Tze Cheng, Michael Collins, Trevor Strohman, Jilin Chen, Alex Beutel, and Ahmad Beirami. 2024.
\newblock \href {https://openreview.net/forum?id=bVIcZb7Qa0} {Controlled decoding from language models}.
\newblock In \emph{Forty-first International Conference on Machine Learning}.

\bibitem[{Nakano et~al.(2022)Nakano, Hilton, Balaji, Wu, Ouyang, Kim, Hesse, Jain, Kosaraju, Saunders, Jiang, Cobbe, Eloundou, Krueger, Button, Knight, Chess, and Schulman}]{nakano2022webgpt}
Reiichiro Nakano, Jacob Hilton, Suchir Balaji, Jeff Wu, Long Ouyang, Christina Kim, Christopher Hesse, Shantanu Jain, Vineet Kosaraju, William Saunders, Xu~Jiang, Karl Cobbe, Tyna Eloundou, Gretchen Krueger, Kevin Button, Matthew Knight, Benjamin Chess, and John Schulman. 2022.
\newblock Webgpt: Browser-assisted question-answering with human feedback.
\newblock \emph{arXiv preprint arXiv:2112.09332}.

\bibitem[{OpenAI et~al.(2024)OpenAI, Achiam, Adler, Agarwal, Ahmad, Akkaya, Aleman, Almeida, Altenschmidt, Altman, Anadkat, Avila, Babuschkin, Balaji, Balcom, Baltescu, Bao, Bavarian, Belgum, Bello, Berdine, Bernadett-Shapiro, Berner, Bogdonoff, Boiko, Boyd, Brakman, Brockman, Brooks, Brundage, Button, Cai, Campbell, Cann, Carey, Carlson, Carmichael, Chan, Chang, Chantzis, Chen, Chen, Chen, Chen, Chen, Chess, Cho, Chu, Chung, Cummings, Currier, Dai, Decareaux, Degry, Deutsch, Deville, Dhar, Dohan, Dowling, Dunning, Ecoffet, Eleti, Eloundou, Farhi, Fedus, Felix, Fishman, Forte, Fulford, Gao, Georges, Gibson, Goel, Gogineni, Goh, Gontijo-Lopes, Gordon, Grafstein, Gray, Greene, Gross, Gu, Guo, Hallacy, Han, Harris, He, Heaton, Heidecke, Hesse, Hickey, Hickey, Hoeschele, Houghton, Hsu, Hu, Hu, Huizinga, Jain, Jain, Jang, Jiang, Jiang, Jin, Jin, Jomoto, Jonn, Jun, Kaftan, Łukasz Kaiser, Kamali, Kanitscheider, Keskar, Khan, Kilpatrick, Kim, Kim, Kim, Kirchner, Kiros, Knight, Kokotajlo, Łukasz Kondraciuk,
  Kondrich, Konstantinidis, Kosic, Krueger, Kuo, Lampe, Lan, Lee, Leike, Leung, Levy, Li, Lim, Lin, Lin, Litwin, Lopez, Lowe, Lue, Makanju, Malfacini, Manning, Markov, Markovski, Martin, Mayer, Mayne, McGrew, McKinney, McLeavey, McMillan, McNeil, Medina, Mehta, Menick, Metz, Mishchenko, Mishkin, Monaco, Morikawa, Mossing, Mu, Murati, Murk, Mély, Nair, Nakano, Nayak, Neelakantan, Ngo, Noh, Ouyang, O'Keefe, Pachocki, Paino, Palermo, Pantuliano, Parascandolo, Parish, Parparita, Passos, Pavlov, Peng, Perelman, de~Avila Belbute~Peres, Petrov, de~Oliveira~Pinto, Michael, Pokorny, Pokrass, Pong, Powell, Power, Power, Proehl, Puri, Radford, Rae, Ramesh, Raymond, Real, Rimbach, Ross, Rotsted, Roussez, Ryder, Saltarelli, Sanders, Santurkar, Sastry, Schmidt, Schnurr, Schulman, Selsam, Sheppard, Sherbakov, Shieh, Shoker, Shyam, Sidor, Sigler, Simens, Sitkin, Slama, Sohl, Sokolowsky, Song, Staudacher, Such, Summers, Sutskever, Tang, Tezak, Thompson, Tillet, Tootoonchian, Tseng, Tuggle, Turley, Tworek, Uribe, Vallone,
  Vijayvergiya, Voss, Wainwright, Wang, Wang, Wang, Ward, Wei, Weinmann, Welihinda, Welinder, Weng, Weng, Wiethoff, Willner, Winter, Wolrich, Wong, Workman, Wu, Wu, Wu, Xiao, Xu, Yoo, Yu, Yuan, Zaremba, Zellers, Zhang, Zhang, Zhao, Zheng, Zhuang, Zhuk, and Zoph}]{openai2024gpt4}
OpenAI, Josh Achiam, Steven Adler, Sandhini Agarwal, Lama Ahmad, Ilge Akkaya, Florencia~Leoni Aleman, Diogo Almeida, Janko Altenschmidt, Sam Altman, Shyamal Anadkat, Red Avila, Igor Babuschkin, Suchir Balaji, Valerie Balcom, Paul Baltescu, Haiming Bao, Mohammad Bavarian, Jeff Belgum, Irwan Bello, Jake Berdine, Gabriel Bernadett-Shapiro, Christopher Berner, Lenny Bogdonoff, Oleg Boiko, Madelaine Boyd, Anna-Luisa Brakman, Greg Brockman, Tim Brooks, Miles Brundage, Kevin Button, Trevor Cai, Rosie Campbell, Andrew Cann, Brittany Carey, Chelsea Carlson, Rory Carmichael, Brooke Chan, Che Chang, Fotis Chantzis, Derek Chen, Sully Chen, Ruby Chen, Jason Chen, Mark Chen, Ben Chess, Chester Cho, Casey Chu, Hyung~Won Chung, Dave Cummings, Jeremiah Currier, Yunxing Dai, Cory Decareaux, Thomas Degry, Noah Deutsch, Damien Deville, Arka Dhar, David Dohan, Steve Dowling, Sheila Dunning, Adrien Ecoffet, Atty Eleti, Tyna Eloundou, David Farhi, Liam Fedus, Niko Felix, Simón~Posada Fishman, Juston Forte, Isabella Fulford, Leo
  Gao, Elie Georges, Christian Gibson, Vik Goel, Tarun Gogineni, Gabriel Goh, Rapha Gontijo-Lopes, Jonathan Gordon, Morgan Grafstein, Scott Gray, Ryan Greene, Joshua Gross, Shixiang~Shane Gu, Yufei Guo, Chris Hallacy, Jesse Han, Jeff Harris, Yuchen He, Mike Heaton, Johannes Heidecke, Chris Hesse, Alan Hickey, Wade Hickey, Peter Hoeschele, Brandon Houghton, Kenny Hsu, Shengli Hu, Xin Hu, Joost Huizinga, Shantanu Jain, Shawn Jain, Joanne Jang, Angela Jiang, Roger Jiang, Haozhun Jin, Denny Jin, Shino Jomoto, Billie Jonn, Heewoo Jun, Tomer Kaftan, Łukasz Kaiser, Ali Kamali, Ingmar Kanitscheider, Nitish~Shirish Keskar, Tabarak Khan, Logan Kilpatrick, Jong~Wook Kim, Christina Kim, Yongjik Kim, Jan~Hendrik Kirchner, Jamie Kiros, Matt Knight, Daniel Kokotajlo, Łukasz Kondraciuk, Andrew Kondrich, Aris Konstantinidis, Kyle Kosic, Gretchen Krueger, Vishal Kuo, Michael Lampe, Ikai Lan, Teddy Lee, Jan Leike, Jade Leung, Daniel Levy, Chak~Ming Li, Rachel Lim, Molly Lin, Stephanie Lin, Mateusz Litwin, Theresa Lopez, Ryan
  Lowe, Patricia Lue, Anna Makanju, Kim Malfacini, Sam Manning, Todor Markov, Yaniv Markovski, Bianca Martin, Katie Mayer, Andrew Mayne, Bob McGrew, Scott~Mayer McKinney, Christine McLeavey, Paul McMillan, Jake McNeil, David Medina, Aalok Mehta, Jacob Menick, Luke Metz, Andrey Mishchenko, Pamela Mishkin, Vinnie Monaco, Evan Morikawa, Daniel Mossing, Tong Mu, Mira Murati, Oleg Murk, David Mély, Ashvin Nair, Reiichiro Nakano, Rajeev Nayak, Arvind Neelakantan, Richard Ngo, Hyeonwoo Noh, Long Ouyang, Cullen O'Keefe, Jakub Pachocki, Alex Paino, Joe Palermo, Ashley Pantuliano, Giambattista Parascandolo, Joel Parish, Emy Parparita, Alex Passos, Mikhail Pavlov, Andrew Peng, Adam Perelman, Filipe de~Avila Belbute~Peres, Michael Petrov, Henrique~Ponde de~Oliveira~Pinto, Michael, Pokorny, Michelle Pokrass, Vitchyr~H. Pong, Tolly Powell, Alethea Power, Boris Power, Elizabeth Proehl, Raul Puri, Alec Radford, Jack Rae, Aditya Ramesh, Cameron Raymond, Francis Real, Kendra Rimbach, Carl Ross, Bob Rotsted, Henri Roussez,
  Nick Ryder, Mario Saltarelli, Ted Sanders, Shibani Santurkar, Girish Sastry, Heather Schmidt, David Schnurr, John Schulman, Daniel Selsam, Kyla Sheppard, Toki Sherbakov, Jessica Shieh, Sarah Shoker, Pranav Shyam, Szymon Sidor, Eric Sigler, Maddie Simens, Jordan Sitkin, Katarina Slama, Ian Sohl, Benjamin Sokolowsky, Yang Song, Natalie Staudacher, Felipe~Petroski Such, Natalie Summers, Ilya Sutskever, Jie Tang, Nikolas Tezak, Madeleine~B. Thompson, Phil Tillet, Amin Tootoonchian, Elizabeth Tseng, Preston Tuggle, Nick Turley, Jerry Tworek, Juan Felipe~Cerón Uribe, Andrea Vallone, Arun Vijayvergiya, Chelsea Voss, Carroll Wainwright, Justin~Jay Wang, Alvin Wang, Ben Wang, Jonathan Ward, Jason Wei, CJ~Weinmann, Akila Welihinda, Peter Welinder, Jiayi Weng, Lilian Weng, Matt Wiethoff, Dave Willner, Clemens Winter, Samuel Wolrich, Hannah Wong, Lauren Workman, Sherwin Wu, Jeff Wu, Michael Wu, Kai Xiao, Tao Xu, Sarah Yoo, Kevin Yu, Qiming Yuan, Wojciech Zaremba, Rowan Zellers, Chong Zhang, Marvin Zhang, Shengjia
  Zhao, Tianhao Zheng, Juntang Zhuang, William Zhuk, and Barret Zoph. 2024.
\newblock {GPT-4} technical report.
\newblock \emph{arXiv preprint arXiv:2303.08774}.

\bibitem[{Ouyang et~al.(2022)Ouyang, Wu, Jiang, Almeida, Wainwright, Mishkin, Zhang, Agarwal, Slama, Ray, Schulman, Hilton, Kelton, Miller, Simens, Askell, Welinder, Christiano, Leike, and Lowe}]{NEURIPS2022_b1efde53}
Long Ouyang, Jeffrey Wu, Xu~Jiang, Diogo Almeida, Carroll Wainwright, Pamela Mishkin, Chong Zhang, Sandhini Agarwal, Katarina Slama, Alex Ray, John Schulman, Jacob Hilton, Fraser Kelton, Luke Miller, Maddie Simens, Amanda Askell, Peter Welinder, Paul~F Christiano, Jan Leike, and Ryan Lowe. 2022.
\newblock \href {https://proceedings.neurips.cc/paper_files/paper/2022/file/b1efde53be364a73914f58805a001731-Paper-Conference.pdf} {Training language models to follow instructions with human feedback}.
\newblock In \emph{Advances in Neural Information Processing Systems}, volume~35, pages 27730--27744. Curran Associates, Inc.

\bibitem[{Pace et~al.(2024)Pace, Mallinson, Malmi, Krause, and Severyn}]{pace2024westofn}
Aliz{\'e}e Pace, Jonathan Mallinson, Eric Malmi, Sebastian Krause, and Aliaksei Severyn. 2024.
\newblock \href {https://openreview.net/forum?id=7kNwZhMefs} {West-of-n: Synthetic preference generation for improved reward modeling}.
\newblock In \emph{ICLR 2024 Workshop on Navigating and Addressing Data Problems for Foundation Models}.

\bibitem[{Pan et~al.(2022)Pan, Bhatia, and Steinhardt}]{pan2022the}
Alexander Pan, Kush Bhatia, and Jacob Steinhardt. 2022.
\newblock \href {https://openreview.net/forum?id=JYtwGwIL7ye} {The effects of reward misspecification: Mapping and mitigating misaligned models}.
\newblock In \emph{International Conference on Learning Representations}.

\bibitem[{Park et~al.(2024)Park, Rafailov, Ermon, and Finn}]{park2024disentangling}
Ryan Park, Rafael Rafailov, Stefano Ermon, and Chelsea Finn. 2024.
\newblock \href {https://doi.org/10.18653/v1/2024.findings-acl.297} {Disentangling length from quality in direct preference optimization}.
\newblock In \emph{Findings of the Association for Computational Linguistics: ACL 2024}, pages 4998--5017, Bangkok, Thailand. Association for Computational Linguistics.

\bibitem[{Pearson(1901)}]{pearson1901liii}
Karl Pearson. 1901.
\newblock Liii. on lines and planes of closest fit to systems of points in space.
\newblock \emph{The London, Edinburgh, and Dublin philosophical magazine and journal of science}, 2(11):559--572.

\bibitem[{Pedregosa et~al.(2011)Pedregosa, Varoquaux, Gramfort, Michel, Thirion, Grisel, Blondel, Prettenhofer, Weiss, Dubourg, Vanderplas, Passos, Cournapeau, Brucher, Perrot, and Duchesnay}]{scikit-learn}
F.~Pedregosa, G.~Varoquaux, A.~Gramfort, V.~Michel, B.~Thirion, O.~Grisel, M.~Blondel, P.~Prettenhofer, R.~Weiss, V.~Dubourg, J.~Vanderplas, A.~Passos, D.~Cournapeau, M.~Brucher, M.~Perrot, and E.~Duchesnay. 2011.
\newblock Scikit-learn: Machine learning in {P}ython.
\newblock \emph{Journal of Machine Learning Research}, 12:2825--2830.

\bibitem[{Peyré and Cuturi(2020)}]{peyre2020computational}
Gabriel Peyré and Marco Cuturi. 2020.
\newblock Computational optimal transport.
\newblock \emph{arXiv preprint arXiv:1803.00567}.

\bibitem[{Rafailov et~al.(2023)Rafailov, Sharma, Mitchell, Manning, Ermon, and Finn}]{rafailov2023direct}
Rafael Rafailov, Archit Sharma, Eric Mitchell, Christopher~D Manning, Stefano Ermon, and Chelsea Finn. 2023.
\newblock \href {https://openreview.net/forum?id=HPuSIXJaa9} {Direct preference optimization: Your language model is secretly a reward model}.
\newblock In \emph{Thirty-seventh Conference on Neural Information Processing Systems}.

\bibitem[{Ram{\'{e}} et~al.(2023)Ram{\'{e}}, Couairon, Dancette, Gaya, Shukor, Soulier, and Cord}]{RameCDGSSC23}
Alexandre Ram{\'{e}}, Guillaume Couairon, Corentin Dancette, Jean{-}Baptiste Gaya, Mustafa Shukor, Laure Soulier, and Matthieu Cord. 2023.
\newblock Rewarded soups: towards pareto-optimal alignment by interpolating weights fine-tuned on diverse rewards.
\newblock In \emph{Advances in neural information processing systems}.

\bibitem[{Rame et~al.(2024)Rame, Vieillard, Hussenot, Dadashi, Cideron, Bachem, and Ferret}]{rame2024warm}
Alexandre Rame, Nino Vieillard, Leonard Hussenot, Robert Dadashi, Geoffrey Cideron, Olivier Bachem, and Johan Ferret. 2024.
\newblock \href {https://openreview.net/forum?id=s7RDnNUJy6} {{WARM}: On the benefits of weight averaged reward models}.
\newblock In \emph{Forty-first International Conference on Machine Learning}.

\bibitem[{Ramos et~al.(2024)Ramos, Fernandes, Farinhas, and Martins}]{ramos-etal-2024-aligning}
Miguel Ramos, Patrick Fernandes, Ant{\'o}nio Farinhas, and Andre Martins. 2024.
\newblock \href {https://aclanthology.org/2024.eamt-1.22} {Aligning neural machine translation models: Human feedback in training and inference}.
\newblock In \emph{Proceedings of the 25th Annual Conference of the European Association for Machine Translation (Volume 1)}, pages 258--274, Sheffield, UK. European Association for Machine Translation (EAMT).

\bibitem[{Rdusseeun and Kaufman(1987)}]{rdusseeun1987clustering}
LKPJ Rdusseeun and P~Kaufman. 1987.
\newblock Clustering by means of medoids.
\newblock In \emph{Proceedings of the statistical data analysis based on the L1 norm conference, neuchatel, switzerland}, volume~31.

\bibitem[{Rei et~al.(2020{\natexlab{a}})Rei, Stewart, Farinha, and Lavie}]{rei-etal-2020-comet}
Ricardo Rei, Craig Stewart, Ana~C Farinha, and Alon Lavie. 2020{\natexlab{a}}.
\newblock \href {https://doi.org/10.18653/v1/2020.emnlp-main.213} {{COMET}: A neural framework for {MT} evaluation}.
\newblock In \emph{Proceedings of the 2020 Conference on Empirical Methods in Natural Language Processing (EMNLP)}, pages 2685--2702, Online. Association for Computational Linguistics.

\bibitem[{Rei et~al.(2020{\natexlab{b}})Rei, Stewart, Farinha, and Lavie}]{rei-etal-2020-unbabels}
Ricardo Rei, Craig Stewart, Ana~C Farinha, and Alon Lavie. 2020{\natexlab{b}}.
\newblock \href {https://aclanthology.org/2020.wmt-1.101} {Unbabel{'}s participation in the {WMT}20 metrics shared task}.
\newblock In \emph{Proceedings of the Fifth Conference on Machine Translation}, pages 911--920, Online. Association for Computational Linguistics.

\bibitem[{Reimers and Gurevych(2019)}]{reimers-gurevych-2019-sentence}
Nils Reimers and Iryna Gurevych. 2019.
\newblock \href {https://doi.org/10.18653/v1/D19-1410} {Sentence-{BERT}: Sentence embeddings using {S}iamese {BERT}-networks}.
\newblock In \emph{Proceedings of the 2019 Conference on Empirical Methods in Natural Language Processing and the 9th International Joint Conference on Natural Language Processing (EMNLP-IJCNLP)}, pages 3982--3992, Hong Kong, China. Association for Computational Linguistics.

\bibitem[{Rubner et~al.(1998)Rubner, Tomasi, and Guibas}]{rubner1998metric}
Yossi Rubner, Carlo Tomasi, and Leonidas~J Guibas. 1998.
\newblock A metric for distributions with applications to image databases.
\newblock In \emph{Sixth international conference on computer vision (IEEE Cat. No. 98CH36271)}, pages 59--66. IEEE.

\bibitem[{Schulman et~al.(2017)Schulman, Wolski, Dhariwal, Radford, and Klimov}]{schulman2017proximal}
John Schulman, Filip Wolski, Prafulla Dhariwal, Alec Radford, and Oleg Klimov. 2017.
\newblock Proximal policy optimization algorithms.
\newblock \emph{arXiv preprint arXiv:1707.06347}.

\bibitem[{Shi et~al.(2024)Shi, Chen, Hu, Liu, Hajishirzi, Smith, and Du}]{shi2024decodingtime}
Ruizhe Shi, Yifang Chen, Yushi Hu, Alisa Liu, Hannaneh Hajishirzi, Noah~A. Smith, and Simon~Shaolei Du. 2024.
\newblock \href {https://openreview.net/forum?id=RmGvEmttB7} {Decoding-time language model alignment with multiple objectives}.
\newblock In \emph{ICML 2024 Workshop on Theoretical Foundations of Foundation Models}.

\bibitem[{Sidahmed et~al.(2024)Sidahmed, Phatale, Hutcheson, Lin, Chen, Yu, Jin, Komarytsia, Ahlheim, Zhu, Chaudhary, Li, Ganesh, Byrne, Hoffmann, Mansoor, Li, Rastogi, and Dixon}]{sidahmed2024perl}
Hakim Sidahmed, Samrat Phatale, Alex Hutcheson, Zhuonan Lin, Zhang Chen, Zac Yu, Jarvis Jin, Roman Komarytsia, Christiane Ahlheim, Yonghao Zhu, Simral Chaudhary, Bowen Li, Saravanan Ganesh, Bill Byrne, Jessica Hoffmann, Hassan Mansoor, Wei Li, Abhinav Rastogi, and Lucas Dixon. 2024.
\newblock Perl: Parameter efficient reinforcement learning from human feedback.
\newblock \emph{arXiv preprint arXiv:2403.10704}.

\bibitem[{Skalse et~al.(2022)Skalse, Howe, Krasheninnikov, and Krueger}]{skalse2022defining}
Joar Max~Viktor Skalse, Nikolaus H.~R. Howe, Dmitrii Krasheninnikov, and David Krueger. 2022.
\newblock \href {https://openreview.net/forum?id=yb3HOXO3lX2} {Defining and characterizing reward gaming}.
\newblock In \emph{Advances in Neural Information Processing Systems}.

\bibitem[{Song et~al.(2020)Song, Tan, Qin, Lu, and Liu}]{song2020mpnet}
Kaitao Song, Xu~Tan, Tao Qin, Jianfeng Lu, and Tie-Yan Liu. 2020.
\newblock Mpnet: Masked and permuted pre-training for language understanding.
\newblock \emph{Advances in Neural Information Processing Systems}, 33:16857--16867.

\bibitem[{Sorensen et~al.(2024)Sorensen, Jiang, Hwang, Levine, Pyatkin, West, Dziri, Lu, Rao, Bhagavatula, Sap, Tasioulas, and Choi}]{DBLP:conf/aaai/SorensenJHLPWDL24}
Taylor Sorensen, Liwei Jiang, Jena~D. Hwang, Sydney Levine, Valentina Pyatkin, Peter West, Nouha Dziri, Ximing Lu, Kavel Rao, Chandra Bhagavatula, Maarten Sap, John Tasioulas, and Yejin Choi. 2024.
\newblock \href {https://doi.org/10.1609/aaai.v38i18.29970} {Value kaleidoscope: Engaging {AI} with pluralistic human values, rights, and duties}.
\newblock \emph{Proceedings of the AAAI Conference on Artificial Intelligence}, 38(18):19937--19947.

\bibitem[{Spearman(1904)}]{spearman1904proof}
Charles Spearman. 1904.
\newblock The proof and measurement of association between two things.
\newblock \emph{The American Journal of Psychology}, 15(1).

\bibitem[{Stahlberg and Byrne(2019)}]{stahlberg-byrne-2019-nmt}
Felix Stahlberg and Bill Byrne. 2019.
\newblock \href {https://doi.org/10.18653/v1/D19-1331} {On {NMT} search errors and model errors: Cat got your tongue?}
\newblock In \emph{Proceedings of the 2019 Conference on Empirical Methods in Natural Language Processing and the 9th International Joint Conference on Natural Language Processing (EMNLP-IJCNLP)}, pages 3356--3362, Hong Kong, China. Association for Computational Linguistics.

\bibitem[{Stiennon et~al.(2020)Stiennon, Ouyang, Wu, Ziegler, Lowe, Voss, Radford, Amodei, and Christiano}]{stiennon2020}
Nisan Stiennon, Long Ouyang, Jeffrey Wu, Daniel Ziegler, Ryan Lowe, Chelsea Voss, Alec Radford, Dario Amodei, and Paul~F Christiano. 2020.
\newblock \href {https://proceedings.neurips.cc/paper_files/paper/2020/file/1f89885d556929e98d3ef9b86448f951-Paper.pdf} {Learning to summarize with human feedback}.
\newblock In \emph{Advances in Neural Information Processing Systems}, volume~33, pages 3008--3021. Curran Associates, Inc.

\bibitem[{Suzgun et~al.(2023)Suzgun, Melas-Kyriazi, and Jurafsky}]{suzgun-etal-2023-follow}
Mirac Suzgun, Luke Melas-Kyriazi, and Dan Jurafsky. 2023.
\newblock \href {https://doi.org/10.18653/v1/2023.findings-acl.262} {Follow the wisdom of the crowd: Effective text generation via minimum {B}ayes risk decoding}.
\newblock In \emph{Findings of the Association for Computational Linguistics: ACL 2023}, pages 4265--4293, Toronto, Canada. Association for Computational Linguistics.

\bibitem[{Tajwar et~al.(2024)Tajwar, Singh, Sharma, Rafailov, Schneider, Xie, Ermon, Finn, and Kumar}]{tajwar2024preference}
Fahim Tajwar, Anikait Singh, Archit Sharma, Rafael Rafailov, Jeff Schneider, Tengyang Xie, Stefano Ermon, Chelsea Finn, and Aviral Kumar. 2024.
\newblock \href {https://openreview.net/forum?id=bWNPx6t0sF} {Preference fine-tuning of {LLM}s should leverage suboptimal, on-policy data}.
\newblock In \emph{Forty-first International Conference on Machine Learning}.

\bibitem[{Tang et~al.(2024)Tang, Guo, Zheng, Calandriello, Cao, Tarassov, Munos, Ávila Pires, Valko, Cheng, and Dabney}]{tang2024understanding}
Yunhao Tang, Daniel~Zhaohan Guo, Zeyu Zheng, Daniele Calandriello, Yuan Cao, Eugene Tarassov, Rémi Munos, Bernardo Ávila Pires, Michal Valko, Yong Cheng, and Will Dabney. 2024.
\newblock Understanding the performance gap between online and offline alignment algorithms.
\newblock \emph{arXiv preprint arXiv:2405.08448}.

\bibitem[{Tomani et~al.(2024)Tomani, Vilar, Freitag, Cherry, Naskar, Finkelstein, Garcia, and Cremers}]{tomani-etal-2024-quality}
Christian Tomani, David Vilar, Markus Freitag, Colin Cherry, Subhajit Naskar, Mara Finkelstein, Xavier Garcia, and Daniel Cremers. 2024.
\newblock \href {https://doi.org/10.18653/v1/2024.acl-long.836} {Quality-aware translation models: Efficient generation and quality estimation in a single model}.
\newblock In \emph{Proceedings of the 62nd Annual Meeting of the Association for Computational Linguistics (Volume 1: Long Papers)}, pages 15660--15679, Bangkok, Thailand. Association for Computational Linguistics.

\bibitem[{Trabelsi et~al.(2024)Trabelsi, Vilar, Finkelstein, and Freitag}]{trabelsi2024efficient}
Firas Trabelsi, David Vilar, Mara Finkelstein, and Markus Freitag. 2024.
\newblock \href {https://openreview.net/forum?id=8iPobEKUUA} {Efficient minimum bayes risk decoding using low-rank matrix completion algorithms}.
\newblock In \emph{The Thirty-eighth Annual Conference on Neural Information Processing Systems}.

\bibitem[{Tran et~al.(2021)Tran, Bhosale, Cross, Koehn, Edunov, and Fan}]{tran-etal-2021-facebook}
Chau Tran, Shruti Bhosale, James Cross, Philipp Koehn, Sergey Edunov, and Angela Fan. 2021.
\newblock \href {https://aclanthology.org/2021.wmt-1.19} {{F}acebook {AI}{'}s {WMT}21 news translation task submission}.
\newblock In \emph{Proceedings of the Sixth Conference on Machine Translation}, pages 205--215, Online. Association for Computational Linguistics.

\bibitem[{Tunstall et~al.(2024)Tunstall, Beeching, Lambert, Rajani, Rasul, Belkada, Huang, Werra, Fourrier, Habib, Sarrazin, Sanseviero, Rush, and Wolf}]{tunstall2023zephyr}
Lewis Tunstall, Edward~Emanuel Beeching, Nathan Lambert, Nazneen Rajani, Kashif Rasul, Younes Belkada, Shengyi Huang, Leandro~Von Werra, Cl{\'e}mentine Fourrier, Nathan Habib, Nathan Sarrazin, Omar Sanseviero, Alexander~M Rush, and Thomas Wolf. 2024.
\newblock \href {https://openreview.net/forum?id=aKkAwZB6JV} {Zephyr: Direct distillation of {LM} alignment}.
\newblock In \emph{First Conference on Language Modeling}.

\bibitem[{Vamvas and Sennrich(2024)}]{vamvas-sennrich-2024-linear}
Jannis Vamvas and Rico Sennrich. 2024.
\newblock \href {https://doi.org/10.18653/v1/2024.acl-short.71} {Linear-time minimum {B}ayes risk decoding with reference aggregation}.
\newblock In \emph{Proceedings of the 62nd Annual Meeting of the Association for Computational Linguistics (Volume 2: Short Papers)}, pages 790--801, Bangkok, Thailand. Association for Computational Linguistics.

\bibitem[{Villani(2021{\natexlab{a}})}]{villani2021topics}
C{\'e}dric Villani. 2021{\natexlab{a}}.
\newblock \emph{Topics in optimal transportation}, volume~58.
\newblock American Mathematical Soc.

\bibitem[{Villani(2021{\natexlab{b}})}]{villani-2021-changing}
Rossana Villani. 2021{\natexlab{b}}.
\newblock \href {https://aclanthology.org/2021.triton-1.2} {The changing profile of the translator profession at the {E}uropean central bank}.
\newblock In \emph{Proceedings of the Translation and Interpreting Technology Online Conference}, pages 7--14, Held Online. INCOMA Ltd.

\bibitem[{von Werra et~al.(2020)von Werra, Belkada, Tunstall, Beeching, Thrush, Lambert, Huang, Rasul, and Gallouédec}]{vonwerra2022trl}
Leandro von Werra, Younes Belkada, Lewis Tunstall, Edward Beeching, Tristan Thrush, Nathan Lambert, Shengyi Huang, Kashif Rasul, and Quentin Gallouédec. 2020.
\newblock Trl: Transformer reinforcement learning.
\newblock \url{https://github.com/huggingface/trl}.

\bibitem[{Wan et~al.(2023)Wan, Kim, and Kang}]{Wan_Kim_Kang_2023}
Ruyuan Wan, Jaehyung Kim, and Dongyeop Kang. 2023.
\newblock \href {https://doi.org/10.1609/aaai.v37i12.26698} {Everyone’s voice matters: Quantifying annotation disagreement using demographic information}.
\newblock \emph{Proceedings of the AAAI Conference on Artificial Intelligence}, 37(12):14523--14530.

\bibitem[{Wen et~al.(2024)Wen, Lou, Lu, Lin, Yu, Lu, He, Han, Zhang, and Sun}]{wen2024rethinking}
Xueru Wen, Jie Lou, Yaojie Lu, Hongyu Lin, Xing Yu, Xinyu Lu, Ben He, Xianpei Han, Debing Zhang, and Le~Sun. 2024.
\newblock Rethinking reward model evaluation: Are we barking up the wrong tree?
\newblock \emph{arXiv preprint arXiv:2410.05584}.

\bibitem[{Wolf et~al.(2020)Wolf, Debut, Sanh, Chaumond, Delangue, Moi, Cistac, Rault, Louf, Funtowicz, Davison, Shleifer, von Platen, Ma, Jernite, Plu, Xu, Le~Scao, Gugger, Drame, Lhoest, and Rush}]{wolf-etal-2020-transformers}
Thomas Wolf, Lysandre Debut, Victor Sanh, Julien Chaumond, Clement Delangue, Anthony Moi, Pierric Cistac, Tim Rault, Remi Louf, Morgan Funtowicz, Joe Davison, Sam Shleifer, Patrick von Platen, Clara Ma, Yacine Jernite, Julien Plu, Canwen Xu, Teven Le~Scao, Sylvain Gugger, Mariama Drame, Quentin Lhoest, and Alexander Rush. 2020.
\newblock \href {https://doi.org/10.18653/v1/2020.emnlp-demos.6} {Transformers: State-of-the-art natural language processing}.
\newblock In \emph{Proceedings of the 2020 Conference on Empirical Methods in Natural Language Processing: System Demonstrations}, pages 38--45, Online. Association for Computational Linguistics.

\bibitem[{Xu et~al.(2024{\natexlab{a}})Xu, Sharaf, Chen, Tan, Shen, Durme, Murray, and Kim}]{xu2024contrastive}
Haoran Xu, Amr Sharaf, Yunmo Chen, Weiting Tan, Lingfeng Shen, Benjamin~Van Durme, Kenton Murray, and Young~Jin Kim. 2024{\natexlab{a}}.
\newblock \href {https://openreview.net/forum?id=51iwkioZpn} {Contrastive preference optimization: Pushing the boundaries of {LLM} performance in machine translation}.
\newblock In \emph{Forty-first International Conference on Machine Learning}.

\bibitem[{Xu et~al.(2023)Xu, Lee, Sukhbaatar, and Weston}]{xu2023things}
Jing Xu, Andrew Lee, Sainbayar Sukhbaatar, and Jason Weston. 2023.
\newblock Some things are more cringe than others: Preference optimization with the pairwise cringe loss.
\newblock \emph{arXiv preprint arXiv:2312.16682}.

\bibitem[{Xu et~al.(2024{\natexlab{b}})Xu, Fu, Gao, Ye, Liu, Mei, Wang, Yu, and Wu}]{xu2024dpo}
Shusheng Xu, Wei Fu, Jiaxuan Gao, Wenjie Ye, Weilin Liu, Zhiyu Mei, Guangju Wang, Chao Yu, and Yi~Wu. 2024{\natexlab{b}}.
\newblock \href {https://openreview.net/forum?id=6XH8R7YrSk} {Is {DPO} superior to {PPO} for {LLM} alignment? a comprehensive study}.
\newblock In \emph{Forty-first International Conference on Machine Learning}.

\bibitem[{Yang et~al.(2024)Yang, Chen, Lin, and Byrne}]{yang-etal-2024-direct}
Guangyu Yang, Jinghong Chen, Weizhe Lin, and Bill Byrne. 2024.
\newblock \href {https://doi.org/10.18653/v1/2024.naacl-short.34} {Direct preference optimization for neural machine translation with minimum {B}ayes risk decoding}.
\newblock In \emph{Proceedings of the 2024 Conference of the North American Chapter of the Association for Computational Linguistics: Human Language Technologies (Volume 2: Short Papers)}, pages 391--398, Mexico City, Mexico. Association for Computational Linguistics.

\bibitem[{Yuan et~al.(2024{\natexlab{a}})Yuan, Cui, Wang, Ding, Wang, Deng, Shan, Chen, Xie, Lin, Liu, Zhou, Peng, Liu, and Sun}]{yuan2024advancing}
Lifan Yuan, Ganqu Cui, Hanbin Wang, Ning Ding, Xingyao Wang, Jia Deng, Boji Shan, Huimin Chen, Ruobing Xie, Yankai Lin, Zhenghao Liu, Bowen Zhou, Hao Peng, Zhiyuan Liu, and Maosong Sun. 2024{\natexlab{a}}.
\newblock \href {https://openreview.net/forum?id=2Y1iiCqM5y} {Advancing {LLM} reasoning generalists with preference trees}.
\newblock In \emph{AI for Math Workshop @ ICML 2024}.

\bibitem[{Yuan et~al.(2024{\natexlab{b}})Yuan, Pang, Cho, Li, Sukhbaatar, Xu, and Weston}]{yuan2024selfrewarding}
Weizhe Yuan, Richard~Yuanzhe Pang, Kyunghyun Cho, Xian Li, Sainbayar Sukhbaatar, Jing Xu, and Jason~E Weston. 2024{\natexlab{b}}.
\newblock \href {https://openreview.net/forum?id=0NphYCmgua} {Self-rewarding language models}.
\newblock In \emph{Forty-first International Conference on Machine Learning}.

\bibitem[{Zheng et~al.(2023{\natexlab{a}})Zheng, Chiang, Sheng, Zhuang, Wu, Zhuang, Lin, Li, Li, Xing, Zhang, Gonzalez, and Stoica}]{NEURIPS2023_91f18a12}
Lianmin Zheng, Wei-Lin Chiang, Ying Sheng, Siyuan Zhuang, Zhanghao Wu, Yonghao Zhuang, Zi~Lin, Zhuohan Li, Dacheng Li, Eric Xing, Hao Zhang, Joseph~E Gonzalez, and Ion Stoica. 2023{\natexlab{a}}.
\newblock \href {https://proceedings.neurips.cc/paper_files/paper/2023/file/91f18a1287b398d378ef22505bf41832-Paper-Datasets_and_Benchmarks.pdf} {Judging llm-as-a-judge with mt-bench and chatbot arena}.
\newblock In \emph{Advances in Neural Information Processing Systems}, volume~36, pages 46595--46623. Curran Associates, Inc.

\bibitem[{Zheng et~al.(2023{\natexlab{b}})Zheng, Dou, Gao, Hua, Shen, Wang, Liu, Jin, Liu, Zhou, Xiong, Chen, Xi, Xu, Lai, Zhu, Chang, Yin, Weng, Cheng, Huang, Sun, Yan, Gui, Zhang, Qiu, and Huang}]{zheng2023secrets}
Rui Zheng, Shihan Dou, Songyang Gao, Yuan Hua, Wei Shen, Binghai Wang, Yan Liu, Senjie Jin, Qin Liu, Yuhao Zhou, Limao Xiong, Lu~Chen, Zhiheng Xi, Nuo Xu, Wenbin Lai, Minghao Zhu, Cheng Chang, Zhangyue Yin, Rongxiang Weng, Wensen Cheng, Haoran Huang, Tianxiang Sun, Hang Yan, Tao Gui, Qi~Zhang, Xipeng Qiu, and Xuanjing Huang. 2023{\natexlab{b}}.
\newblock Secrets of {RLHF} in large language models part i: {PPO}.
\newblock \emph{arXiv preprint arXiv:2307.04964}.

\bibitem[{Zhou et~al.(2024)Zhou, Zheng, Wang, Xi, Dou, Bao, Shen, Xiong, Fan, Mou, Zheng, Gui, Zhang, and Huang}]{zhou2024rmb}
Enyu Zhou, Guodong Zheng, Binghai Wang, Zhiheng Xi, Shihan Dou, Rong Bao, Wei Shen, Limao Xiong, Jessica Fan, Yurong Mou, Rui Zheng, Tao Gui, Qi~Zhang, and Xuanjing Huang. 2024.
\newblock {RMB}: Comprehensively benchmarking reward models in {LLM} alignment.
\newblock \emph{arXiv preprint arXiv:2410.09893}.

\bibitem[{Ziegler et~al.(2020)Ziegler, Stiennon, Wu, Brown, Radford, Amodei, Christiano, and Irving}]{ziegler2020finetuning}
Daniel~M. Ziegler, Nisan Stiennon, Jeffrey Wu, Tom~B. Brown, Alec Radford, Dario Amodei, Paul Christiano, and Geoffrey Irving. 2020.
\newblock Fine-tuning language models from human preferences.
\newblock \emph{arXiv preprint arXiv:1909.08593}.

\end{thebibliography}

\appendix

\clearpage
\section{Overoptimization of BoN Sampling}
\label{apd:overoptimization}

Figure \ref{fig:overoptimization} shows the performance of BoN sampling using proxy reward models evaluated by a gold reference reward model.
The proxy reward models are based on the Pythia-1B model \citep{pmlr-v202-biderman23a} and trained using the first 1000, 2000, and 4000 entries of the training set of AlpacaFarm.
The gold reference reward model is based on the Pythia-2.8B model and trained using the entire training set (9600 entries).
Spearman's rank correlation coefficients \citep{spearman1904proof} of the proxy reward models with the gold reference reward models are present in Table~\ref{tab:pythia-corr}.
The hyperparameters used in the reward model training are described in Table~\ref{tab:rm-hypers}.

The performance of BoN sampling improves with larger samples up to some point and it then decreases with more samples from that point. 

\begin{figure}[h]
    \centering
    \includegraphics[width=0.95\columnwidth]{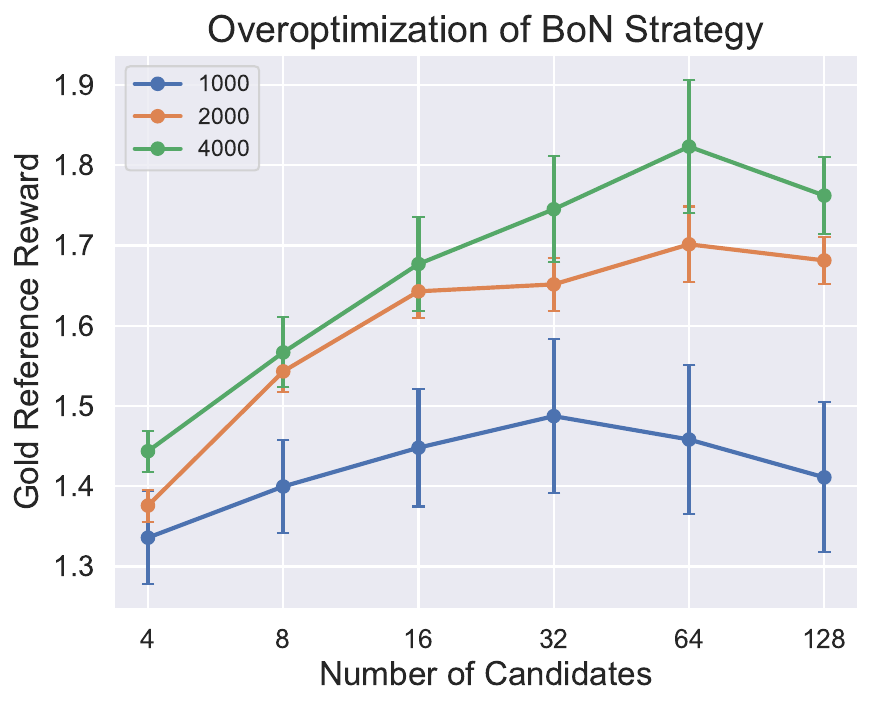}
    \caption{Performance of BoN sampling using proxy reward models. The lines show the mean and the bars show the standard deviation of three runs.}
    \label{fig:overoptimization}
\end{figure}

\begin{table}[h]
    \centering
    \caption{Spearman's rank correlation coefficients of the proxy reward models with the gold reference reward model (Pythia 2.8B). The proxy reward models are trained with 1000, 2000, and 4000 instances of the training split.}
    \label{tab:pythia-corr}
    \begin{tabular}{cc}
    \toprule
    \#Training & $\rho$ \\
    \midrule
    1000 & 0.189 $\pm$ 0.264 \\
    2000 & 0.327 $\pm$ 0.215 \\
    4000 & 0.358 $\pm$ 0.224 \\
    \bottomrule
    \end{tabular}
\end{table}

\section{Evaluation of MBR Objective as a Proximity Regularizer}
\label{apd:pca}

Table~\ref{tab:helpful-pca} shows the correlation of the distance to the center of the distribution of the sentence embeddings (i.e., $L_1$-norm of the component vector) with the value of the MBR objective in the hh-rlhf datasets. See Section~\ref{sec:regularizer} for the experimental setups.
The distance from the center of the distribution has a strong negative correlation with the MBR objective. 
On the other hand, the correlation with the log probability of the output is weak which shows that the log probability and the KL-divergence using that is not a reliable source to quantify the proximity of the output with respect to the embedding space (Table~\ref{tab:logprob-pca}).

Figure~\ref{fig:pca-hh} shows the average normalized MBR objective values mapped to the first and second principal components. The result shows that the outputs that lie in the center of the distribution tend to have higher MBR scores, which indicates that the MBR score serves as a regularizer to keep the output faithful to the reference policy.

As an ablation study, we evaluate the correlation for a machine translation task using a machine translation model and a utility function for machine translation.
We use WMT'21 De-En \citep{akhbardeh-etal-2021-findings} as a dataset and \texttt{wmt21-dense-24-wide-x-en} \citep{tran-etal-2021-facebook} as the translation model. Both the embedding function and the utility function use \texttt{wmt20-comet-da} \citep{rei-etal-2020-unbabels}. Note that \texttt{wmt20-comet-da} is not designed to be a symmetric function with respect to $y$ and $y'$ as the model measures the utility over $y$ and $y'$ and also the translation quality directly using the source text $x$.
Figure~\ref{fig:pca-wmt} shows the mapping of the values of the MBR objective on WMT'21 De-En. Overall, we observe qualitatively the same result as in the AlpacaFarm and hh-rlhf datasets. The result shows that the MBR objective serving as a regularizer is observed in a machine translation task in addition to the instruction-following tasks.

\begin{table}[thb]
    \centering
    \caption{Correlation of the distance from the center of the distribution in the component space with the MBR objective.}
    \label{tab:helpful-pca}
    \adjustbox{max width=\columnwidth}{
    \begin{tabular}{ccc}
    \toprule
    Dim & PCA & ICA \\
    \midrule\midrule
    \multicolumn{3}{c}{hh-rlhf Helpfulness}\\\midrule
    2 & -0.5702 $\pm$ 0.2013 & -0.5696 $\pm$ 0.1906 \\
    5 & -0.7478 $\pm$ 0.1339 & -0.6931 $\pm$ 0.1299 \\
    10 & -0.8407 $\pm$ 0.1136 & -0.6792 $\pm$ 0.1375 \\
    \midrule\midrule
    \multicolumn{3}{c}{hh-rlhf Harmlessness}\\\midrule    
    2 & -0.6050 $\pm$ 0.1770 & -0.5917 $\pm$ 0.1727 \\
    5 & -0.7536 $\pm$ 0.1305 & -0.6920 $\pm$ 0.1298 \\
    10 & -0.8550 $\pm$ 0.1066 & -0.6909 $\pm$ 0.1311 \\
    \midrule\midrule
    \multicolumn{3}{c}{WMT'21 De-En}\\\midrule    
    2 & -0.3917 $\pm$ 0.2108 & -0.3820 $\pm$ 0.2055 \\
    5 & -0.5676 $\pm$ 0.1540 & -0.5287 $\pm$ 0.1458 \\
    10 & -0.6705 $\pm$ 0.1306 & -0.5612 $\pm$ 0.1360 \\
    \bottomrule
    \end{tabular}
    }
\end{table}

\begin{table}[th]
    \centering
    \caption{Correlation of the distance from the center of the distribution in the component space with the log probability on AplacaFarm.}
    \label{tab:logprob-pca}
    \adjustbox{max width=\columnwidth}{
    \begin{tabular}{ccc}
    \toprule
    Dim & PCA & ICA \\
    \midrule\midrule
    2 & 0.0826 $\pm$ 0.2340 & 0.0806 $\pm$ 0.2373 \\
    5 & 0.0954 $\pm$ 0.2315 & 0.0905 $\pm$ 0.2075 \\
    10 & 0.0784 $\pm$ 0.2357 & 0.0425 $\pm$ 0.2061 \\
    \bottomrule
    \end{tabular}
    }
\end{table}

\begin{figure}[tbh]
    \centering
    \subfloat[Helpfulness]{
    \includegraphics[width=0.725\columnwidth]{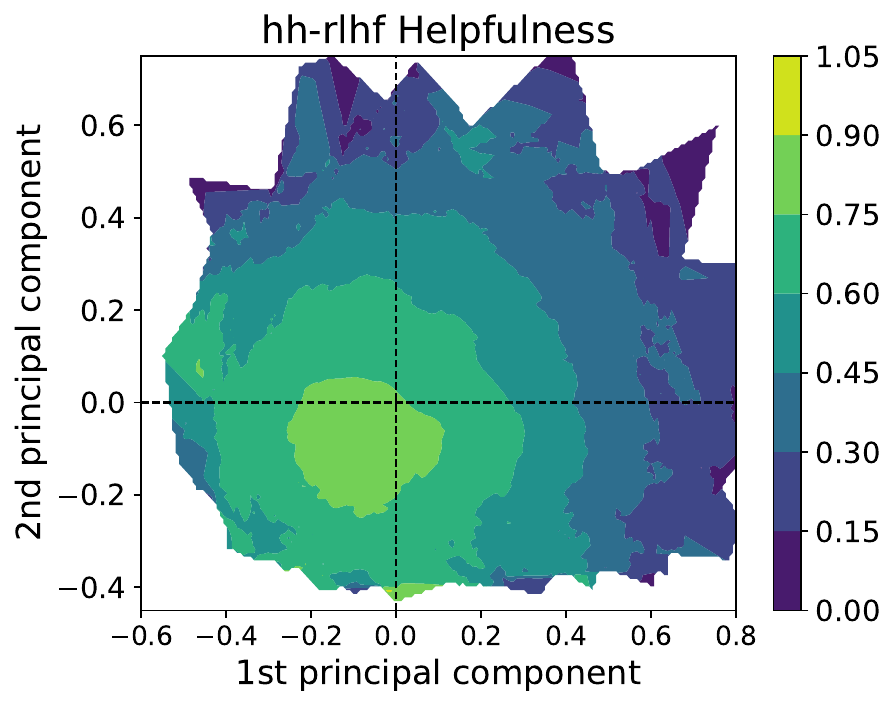}}\\
    \subfloat[Harmlessness]{
    \includegraphics[width=0.725\columnwidth]{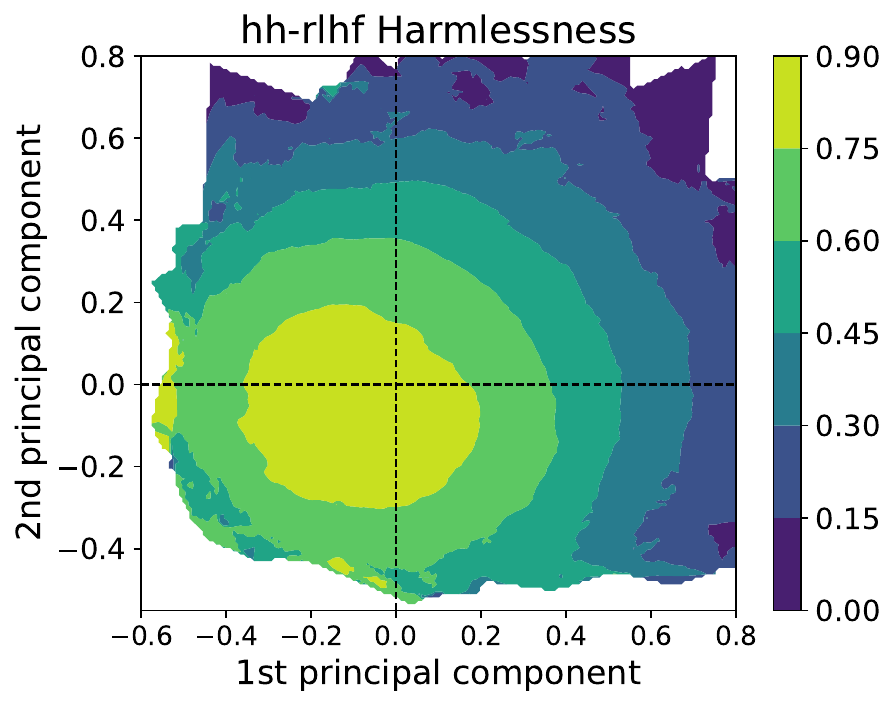}}\\
    \subfloat[WMT'21 De-En]{
    \includegraphics[width=0.725\columnwidth]{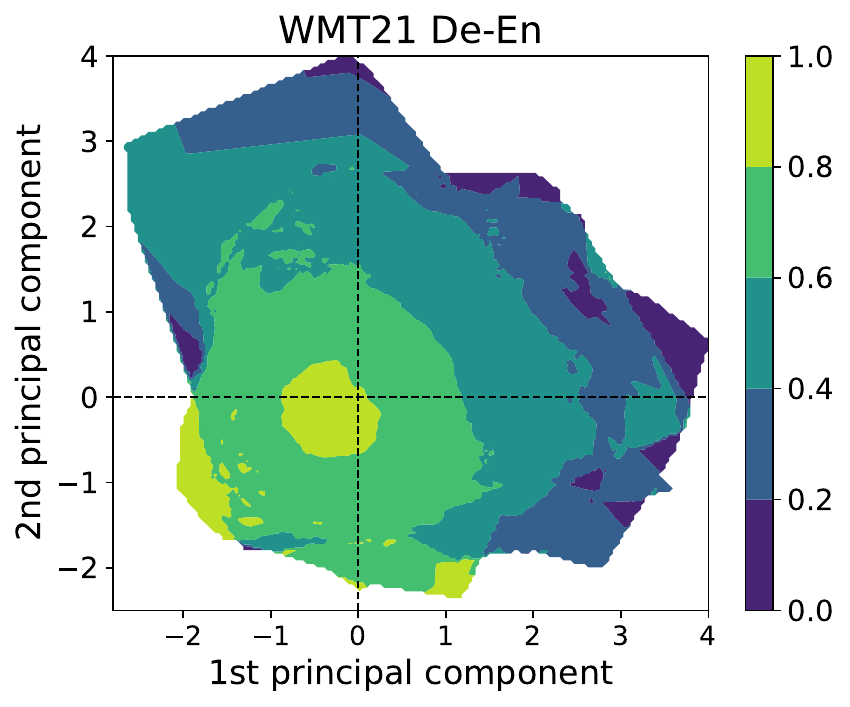}\label{fig:pca-wmt}}
    \caption{Visualization of the mean values of the MBR objective in the space of the first and second principal components.}
    \label{fig:pca-hh}
\end{figure}

\clearpage
\section{Derivation of Proposition~\ref{prop:wd}}
\label{apd:wd}
We show the derivation of Proposition.~\ref{prop:wd}. From the definition of Wasserstein distance with $p=1$ \citep{peyre2020computational,villani2021topics}, we get the following:
\begin{align}
    WD(\pi_y,& \hat{\pi}_\mathrm{ref}(\cdot | x)) = \nonumber \\
    &\min_{\{\mu_{i, j}\}_{i, j} \in \mathcal{J}} \sum_{i = 1}^{|\hypos|} \sum_{j = 1}^{|\hypos|} \mu_{i, j} C(y_i, y_j),
\label{eq:p1}
\end{align}
where $\mathcal{J}$ is a set of all couplings $\{\mu_{i, j}\}_{i, j}$ \citep{villani2021topics}:
\begin{align}
    \mathcal{J} = \bigl\{ \{\mu_{i, j}\}_{i, j}& : \nonumber\\
    &\sum_{i = 1}^{|\hypos|} \mu_{i, j} = \hat{\pi}_\mathrm{ref}(y_j | x), \nonumber\\
    &\sum_{j = 1}^{|\hypos|} \mu_{i, j} = \pi_y(y_i), \nonumber\\
    &\mu_{i, j} \geq 0 \bigr\}.
\end{align}
Because $\pi_y(y_i) = 0$ for all $y_i \neq y$ and $\mu_{i, j} \geq 0$, we get $\mu_{i, j} = 0$ for all $y_i \neq y$. Thus,
\begin{align}
    \eqref{eq:p1} = \min_{\mathcal{J}} \sum_{j = 1}^{|\hypos|} \mu_{y, j} C(y, y_j)
\label{eq:p2}
\end{align}
Using $\mu_{i, j} = 0$ for all $i \neq y$ and $\sum_{i = 1}^{|\hypos|} \mu_{i, j} = \hat{\pi}_\mathrm{ref}(y_j | x)$, we get $\mu_{y, j} = \hat{\pi}_\mathrm{ref}(y_j | x)$. Thus,
\begin{align}
    \eqref{eq:p2} &= \min_{\mathcal{J}} \sum_{j = 1}^{|\hypos|} \hat{\pi}_\mathrm{ref}(y_j | x) C(y, y_j) \nonumber \\
    &= \sum_{j = 1}^{|\hypos|} \hat{\pi}_\mathrm{ref}(y_j | x) C(y, y_j).
\label{eq:p3}
\end{align}
Because $\hat{\pi}_\mathrm{ref}(y_j | x)$ is an empirical distribution from the set of samples $\hypos$, $\hat{\pi}_\mathrm{ref}(y_j \mid x) = \frac{1}{N} \sum_{y_i \in \hypos} \mathbb{I}[y_j = y_i]$. Thus, 
\begin{align}
    \eqref{eq:p3} &= \sum_{y' \in \hypos}\frac{1}{N} C(y, y') \\
    &= -\sum_{y' \in \hypos}\frac{1}{N} U(y, y').
\end{align}
Thus, we get Proposition~\ref{prop:wd}.

\section{Evaluation of KL-Regularized BoN}
\label{apd:klrbon}

A naive implementation of proximity regularization for BoN sampling is to introduce KL-regularization. 
BoN with KL-regularization (\klrbon{}) can be derived from Eq.~\eqref{eq:ppo-loss} as follows:
\begin{align}
    y_{\mathrm{RBoN_{KL}}}&(x) = \nonumber \\
    &\argmax_{y \in \hypos} R(x, y) - \beta \mathbb{D}_\mathrm{KL}[\pi_y || \pi_\mathrm{ref}(\cdot | x)],
\label{eq:klbon}
\end{align}
where $\pi_y$ represents a policy of choosing $y$ with a probability of $1$, which is the policy it will end up with if it chooses $y$ as the output. Thus, $\mathbb{D}_\mathrm{KL}[\pi_y || \pi_\mathrm{ref}(\cdot | x)]$ represents the KL divergence between the resulting policy and the reference policy. 
Intuitively, \klrbon{} optimizes the same objective as Eq.~\eqref{eq:ppo-loss} but with modifications to make it available at decoding time.
Eq.~\eqref{eq:klbon} is derived from Eq.~\eqref{eq:ppo-loss} by computing the optimal response for a given $x$ instead of computing the optimal policy.

\begin{figure}[tb]
    \centering
    \subfloat[Proxy: SHP-Large]{\includegraphics[width=0.72\columnwidth]{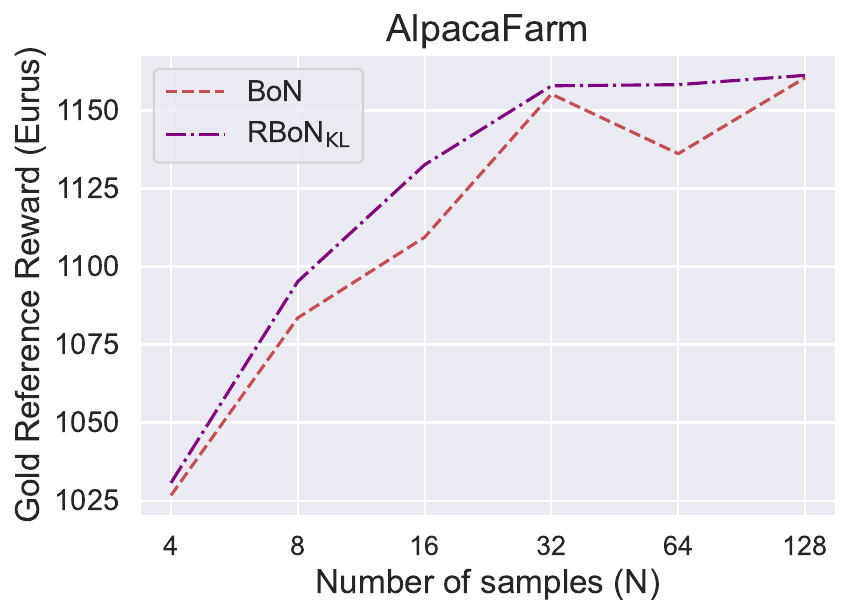}}\\
    \subfloat[Proxy: SHP-XL]{\includegraphics[width=0.72\columnwidth]{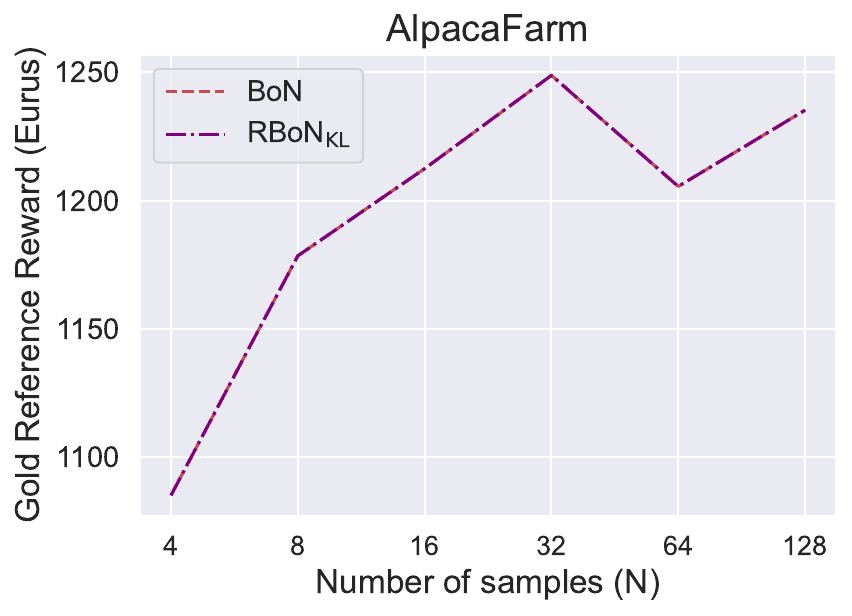}} \\
    \subfloat[Proxy: OASST]{\includegraphics[width=0.72\columnwidth]{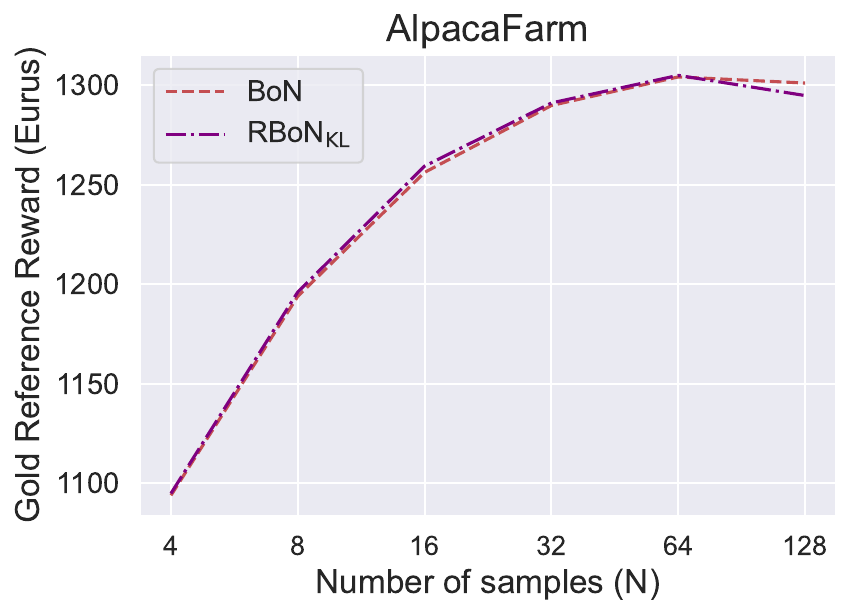}} 
    \caption{Evaluation of the \klrbon{} using Mistral on the AlpacaFarm dataset.}
    \label{fig:mistral-alpaca-kl}
\end{figure}

The tradeoff between the reward and the proximity to the reference model is controlled by the hyperparameter $\beta$.
With a small $\beta$, the output is more aligned with the proxy reward model. With $\beta=0$, the vanilla BoN is restored.
With larger $\beta$, the output is closer to the behavior of the reference model $\pi_\mathrm{ref}$, where $\beta=+\infty$ selects the response with the highest model probability, recovering the maximum a posteriori (MAP) decoding \citep{stahlberg-byrne-2019-nmt,eikema-aziz-2020-map,Holtzman2020The}.

Figure~\ref{fig:mistral-alpaca-kl} shows the performance of \klrbon{}. Overall, its improvement over BoN is marginal.

\section{Evaluation using Reward Model Trained on AlpacaFarm}
\label{apd:trainedrm}

Figure~\ref{fig:pythia-eval} shows the performance of \wdrbon{} compared to BoN using a reward model trained on the AlpacaFarm training set. The reward model is the gold reference reward model based on Pythia-2.8B used in Appendix~\ref{apd:overoptimization}. The improvement of \wdrbon{} over BoN is large when the number of samples are large and also when the proxy reward model is less algined with the gold reference reward model. On the other hand, when the proxy reward model is trained using noiseless (the same preference annotations as the gold reference model) and large enough dataset ($|\mathcal{D}| =4000$), the performance of \wdrbon{} is on par with BoN. 

\begin{figure}[th]
    \centering
    \includegraphics[width=0.9\columnwidth]{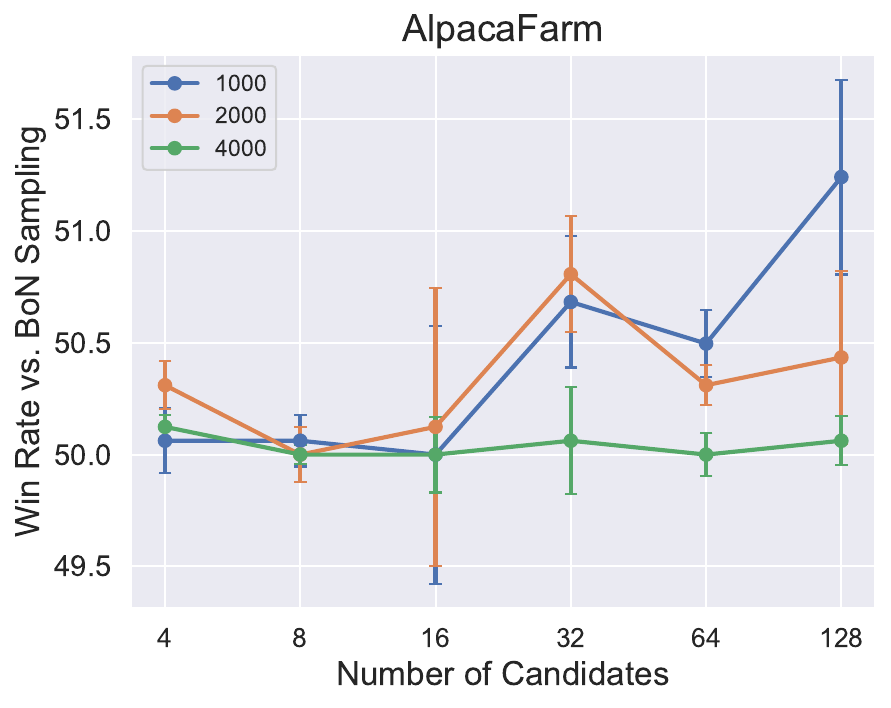}
    \caption{The average win rate of the MBR-BoN against BoN using a reward model trained on the training set of AlpacaFarm.}
    \label{fig:pythia-eval}
\end{figure}

\begin{figure}[thb]
    \centering
    \subfloat[SHP-Large]{\includegraphics[width=0.71\columnwidth]{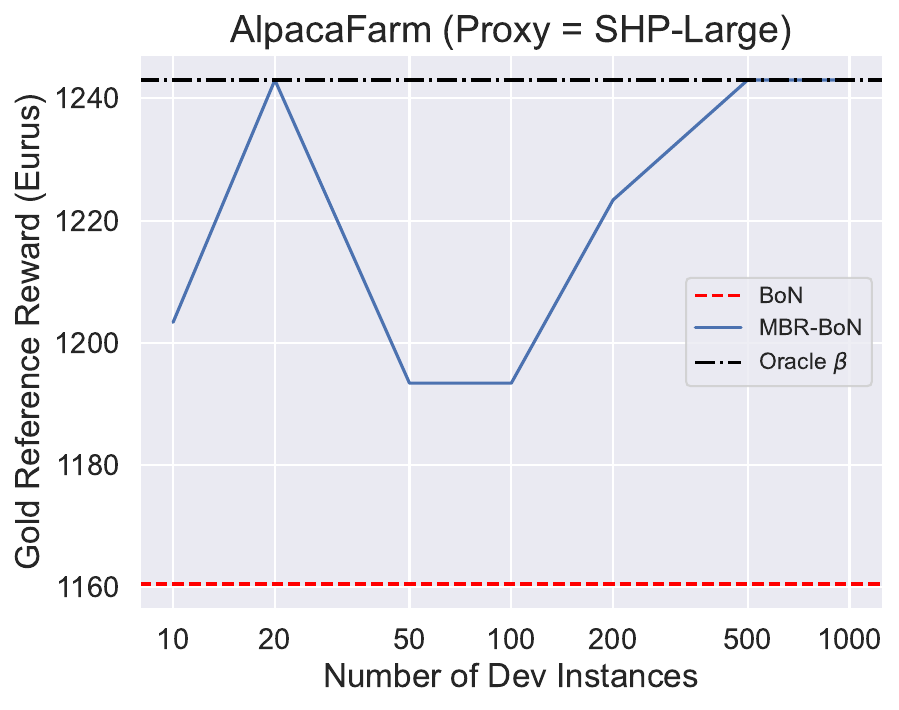}}\\ \subfloat[SHP-XL]{\includegraphics[width=0.71\columnwidth]{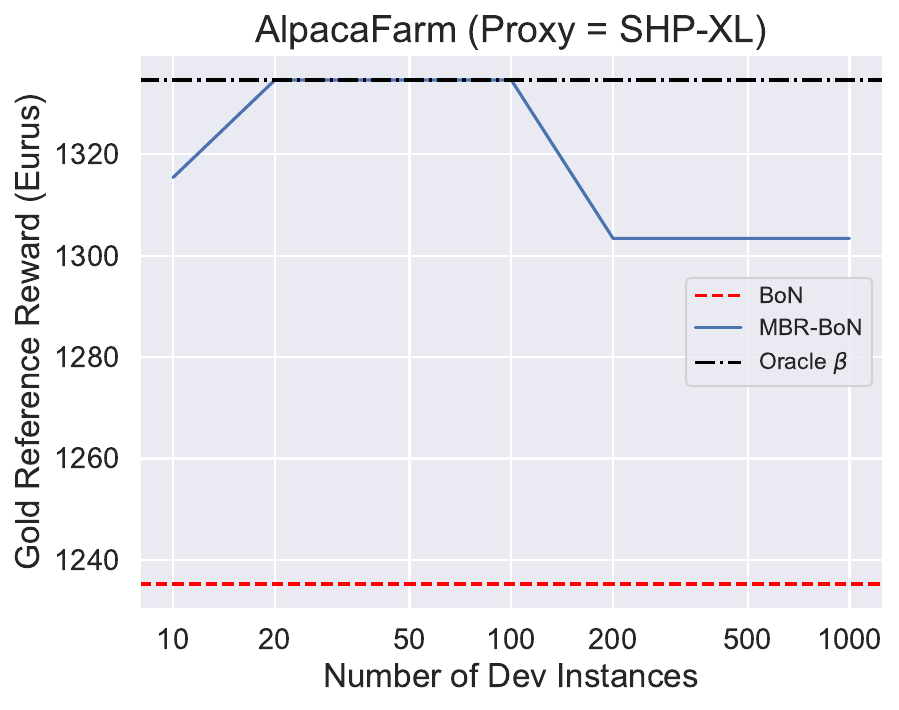}}\\
    \subfloat[OASST]{\includegraphics[width=0.71\columnwidth]{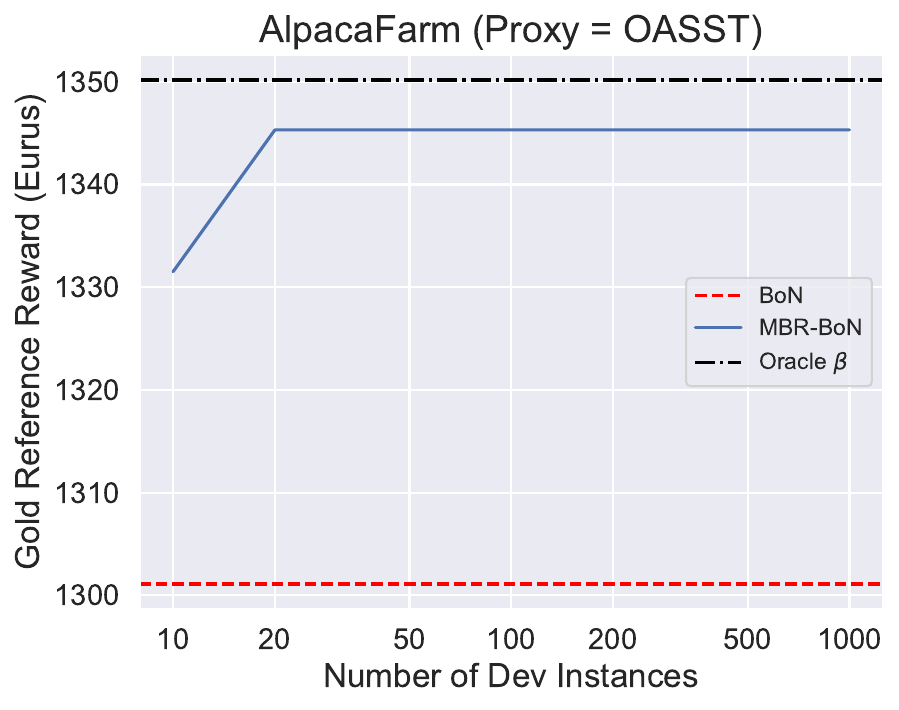}}
    \caption{Evaluation of MBR-BoN with varying sizes of development set to tune the optimal $\beta$.}
    \label{fig:devset}
\end{figure}

\section{Analysis on the Size of the Development Set for Tuning Beta}
\label{apd:devset}
We run a posthoc analysis to evaluate the effect of the size of the development set to tune the hyperparameter $\beta$ for MBR-BoN.
Figure~\ref{fig:devset} shows the performance of MBR-BoN with varying sizes of development set to compute the $\beta$, from 10 to 1000.
We observe that the score is relatively consistent and \wdrbon{} outperforms BoN even with 10 examples for fine-tuning $\beta$.

\begin{figure}[tbh]
    \centering
    \includegraphics[width=0.9\columnwidth]{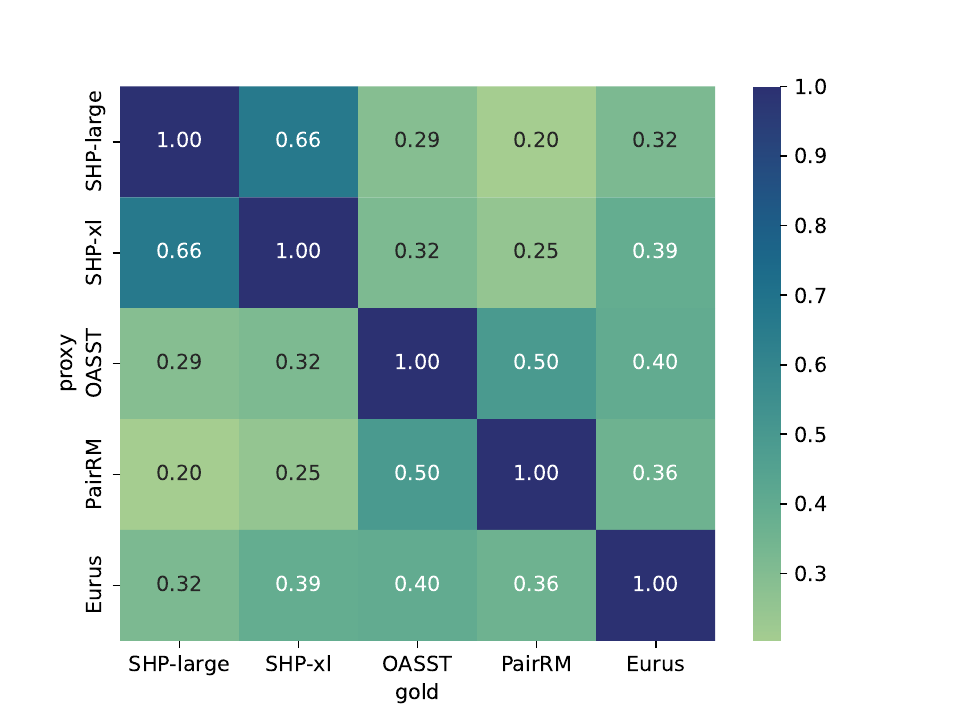}
    \caption{Average Spearman's rank correlation coefficient of the reward models in the evaluation split of the AlpacaFarm dataset for the responses generated by Mistral. 128 responses are used to compute Spearman's rank correlation for each instruction, averaged over the 805 instructions.}
    \label{fig:corr}
\end{figure}
\section{Effect of the Regularization Strength}
\label{apx:beta}

\paragraph{Correlation Coefficient.}
To understand the effect of the regularization strength on the performance of RBoN under different pairs of proxy and gold reward models, we evaluate RBoN using SHP-Large, SHP-XL, OASST, and PairRM \cite{llm-blender-2023} as gold reward models.
Figure \ref{fig:corr} reports the average Spearman's rank correlation coefficient $\rho$ of a pair of reward models \citep{spearman1904proof}. Note that SHP-Large and SHP-XL reward models are highly correlated as they are trained on the same training procedure.

\paragraph{Tradeoff between Proxy Reward and Proximity Scores.}
Figure~\ref{fig:tradeoff} shows the tradeoff of the proxy reward score and the MBR objective score with different values of $\beta$ on \wdrbon{}. 
The result shows that the hyperparameter $\beta$ effectively controls the weights over the proxy reward model and proximity to the reference policy.

\begin{figure*}
    \centering
    \subfloat[SHP-Large]{
    \includegraphics[width=0.32\textwidth]{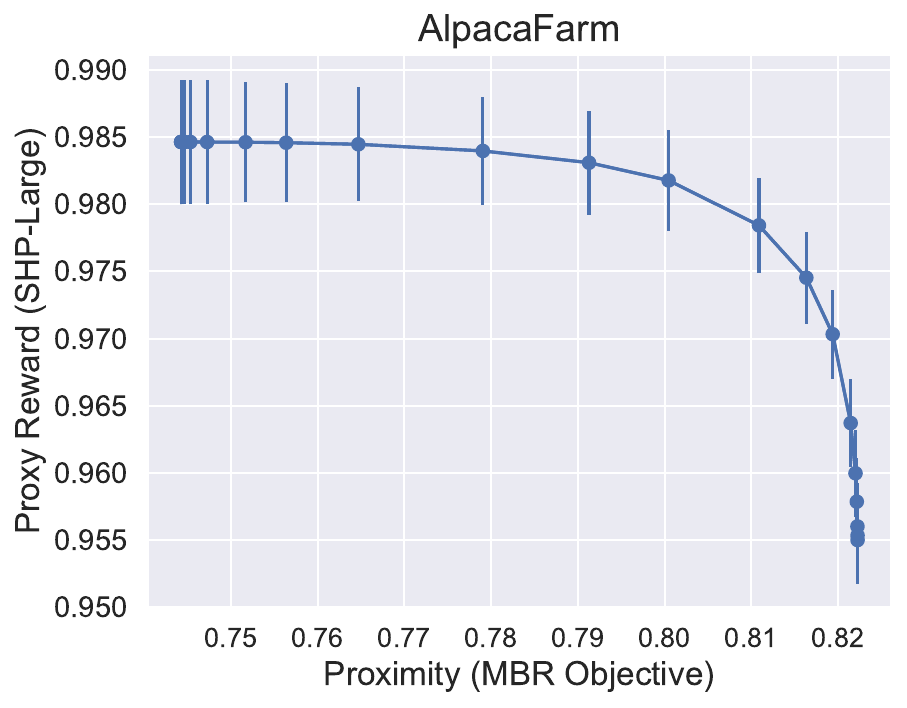}
    }
    \subfloat[SHP-XL]{
    \includegraphics[width=0.32\textwidth]{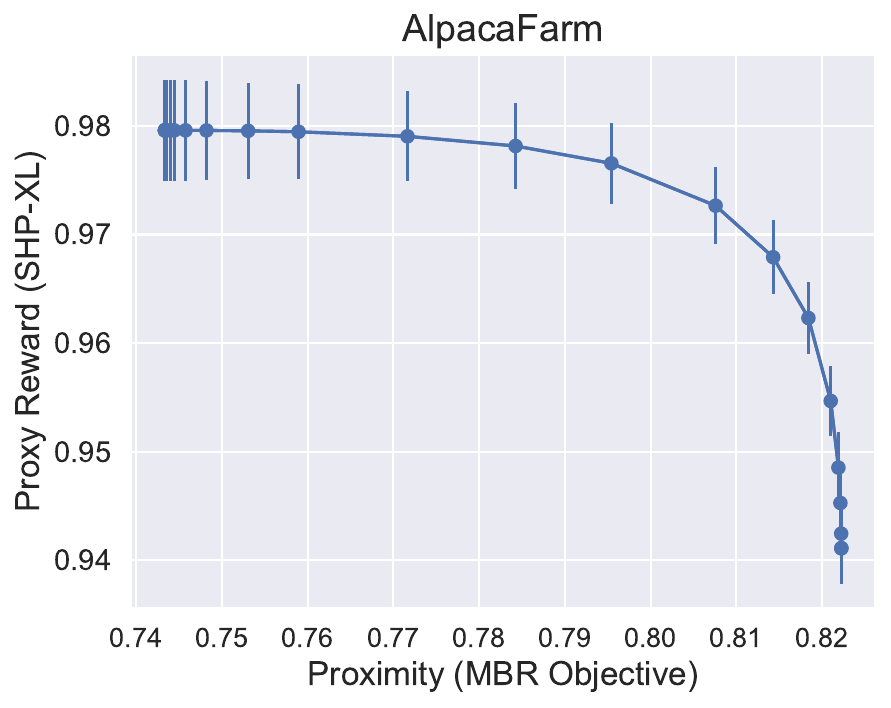}
    }
    \subfloat[OASST]{
    \includegraphics[width=0.32\textwidth]{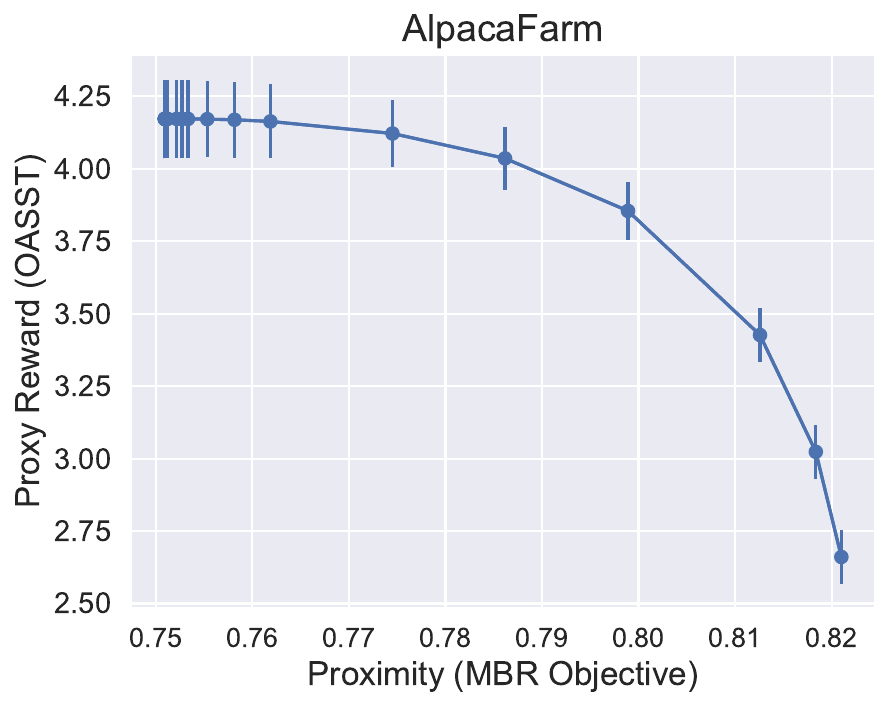}
    }
    \caption{The tradeoff of the proxy reward score and proximity (MBR objective) with \wdrbon{} using different $\beta$ strengths on AlpacaFarm. The responses are generated by Mistral. The number of samples $N$ is 128. The line shows the mean and the error bar shows the standard error of the estimation of the mean value.}    
    \label{fig:tradeoff}
\end{figure*}

\paragraph{Evaluation of \wdrbon{} using Various Reference Reward Models.}
We perform the generation (BoN and \wdrbon{}) using one of the reward models as the proxy reward model and evaluate the selected responses using the remaining reward models as the gold reference rewards.
We do not use PairRM as a proxy reward model because it is a pairwise reward model that estimates the preference for a pair of responses rather than computing an absolute preference for a response. The use of a pairwise reward model as a proxy reward model for RBoN is future work.

Figure~\ref{fig:mistral} shows the performance of BoN and \wdrbon{} with varying $\beta$ with $N=128$ using Mistral on the AlpacaFarm dataset. 
\wdrbon{} outperforms BoN in all settings except when the proxy reward model is highly correlated with the gold reward model (e.g., SHP-Large and SHP-XL).

The experiment shows that the optimal $\beta$ depends on various factors, but the strength of the correlation between the proxy reward model and the gold reference reward seems to be the key factor.
For example, SHP-Large is strongly correlated with SHP-XL ($\rho = 0.66$), so the optimal $\beta$ is close to 0. In this case, \wdrbon{} has little to no advantage over BoN. On the other hand, SHP-Large is only weakly correlated with OASST and PairRM ($\rho = 0.29, 0.20$), where the optimal $\beta$ for SHP-Large $\rightarrow$ OASST and PairRM is large ($\beta = 0.1-1.0$).

\begin{figure*}
    \centering
    \subfloat[SHP-Large $\rightarrow$ SHP-XL]{
    \includegraphics[width=0.32\textwidth]{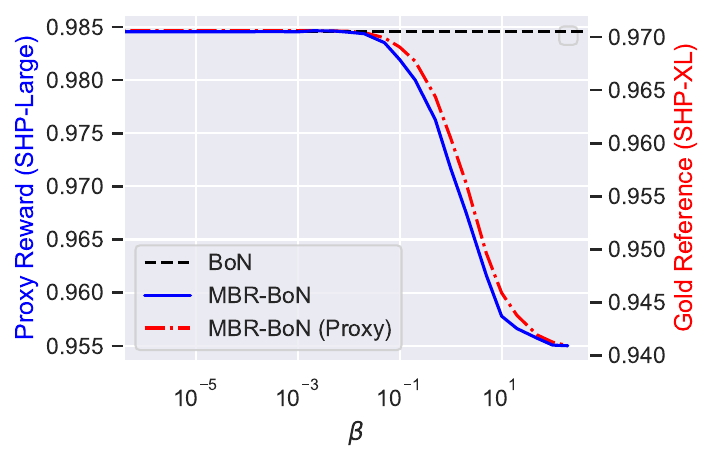}
    }
    \subfloat[SHP-Large $\rightarrow$ OASST]{
    \includegraphics[width=0.32\textwidth]{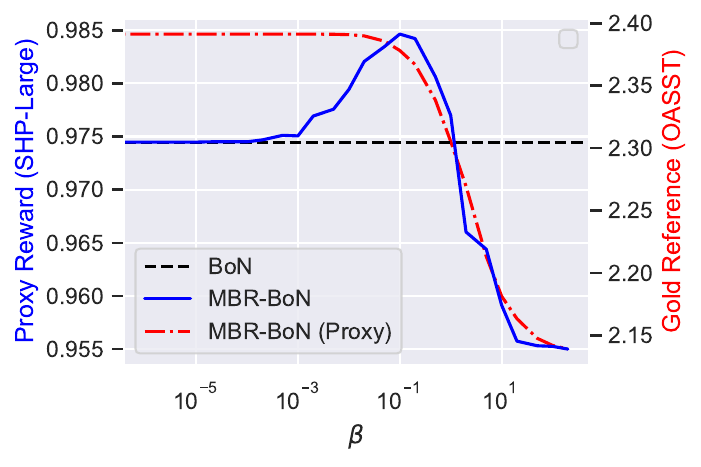}
    }
    \subfloat[SHP-Large $\rightarrow$ PairRM]{
    \includegraphics[width=0.32\textwidth]{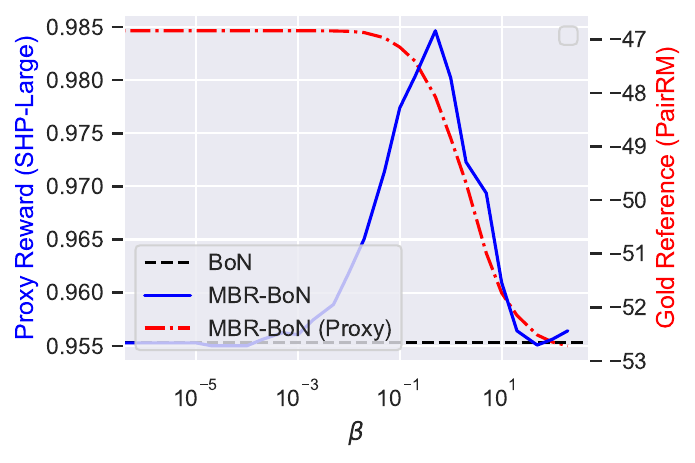}
    }
    
    \subfloat[SHP-XL $\rightarrow$ SHP-Large]{
    \includegraphics[width=0.32\textwidth]{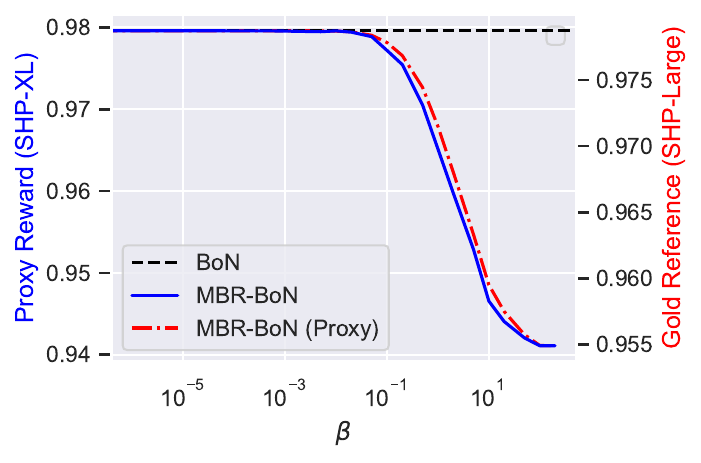}
    }
    \subfloat[SHP-XL $\rightarrow$ OASST]{
    \includegraphics[width=0.32\textwidth]{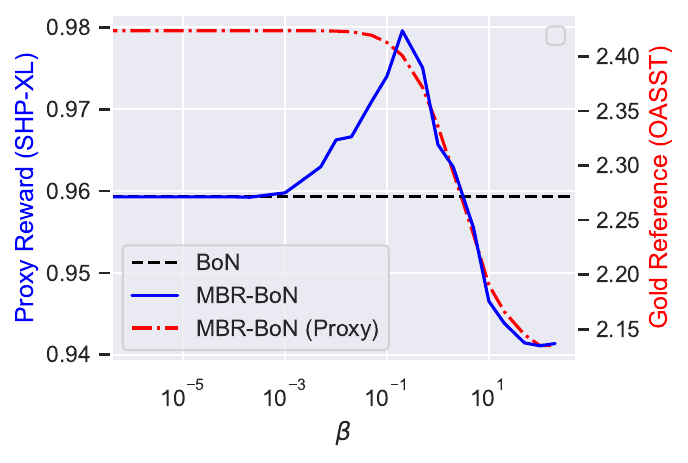}
    }
    \subfloat[SHP-XL $\rightarrow$ PairRM]{
    \includegraphics[width=0.32\textwidth]{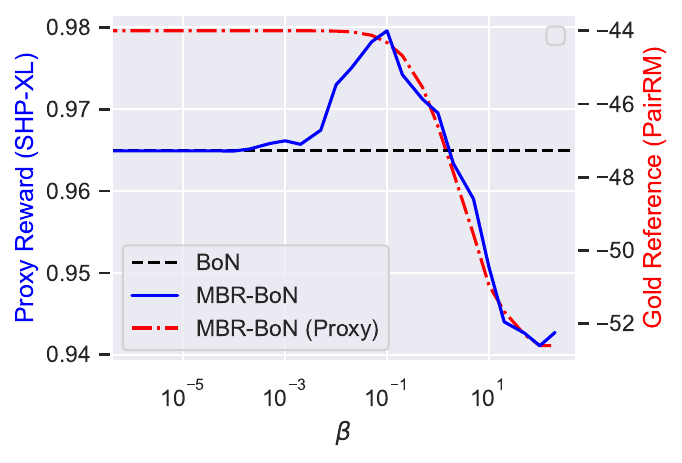}
    }
    
    \subfloat[OASST $\rightarrow$ SHP-Large]{
    \includegraphics[width=0.32\textwidth]{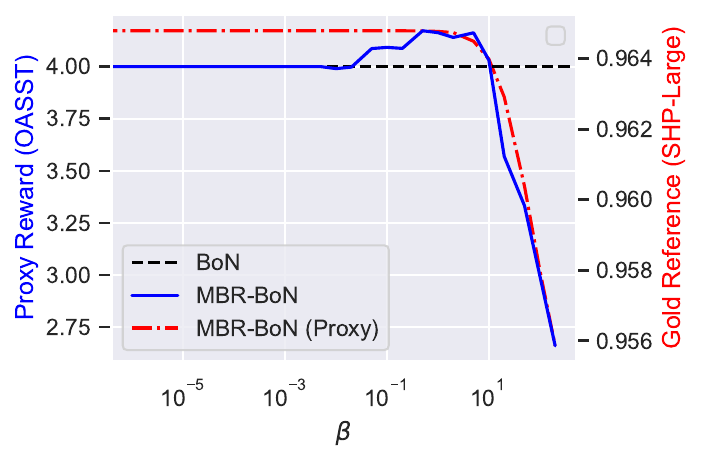}
    }
    \subfloat[OASST $\rightarrow$ SHP-XL]{
    \includegraphics[width=0.32\textwidth]{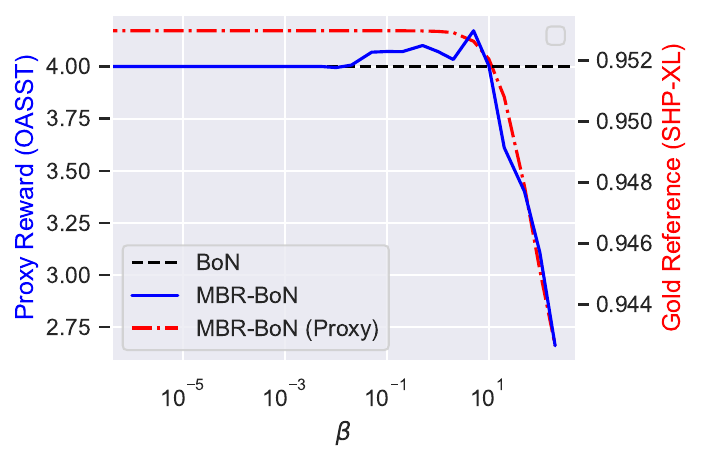}
    }
    \subfloat[OASST $\rightarrow$ PairRM]{
    \includegraphics[width=0.32\textwidth]{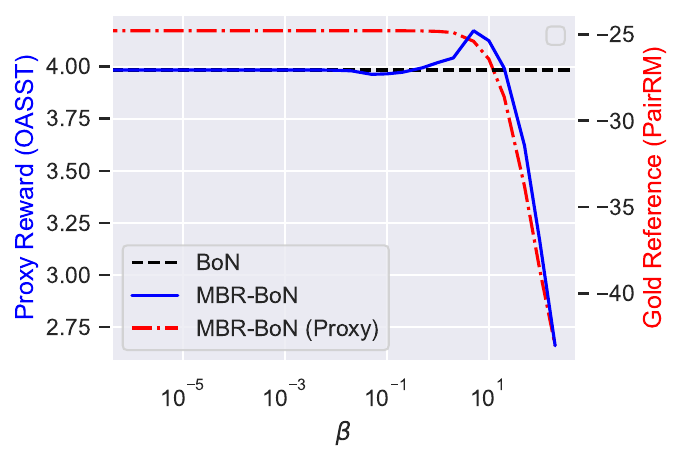}
    }
    \caption{The gold reward score and the proxy reward score of the \wdrbon{} with different regularization strengths and reward models. The captions of the subfigures show the proxy and the gold reward model (Proxy $\rightarrow$ Gold). The performance of BoN is shown in the horizontal lines. The responses are generated by Mistral. The number of samples $N$ is 128.}    
    \label{fig:mistral}
\end{figure*}

Figures~\ref{fig:mistral-n} and \ref{fig:dolly-n} show the performance of BoN ($\beta=0$), MBR decoding ($\beta=+\infty$), and \wdrbon{} with different number of samples $N$ using Mistral and Dolly on AlpacaFarm. We observe qualitatively similar results with smaller $N$ to the result of $N=128$ in Figure~\ref{fig:mistral-alpaca}.

\begin{figure*}
    \centering
    \subfloat[SHP-Large $\rightarrow$ SHP-XL]{
    \includegraphics[width=0.32\textwidth]{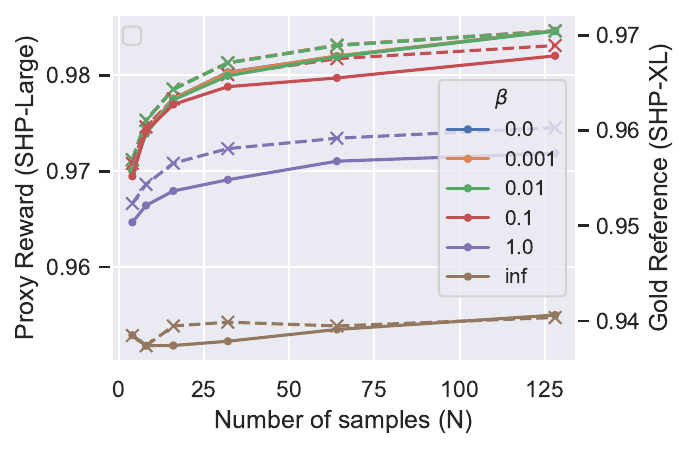}
    }
    \subfloat[SHP-Large $\rightarrow$ OASST]{
    \includegraphics[width=0.32\textwidth]{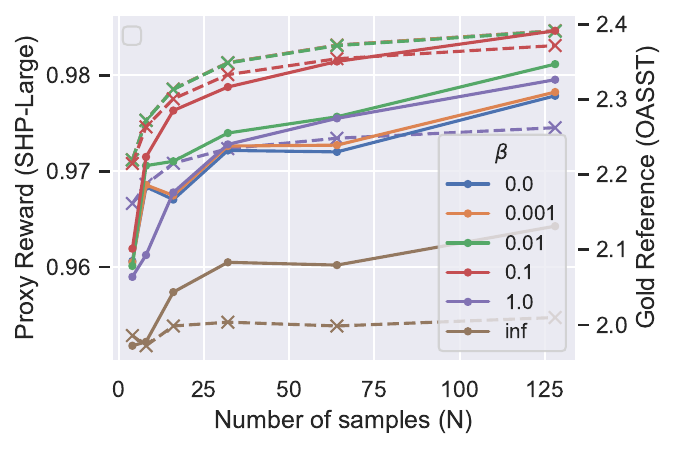}
    }
    \subfloat[SHP-Large $\rightarrow$ PairRM]{
    \includegraphics[width=0.32\textwidth]{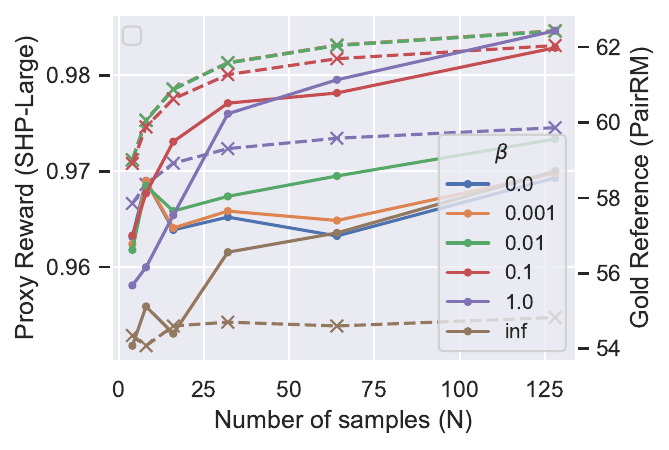}
    }
    
    \subfloat[SHP-XL $\rightarrow$ SHP-Large]{
    \includegraphics[width=0.32\textwidth]{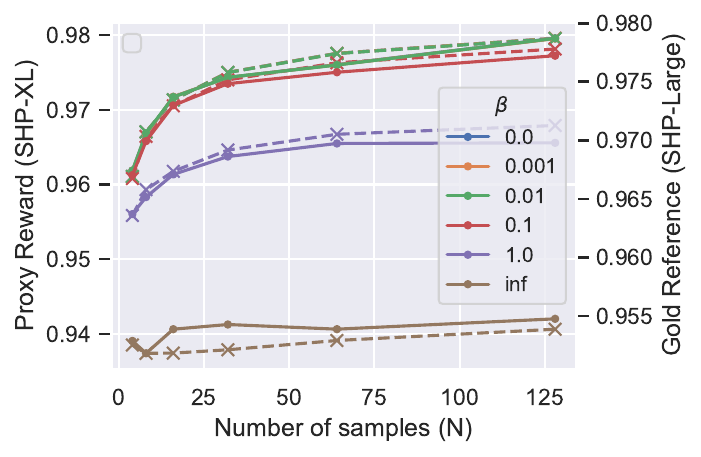}
    }
    \subfloat[SHP-XL $\rightarrow$ OASST]{
    \includegraphics[width=0.32\textwidth]{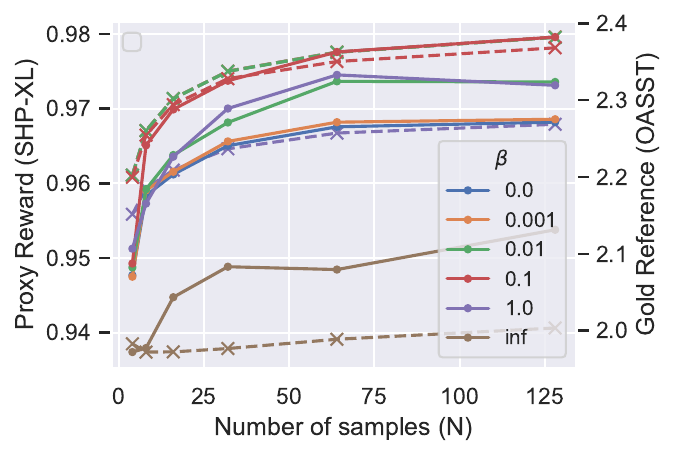}
    }
    \subfloat[SHP-XL $\rightarrow$ PairRM]{
    \includegraphics[width=0.32\textwidth]{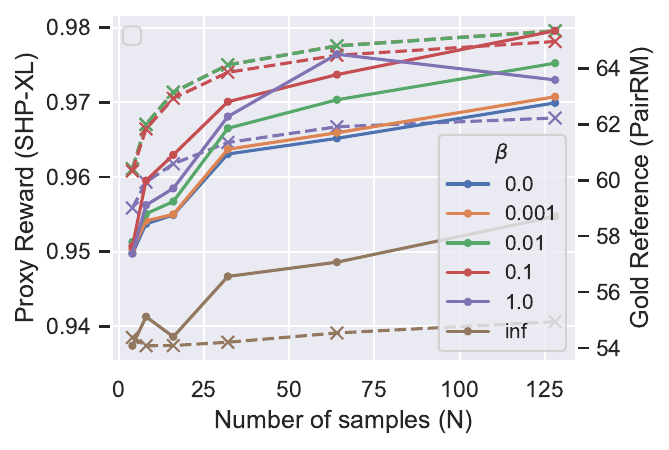}
    }
    
    \subfloat[OASST $\rightarrow$ SHP-Large]{
    \includegraphics[width=0.32\textwidth]{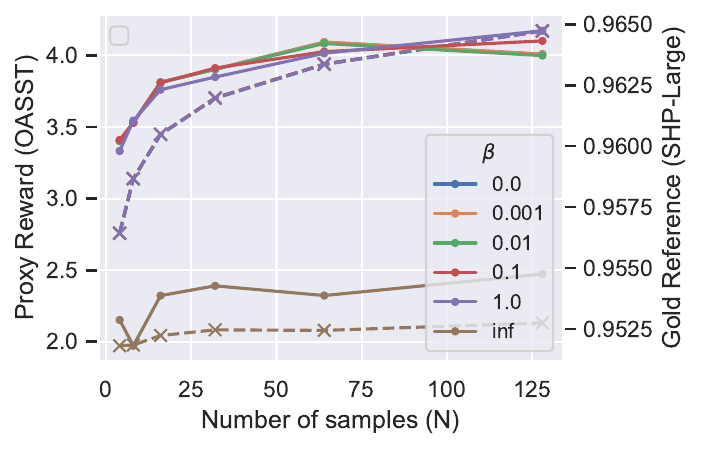}
    }
    \subfloat[OASST $\rightarrow$ SHP-XL]{
    \includegraphics[width=0.32\textwidth]{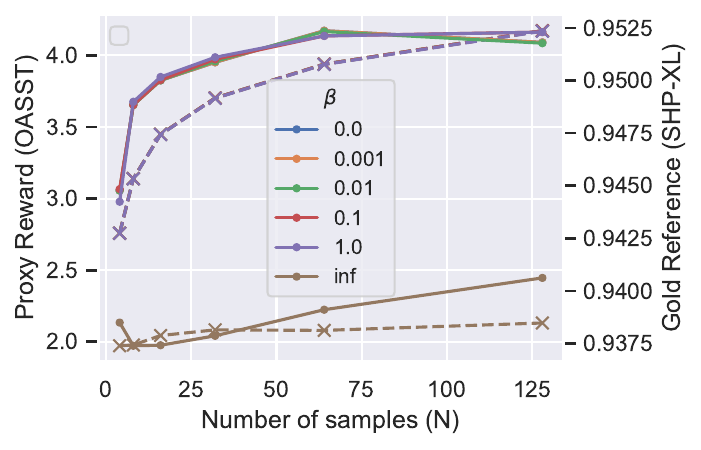}
    }
    \subfloat[OASST $\rightarrow$ PairRM]{
    \includegraphics[width=0.32\textwidth]{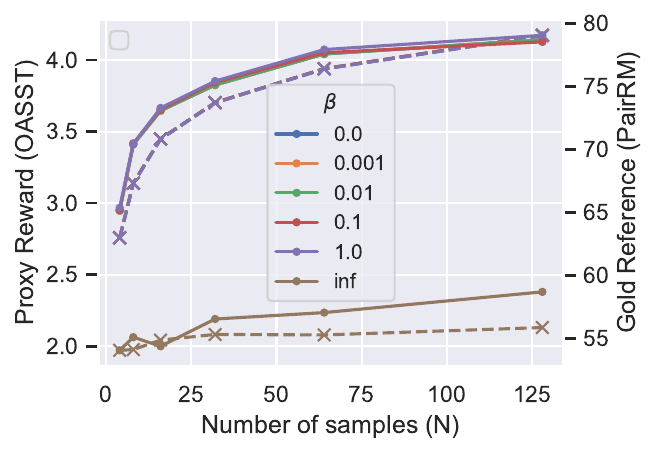}
    }
    \caption{Evaluation of \wdrbon{} using Mistral on AlpacaFarm. The gold reward score and the proxy reward score of the \wdrbon{} with different regularization strengths and reward models. The captions of the subfigures show the proxy and the gold reward model (Proxy $\rightarrow$ Gold). The reward scores of the reference reward (right axis) are shown in solid lines whereas the reward scores of the proxy reward (left axis) are shown in dashed lines. $\beta=$inf corresponds to the MBR decoding.} 
    \label{fig:mistral-n}
\end{figure*}

\begin{figure*}[t]
    \centering
    \subfloat[SHP-Large $\rightarrow$ SHP-XL]{
    \includegraphics[width=0.32\textwidth]{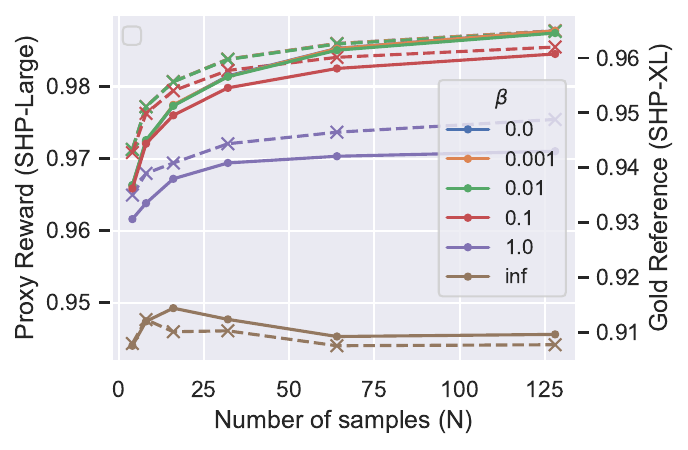}
    }
    \subfloat[SHP-Large $\rightarrow$ OASST]{
    \includegraphics[width=0.32\textwidth]{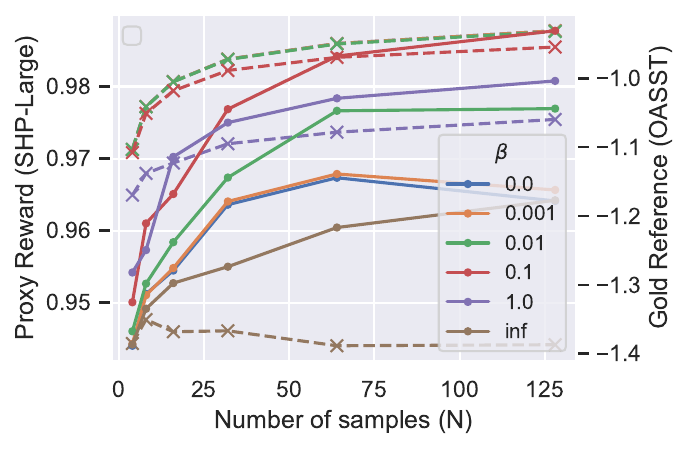}
    }
    \subfloat[SHP-Large $\rightarrow$ PairRM]{
    \includegraphics[width=0.32\textwidth]{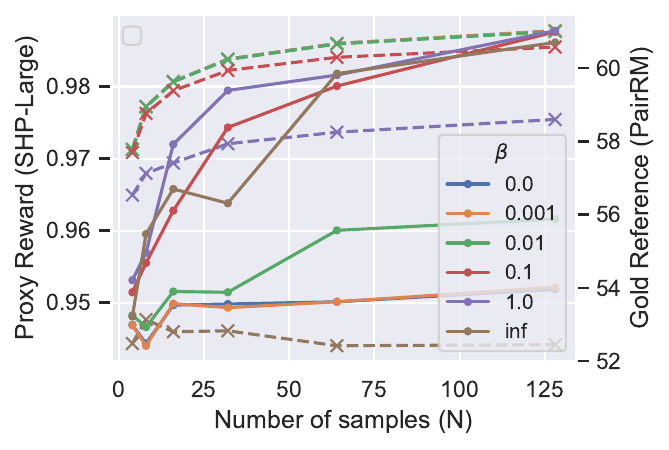}
    }
    
    \subfloat[SHP-XL $\rightarrow$ SHP-Large]{
    \includegraphics[width=0.32\textwidth]{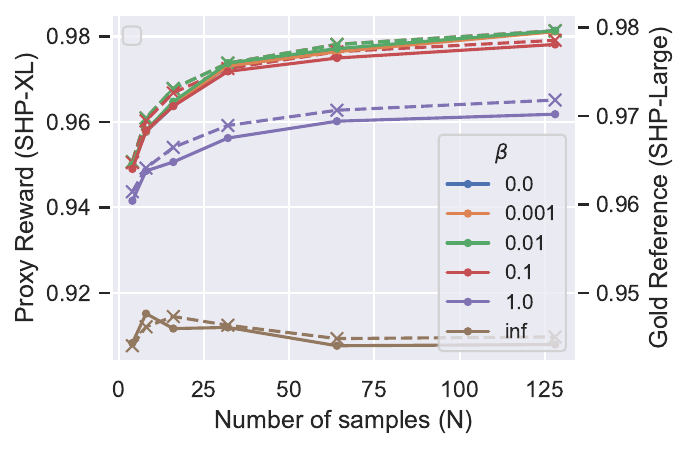}
    }
    \subfloat[SHP-XL $\rightarrow$ OASST]{
    \includegraphics[width=0.32\textwidth]{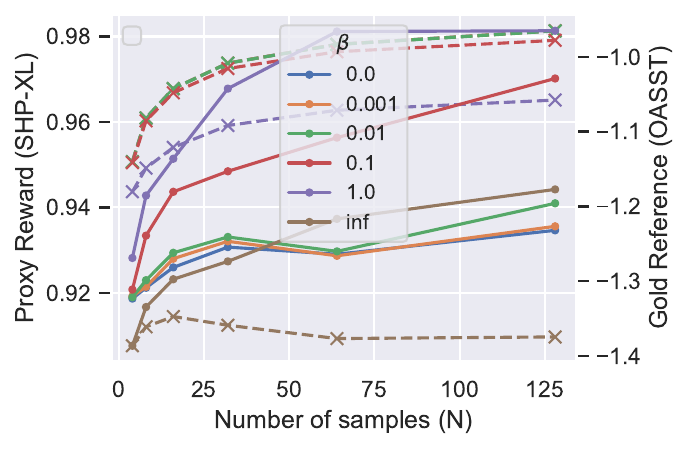}
    }
    \subfloat[SHP-XL $\rightarrow$ PairRM]{
    \includegraphics[width=0.32\textwidth]{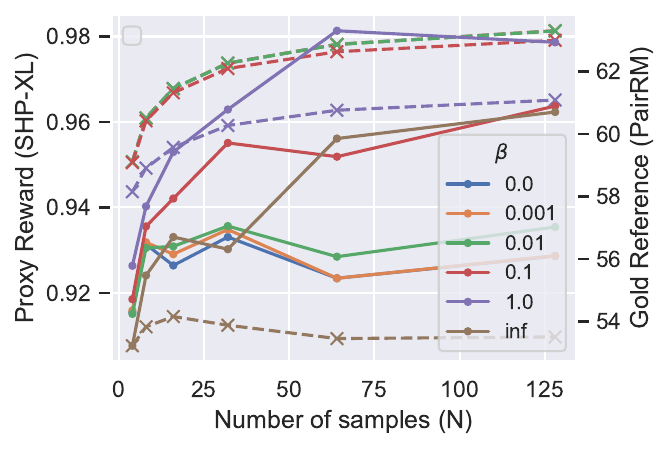}
    }
    
    \subfloat[OASST $\rightarrow$ SHP-Large]{
    \includegraphics[width=0.32\textwidth]{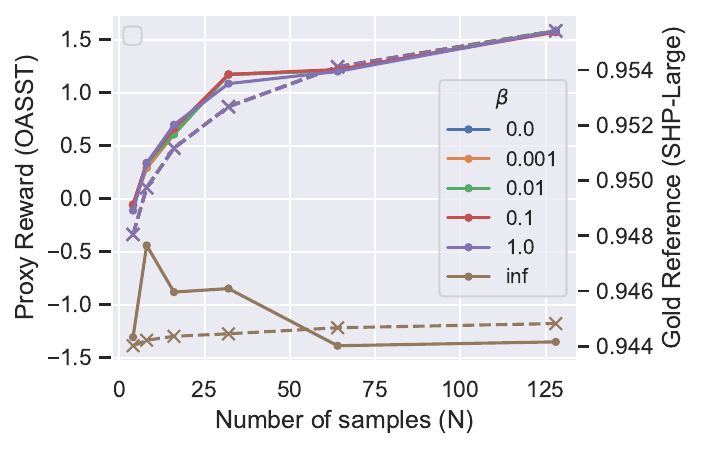}
    }
    \subfloat[OASST $\rightarrow$ SHP-XL]{
    \includegraphics[width=0.32\textwidth]{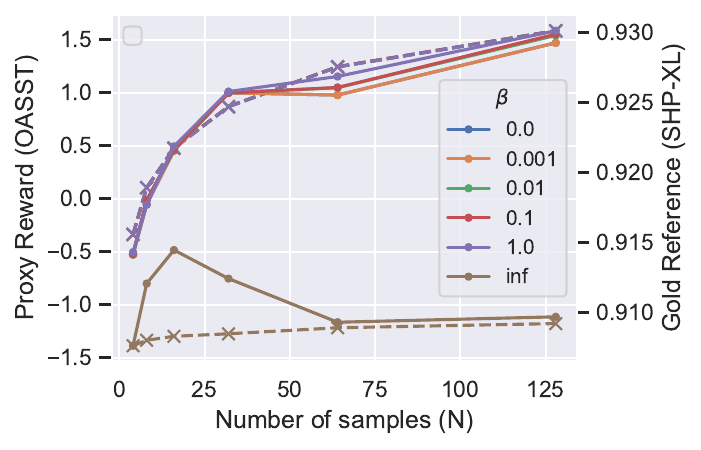}
    }
    \subfloat[OASST $\rightarrow$ PairRM]{
    \includegraphics[width=0.32\textwidth]{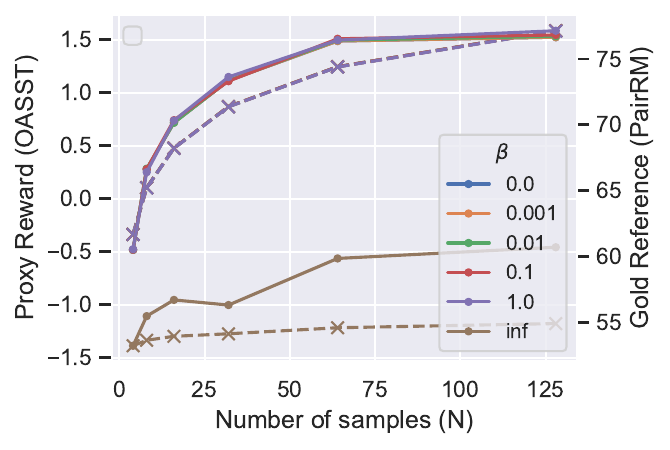}
    }
    \caption{Evaluation of \wdrbon{} using Dolly on AlpacaFarm. The gold reward score and the proxy reward score of the \wdrbon{} with different regularization strengths and reward models. The captions of the subfigures show the proxy and the gold reward model (Proxy $\rightarrow$ Gold). The reward scores of the reference reward (right axis) are shown in solid lines whereas the reward scores of the proxy reward (left axis) are shown in dashed lines. $\beta=$inf corresponds to the MBR decoding.} 
    \label{fig:dolly-n}
\end{figure*}

\clearpage
\section{GPT-4o Evaluation of the DPO}
\label{apx:gpt4}

Figure~\ref{fig:gpt4} shows the average score of the models trained by DPO in Section~\ref{sec:learning} using GPT-4o as a judge \citep{NEURIPS2023_91f18a12,openai2024gpt4}.
We evaluate using the first 300 entries of the test split of the datasets. 
We use the following prompt to ask GPT-4o to evaluate the quality of the output.
\begin{quote}
    Please act as an impartial judge and evaluate the quality of the response provided by an AI assistant to the user question displayed below. Your evaluation should consider factors such as the helpfulness, relevance, accuracy, depth, creativity, and level of detail of the response. Begin your evaluation by providing a short explanation. Be as objective as possible. After providing your explanation, you must rate the response on a scale of 1 to 10 by strictly following this format: ``[[rating]]'', for example: ``Rating: [[5]]''.\\
    \\
    {[}Question{]}\\
    \{question\}\\
    {[}The Start of Assistant’s Answer{]}\\
    \{answer\}\\
    {[}The End of Assistant’s Answer{]}
\end{quote}
The model name is gpt-4o and the model version is 2024-05-13. We set the model temperature, frequency penalty, and presence penalty to 0. 
Overall, we observe the same qualitative result that models trained using the proposed method outperform the model using the BoN sampling.
For the generations of the fine-tuned models we evaluate, the average agreement of GPT-4o evaluation with the Eurus reward model is 0.708 for AlpacaFarm and 0.750 for hh-rlhf datasets.

\begin{figure}
    \centering
    \subfloat[AlpacaFarm]{\includegraphics[width=\columnwidth]{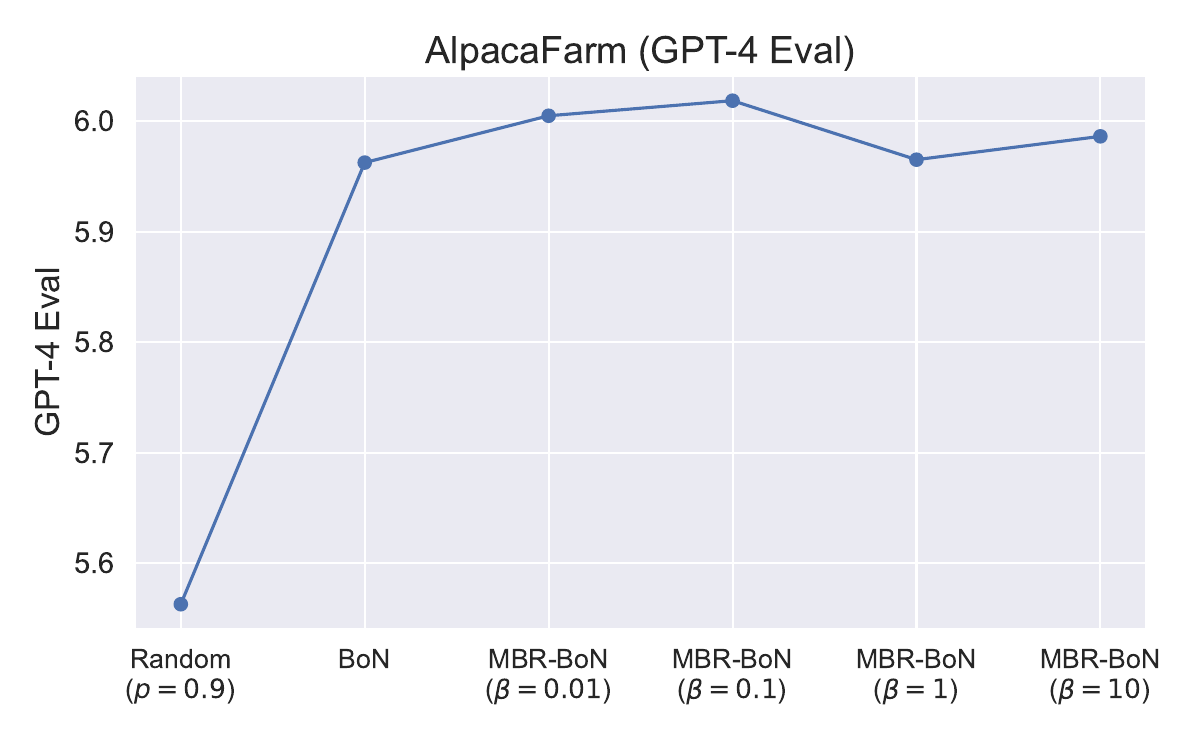}}\\
    \subfloat[Helpfulness]{\includegraphics[width=\columnwidth]{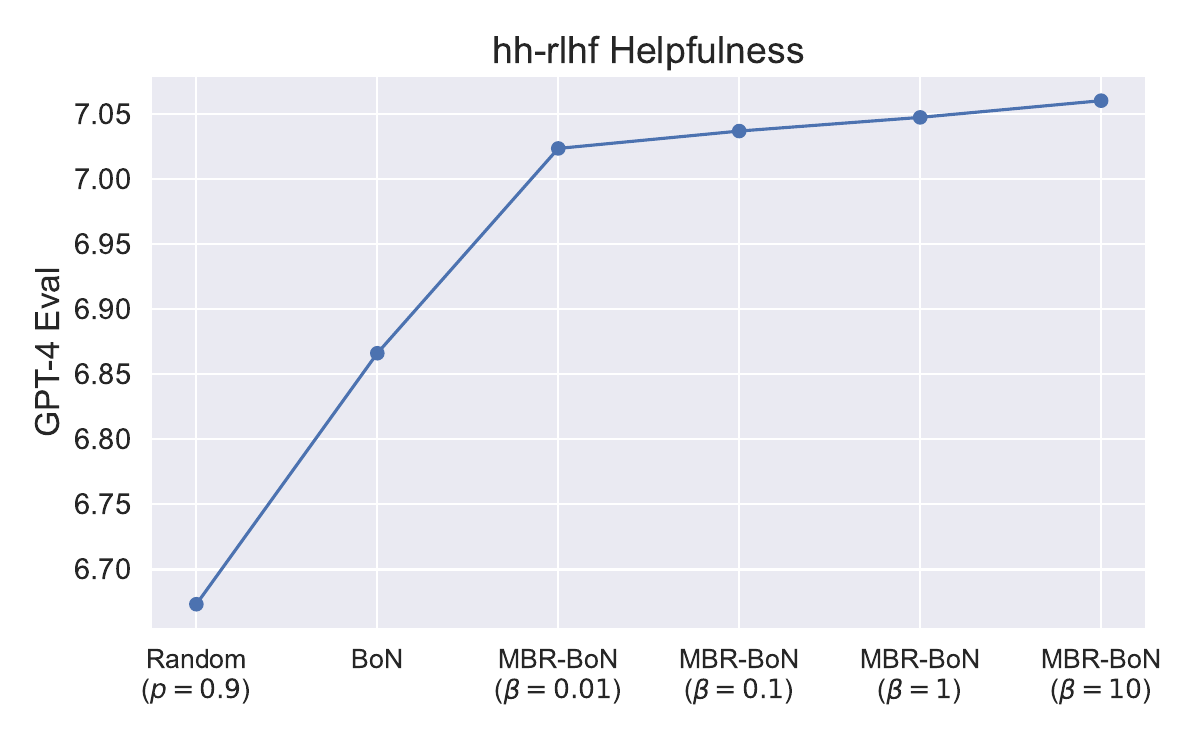}}\\
    \subfloat[Harmlessness]{\includegraphics[width=\columnwidth]{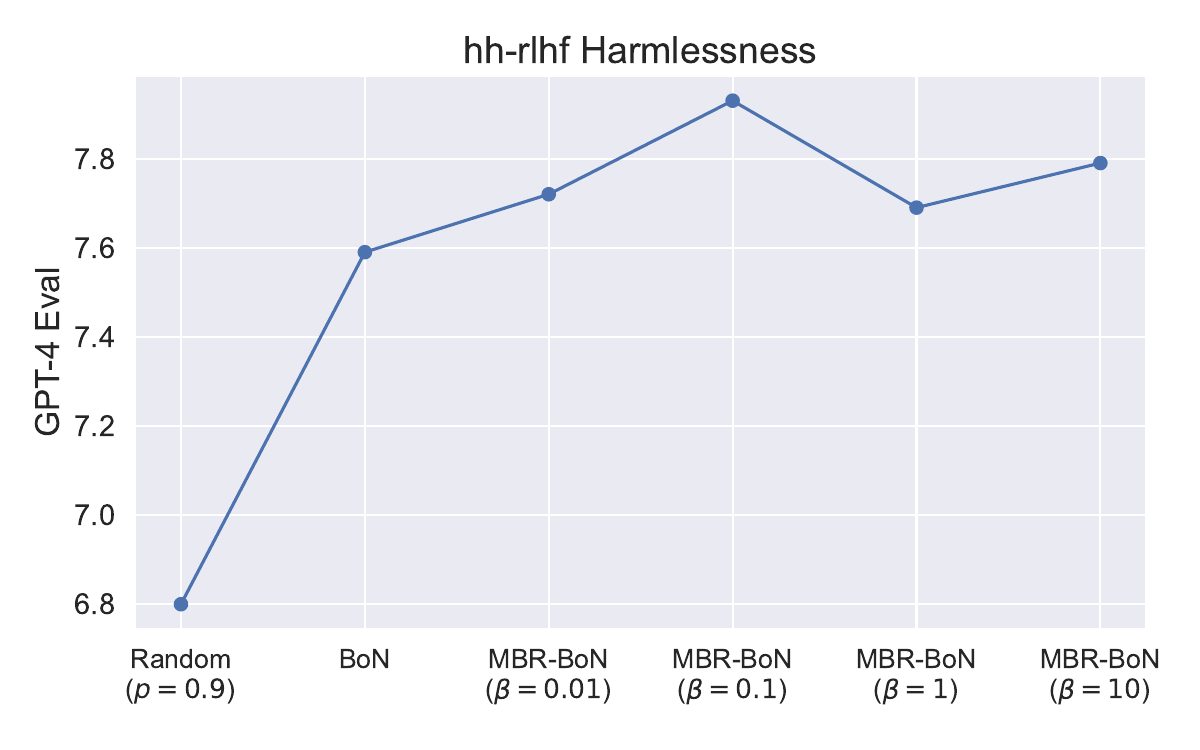}}
    \caption{GPT-4o Evaluation of the fine-tuned models trained using MBR-BoN.}
    \label{fig:gpt4}
\end{figure}

\section{Walltime}
\label{apd:walltime}

We report the wall clock time of BoN and MBR-BoN in Table \ref{tab:walltime}. The batch size for generating samples is set to 4. The code base is based on Huggingface's Transformers library \citep{wolf-etal-2020-transformers} and is not based on a library optimized for inference speed (e.g., vLLM; \citealp{kwon2023efficient}).
We use OASST reward model with the batch size set to 8.
We set the batch size for the computation of the similarity between sequences for the MBR values to 64. 
In our code base, we store the generated samples, computed reward values, and the MBR values to a cloud storage. The reported wall clock time may also include the time for the logging procedures.
The wall clock time depends on various factors including the code base and the hardware. All the experiments are conducted using an NVIDIA T4 GPU.

\begin{table}
    \centering
    \caption{Summary of wall clock time of BoN and MBR-BoN with $N=128$ for AlpacaFarm dataset. All experiements are run on an NVIDIA T4 GPU.}
    \label{tab:walltime}
    \adjustbox{max width=\columnwidth}{
    \begin{tabular}{ccc}
    \toprule
         & \multicolumn{2}{c}{Run time (seconds)} \\
         \cmidrule{2-3}
         & BoN & MBR-BoN \\
         \midrule
        Generate samples & 134 & 134 \\
        Compute the reward values & 0.1 & 0.1 \\
        Compute the MBR values & - & 2 \\
    \bottomrule
    \end{tabular}
    }
\end{table}

\section{Hyperparameters}
\label{apx:hyperparams}

Table \ref{tab:gen-hypers} describes the hyperparameters used to generate responses from the $\pi_\mathrm{ref}$. The parameters are used for both Sections \ref{sec:decoding} and \ref{sec:learning}.
Table \ref{tab:dpo-hypers} summarizes the hyperparameters used for DPO in Section \ref{sec:learning}.

\begin{table}
    \caption{Generation hyperparameters used in Section \ref{sec:decoding} and \ref{sec:learning}}
    \label{tab:gen-hypers}
    \centering
    \begin{tabular}{c|c}
        \toprule
        Parameter & Value \\\midrule
        Max instruction length & 256 \\
        Max new tokens & 256 \\
        Temperature & 1.0 \\
        Top-$p$     & 0.9 \\
         \bottomrule
    \end{tabular}
\end{table}

\begin{table}
    \caption{DPO hyperparameters used in Section \ref{sec:learning}.}
    \label{tab:dpo-hypers}
    \centering
    \begin{tabular}{cc}
        \toprule
        Parameter & Value \\\midrule
        Epochs & 3 \\
        Learning rate & 1e-5 \\
        Optimizer & AdamW \\
        Batch size    & 4 \\
        Regularization factor ($\beta$)    & 0.1 \\
        LoRA $r$  & 128 \\
        LoRA $\alpha$  & 32 \\
         \bottomrule
    \end{tabular}
\end{table}

\begin{table}
    \centering
    \caption{Hyperparameters for training reward models used in Appendix~\ref{apd:overoptimization}. The values follow the defaults of the TRL library.}
    \label{tab:rm-hypers}
    \begin{tabular}{cc}
        \toprule
        Parameter & Value \\\midrule
        Epochs & 3 \\
        Learning rate & 5e-05 \\
        Optimizer & AdamW \\
        Batch size    & 8 \\
        \bottomrule
    \end{tabular}
\end{table}

\section{Reproducibility Statement}

All datasets and models used in the experiments are publicly available except for GPT-4o (Table \ref{tab:links}). 
The code is implemented using Huggingface's Transformers library \citep{wolf-etal-2020-transformers} and TRL library \citep{vonwerra2022trl}.
The PCA and ICA are implemented using scikit-learn \cite{scikit-learn}.
Our code is available at \url{https://github.com/CyberAgentAILab/regularized-bon} with an MIT license.

\begin{table*}
    \caption{List of datasets and models used in the experiments.}
    \label{tab:links}
    \centering
    \begin{tabularx}{\textwidth}{cX}
    \toprule
        Name & Reference \\
    \midrule
        AlpacaFarm & \cite{NEURIPS2023_5fc47800} \url{https://huggingface.co/datasets/tatsu-lab/alpaca_farm} \\\midrule
        Anthropic's hh-rlhf & \cite{bai2022training} \url{https://huggingface.co/datasets/Anthropic/hh-rlhf} \\\midrule
        WMT'21 De-En & \cite{akhbardeh-etal-2021-findings} \url{https://github.com/wmt-conference/wmt21-news-systems} \\\midrule
        mistral-7b-sft-beta (Mistral) & \cite{jiang2023mistral,tunstall2023zephyr} \url{https://huggingface.co/HuggingFaceH4/mistral-7b-sft-beta} \\\midrule
        dolly-v2-3b (Dolly) & \cite{DatabricksBlog2023DollyV2} \url{https://huggingface.co/databricks/dolly-v2-3b} \\\midrule
        Pythia-1B & \cite{pmlr-v202-biderman23a} \url{https://huggingface.co/EleutherAI/pythia-1b} \\\midrule
        Pythia-2.8B & \cite{pmlr-v202-biderman23a} \url{https://huggingface.co/EleutherAI/pythia-2.8b} \\\midrule
        wmt21-dense-24-wide & \cite{tran-etal-2021-facebook} \url{https://huggingface.co/facebook/wmt21-dense-24-wide-x-en} \\\midrule
        SHP-Large & \cite{pmlr-v162-ethayarajh22a} \url{https://huggingface.co/stanfordnlp/SteamSHP-flan-t5-large} \\\midrule
        SHP-XL & \cite{pmlr-v162-ethayarajh22a} \url{https://huggingface.co/stanfordnlp/SteamSHP-flan-t5-xl} \\\midrule
        OASST & \cite{kopf2023openassistant} \url{https://huggingface.co/OpenAssistant/reward-model-deberta-v3-large-v2} \\\midrule
        PairRM & \cite{llm-blender-2023} \url{https://huggingface.co/llm-blender/PairRM} \\\midrule
        Eurus & \cite{yuan2024advancing} \url{https://huggingface.co/openbmb/Eurus-RM-7b} \\\midrule
        MPNet & \cite{song2020mpnet} \url{https://huggingface.co/sentence-transformers/all-mpnet-base-v2} \\\midrule
        wmt20-comet-da & \cite{rei-etal-2020-comet} \url{https://huggingface.co/Unbabel/wmt20-comet-da} \\
        \bottomrule
    \end{tabularx}
\end{table*}

\end{document}